\documentclass[accepted]{uai2026}
\usepackage{graphicx} % Required for inserting images

\usepackage{mathtools} % amsmath with fixes and additions
\usepackage{booktabs} % commands to create good-looking tables
\usepackage{tikz} % nice language for creating drawings and diagrams

\usepackage{natbib} % has a nice set of citation styles and commands
\usepackage{latexsym}
\usepackage[utf8]{inputenc}
\usepackage{amssymb}
\usepackage{bm}
\usepackage{dsfont}
\usepackage{nccmath}
\usepackage{amsmath}
\usepackage{mathtools}
\usepackage{mathrsfs}
\usepackage{amsthm}
\usepackage{amsfonts}
\usepackage[linesnumbered,ruled,vlined]{algorithm2e}
\usepackage{multirow}
\usepackage{diagbox}
\usepackage{comment}
\usepackage{color}
\usepackage{subcaption}
\usepackage{enumitem}
\usepackage{mathtools}
\usepackage{physics}
\usepackage{optidef}

\usepackage{hyperref}

\usepackage{wrapfig}

\newtheorem{assumption}{Assumption}
\newtheorem{definition}{Definition}
\newtheorem{lemma}{Lemma}
\newtheorem{theorem}{Theorem}
\newtheorem{proposition}{Proposition}

\newtheorem{problem}{Problem}
\newtheorem{remark}{Remark}
\newtheorem{example}{Example}

\newcommand{\eventually}{\Diamond}

\begin{document}
\title{Sample-Efficient Learning of Probabilistic Causes for Reachability in Markov Decision Processes with Probabilistic Guarantees}
% \author{Ryohei Oura, Hideki Okamoto, Bardh Hoxha, Georgios Fainekos}
\author[1, 2]{\href{ryohei_oura@mail.toyota.co.jp}{Ryohei~Oura}{}}
\author[1]{\href{georgios.fainekos@toyota.com}{Georgios~Fainekos}{}}
\author[1]{\href{hideki.okamoto@toyota.com}{Hideki~Okamoto}{}}
\author[1]{\href{bardh.hoxha@toyota.com}{Bardh~Hoxha}{}}
% Add affiliations after the authors
\affil[1]{%
    Toyota Research Institute of North America\\
    Toyota Motor North America
}
\affil[2]{%
    Frontier Research Center\\
    Toyota Motor Corporation
}
% \institute{Toyota Research Institute of North America, Toyota Motor North America}
% \author{Anonymous Authors}
% \institute{}

\maketitle

\begin{abstract}
    % Probabilistic model checking for Markov decision processes (MDPs) provides quantitative guarantees for specifications, but often offers limited insight into why undesired outcomes occur. Probability-raising (PR) causality addresses this by identifying states whose visitation increases the probability of reaching designated effect states. Existing PR-cause methods, however, rely on MDP modifications that are ill-suited for learning: they compare conditional and unconditional reachability in a way that can make causal gaps hard to detect from transition samples, and they require reachability values of the original MDP during construction, which are unavailable in MDPs with unknown transition probabilities. We study unknown MDPs and propose a learning approach with probabilistic guarantees for PR-cause identification. Our key ingredient is a lightweight restart-based MDP modification that reduces PR-cause checking to two conditional reachability queries without relying on reachability probabilities in the original MDP. We prove the proposed modification correct, derive finite-sample guarantees and sample-complexity bounds, and develop an anytime learning-and-checking algorithm built upon two-sided value iteration. Experiments on two benchmarks demonstrate reliable and quick identification of PR causes.
    Probabilistic model checking for Markov decision processes (MDPs) provides quantitative guarantees, but often offers limited insight into why undesired outcomes occur. Probability-raising (PR) causality addresses this by identifying states whose visitation increases the probability of reaching designated states. Existing PR-cause identification methods, however, use MDP modifications not well-suited for learning: the gap between conditional and unconditional reachability probabilities can be hard to detect from transition samples, and construction requires reachability probabilities of the MDP, which are unavailable when transition probabilities are unknown. We study unknown MDPs and propose a learning approach with probabilistic guarantees for PR-cause identification. Our key ingredient is a \textit{restart-based MDP modification} that reduces PR-cause checking to two conditional reachability queries without using reachability values of the original MDP. We prove correctness, establish sample-complexity bounds, and develop an anytime learning-and-checking algorithm based on two-sided value iteration that progressively classifies states as causal, non-causal, or undecided. Experiments on two benchmarks demonstrate reliable and fast identification of PR causes.
\end{abstract}

\section{Introduction}
% Correct explanation and verification for stochastic systems like MDPs are important. However, verification does not provide in-depth information for improved behavior when the system does not meet the specification/task.
% A formal explanation of the intricate behavior of MDP is important to provide understandable and actionable insights for human users.
Probabilistic model checking (PMC) \citep{baier2008principles} is widely used to analyze Markov decision processes (MDPs) that capture probabilistic and nondeterministic behavior, such as that arising from permissive policies, randomized decisions, and human-related uncertainties in various applications \citep{drager2015permissive, junges2018model, chen2022multi}. 
% However, its outputs are often difficult to interpret. 
While PMC can quantify the probabilities of undesired outcomes, it typically does not explain why such a behavior happens \citep{baier2021verification}. 
% As a result, a verification result may fail to provide concise information on why it happens. 
In practice, one often wants to know not just the probability but which states are responsible for it: for instance, states whose visitation raises the violation probability no matter which action a permissive controller selects, or states whose traversal raises the satisfaction probability no matter how a human or other agent acts. Identifying such states helps localize the source of undesired/desired behavior.
Cause-effect reasoning aims to close this gap by connecting outcomes for given specifications to concrete system elements such as states or actions \citep{baier2024foundations, oura2025probability, kazemi2025causal, coenen2022temporal, triantafyllou2022actual} (see Section \ref{related_works:causal_explanation} for details).

% Cause-effect reasoning gathers a lot of attention in both machine learning and formal methods communities as a promising way to provide critical information about why the system issues such unsatisfactory behavior.
% PR causality has the advantage of a verified explanation for the system's behavior over other methods.
Probability-raising (PR) causality provides a verification-friendly explanation for probabilistic reachability in MDPs \citep{baier2022probability, baier2024foundations}. In this framework, a cause is a set of states whose visitation increases the probability of reaching designated states, such as rejecting states for logical specifications. The authors have also introduced algorithms based on model checking to compute PR causes under some optimality criteria, but existing approaches have presumed that the MDP is fully known. We focus on reachability since it is a fundamental property to which linear temporal logic specifications are reduced via a standard product construction \citep{baier2008principles}.

% Although PR cause computation is well-established, estimation of the causes for unknown MDPs has not been considered. We face several challenges to correctly estimating PR causes when we treat unknown MDPs: (i) ... (ii) ...
% PAC learning of MDP can be employed to overcome these challenges \cite{}. Controller synthesis, verification, etc., based on PAC-MDP learning has been advocated.
% Most related work is PAC model checking. They introduced a method for learning the upper/lower bounds for reachability probability. However, they merely consider verifying the system rather than computing the cause of the reliability property.
Learning-based verification provides probably approximately correct (PAC) guarantees for unknown stochastic systems \citep{ashok2019pac, agarwal2025pac, fu2014probably, perez2024pac} (see Section \ref{related_works:learning_verification} for details), but it does not directly yield cause identifications. 
% Two-sided bounds computed from transition samples can certify reachability probabilities, but they do not directly address cause identification.
In particular, they are not designed to decide which states (or state sets) are responsible for undesired behavior. Our work builds on a two-sided value iteration scheme \citep{ashok2019pac}, but targets certified cause identification rather than property verification alone.

Bringing PR-cause computation to unknown MDPs raises challenges not typically encountered when the model is fully known. The causation decision depends on state-specific probability gaps that are unknown in advance, making a single fixed-accuracy learning phase potentially insufficient or wasteful. Moreover, standard PR-cause pipelines rely on MDP transformations \citep{baier2022probability} that are ill-suited for learning. They compare a conditional reachability probability with an unconditional one in a way that can be statistically unfavorable. The unconditional quantity averages over runs that do and do not visit the candidate state, which can dilute the gap, especially because the policy optimization used for PR-cause checking tends to steer toward truly causal candidates. Finally, these transformations for PR-cause identification, as well as existing MDP transformations for computing conditional probabilities more generally, require reachability probabilities of the original MDP during construction \citep{baier2014computing, baier2022probability}, which are unavailable in unknown MDPs 
(see Section \ref{related_works:MDP_transformation} for details).

As a closely related and complementary work, \cite{oura2025probability} has studied PAC-guaranteed identification of PR causes using the existing MDP transformation for an uncertain MDP by instantiating transition probabilities. However, the authors do not consider finding a true PR cause in an unknown MDP.
% If one combines this uncertainty-set certification with learning the set from transition data and then seeks near-uniform and high-confidence guarantees over the learned set, the required number of sampled MDP instances can become huge, making this combination computationally intensive.

% In this paper, we propose ...
To address these challenges, we propose a learning approach for PR-cause identification with probabilistic guarantees. First, we introduce a restart-based MDP modification that reduces PR-cause checking to comparing reachability probabilities conditioned on visiting a candidate state and on bypassing it, and we prove its correctness (Proposition \ref{prop:soundness_completeness_restart_modification}). Second, we analyze and compare the sample-complexity bounds for learning PR causes with high confidence in unknown MDPs under the proposed modification and the existing one (Theorem \ref{thm:sample_efficiency}). Third, built upon two-sided value iteration, we present an anytime learning algorithm to progressively classify states into three classes (causal/non-causal/undecided). We prove its correctness (Theorems \ref{thm:Algorithm1_is_solution} and \ref{thm:Algorithm1_is_anytime}). Finally, we evaluate the proposed method on two benchmarks against two baselines: one combining \citep{baier2022probability} with \citep{ashok2019pac}, and the other combining learning of a PAC-guaranteed interval MDP \citep{suilen2022robust} with \citep{oura2025probability}.
\hspace{-3mm}

\section{Preliminaries}
We introduce MDPs and probabilistic causation in an MDP as a set of states, aiming to capture and explain cause-effect relationships for reachability properties on the MDP.

\subsection{Markov Decision Processes and Probability-Raising Causality}
A Markov decision process (MDP) is a tuple $\mathcal{M} = (S, A, s_I, P)$, where $S$ is a finite set of states, $A$ is a finite set of actions, $s_I \in S$ is an initial state, and $P : S \times A \times S \to [0, 1]$ is a transition probability function with $\sum_{s' \in S} P(s, a, s') \in \{0, 1\}$ for any $s \in S$ and any $a \in A$.

For any $s \in S$, we denote the set of actions enabled at $s$ as $\mathcal{A}(s) = \{ a \in A \;|\; \exists s' \mbox{ s.t. } P(s,a,s') > 0 \}$.  We say that $s \in S$ is terminal if $\mathcal{A}(s) = \emptyset$. We denote the set of terminal states by $E$, assuming $E \neq \emptyset$ and $s_I \not\in E$. A path is a finite or infinite sequence $s_0 a_0 s_1 \ldots $, where $P(s_i, a_i, s_{i+1}) > 0$ for any $s_i, s_{i+1} \in S$ and $a_i \in A$ with $i =0, 1, \ldots$. Intuitively, a path that reaches $E$ is considered undesirable.
A policy $\pi : \mathrm{FinPath} \times A \to [0,1]$ is a function that assigns a probability to each action for a finite path, where $\mathrm{FinPath}$ is the set of all finite paths. A policy $\pi$ is memoryless if $\pi(\rho, a) = \pi(\mathrm{last}(\rho), a)$ for any $\rho \in \mathrm{FinPath}$ and any $a \in A$, where $\mathrm{last}(\rho)$ denotes the last state of $\rho$.
For any MDP $\mathcal{M}$, any state $s \in S$, and any policy $\pi$, a unique probability measure $\mathrm{Pr}^\pi_{\mathcal{M}, s}$ on measurable sets of paths starting from $s$ under $\pi$ in a standard way \citep{baier2008principles}. For brevity, we write $\mathrm{Pr}^\pi_{\mathcal{M}}$ to denote $\mathrm{Pr}^\pi_{\mathcal{M}, s_I}$. We also denote the minimal and maximal probabilities over all policies by $\mathrm{Pr}^\mathrm{min}_{\mathcal{M}, s}$ and $\mathrm{Pr}^\mathrm{max}_{\mathcal{M}, s}$, respectively. We denote by $\mathrm{Post}_{s,a}$ the set of successor states of $s$ by executing $a$.
% \begin{assumption}[Successors and Transition Sampling]
% \label{assume:unknown_MDP}
    % For any MDP $\mathcal{M}$, we assume the following two properties.
    % \begin{enumerate}
        % \item For any $s \in S$ and any $a \in A$, $\mathrm{Post}_{s,a}$ is known.
        % \item For any $s \in S$ and any $a \in \mathcal{A}(s)$, we can generate a transition $(s,a,s')$ according to $P$.
    % \end{enumerate}
% \end{assumption}
We refer to an MDP as unknown if the successors $\mathrm{Post}_{s,a}$ are known for any $s \in S$ and $a \in A$, but its transition probability function is unknown. 

For convenience, we use linear temporal logic (LTL)-like notations such as $\neg$ (negation), $\land$ (and),  $\eventually$ (eventually), and $\mathrm{U}$ (until) to represent certain path properties. For any $X_0, X_1 \subseteq S$, the formula $X_0 \mathrm{U} X_1$ is satisfied by a (state) path $\rho = s_0 s_1 \ldots$ such that there exists $j \geq 0$ such that, for any $i < j$, $s_i \in X_0$ and $s_j \in X_1$ hold, and we denote $\eventually X_1 = S \mathrm{U} X_1$. Furthermore, $\neg \eventually X_1$ is satisfied by a path $\rho$ such that all states in $\rho$ are not contained in $X_1$. Throughout, when a set argument is a singleton, we omit the braces. For example, we write $\eventually c$ for $\eventually \{c\}$ and $X \mathrm{U} c$ for $X \mathrm{U} \{c\}$.
\begin{definition}
\label{def:SPR_cause}
    For any MDP $\mathcal{M}$, the set of terminal states $E \subseteq S$, and a nonempty subset $C \subseteq S \setminus E$, we say that $C$ is a strict probability-raising (SPR) cause for $E$ on $\mathcal{M}$ if and only if the following conditions \textbf{(C1)} and \textbf{(C2)} hold:
    \begin{description}
        \item[(C1)] $\forall c \in C, \forall \pi\mbox{ s. t. } \mathrm{Pr}_{\mathcal{M}}^\pi(\neg C \mathrm{U} c) \in (0, 1)$,
        \begin{align}
        \mathrm{Pr}_{\mathcal{M}}^\pi(\eventually E \;|\; \neg C \mathrm{U} c) > \mathrm{Pr}_{\mathcal{M}}^\pi(\eventually E ),
        \end{align}
        \item[(C2)] $\forall c \in C, \exists \pi$ s. t. $\mathrm{Pr}_{\mathcal{M}}^\pi(\neg C \mathrm{U} c) \in (0, 1)$.
    \end{description}
\end{definition}
Intuitively, $\eventually E$ denotes eventually reaching $E$, while $\neg C \mathrm{U} c$ means reaching $c$ before any other state in $C$. An SPR cause is a set $C$ of waypoint states that increases the probability of reaching $E$, as required by \textbf{(C1)}. We exclude policies under which $\mathrm{Pr}^\pi_\mathcal{M}(\neg C \mathrm{U} c) = 0$ (the conditional probability is undefined) or $1$ (the conditional equals the unconditional). \textbf{(C2)} ensures that \textbf{(C1)} is not satisfied vacuously: for each $c \in C$, there exists at least one policy under which this comparison in \textbf{(C1)} is non-trivial.
% removing states in $C$ that are always or never reached before others.
% Intuitively, $\eventually E$ denotes that a path eventually reaches a state in $E$, while $\neg C \mathrm{U} c$ means that a path eventually reaches $c$ without visiting any other states in $C$ beforehand.
% SPR cause is defined as a subset $C$ of states that act as waypoints increasing the probability of reaching terminal states $E$ under any policy. This probability-raising condition is defined by $\textbf{(C1)}$, implying that visiting any state in $C$ increases the probability of reaching $E$. We eliminate policies under which we always or never reach $c$ from consideration because such policies do not change the reachability probability. The non-vacuity is defined by $\textbf{(C2)}$, removing states in $C$ that can only be reached by passing through the others or that are always passed through along the way.
For each $c \in S$, we call $c$ a causal state if $\{c\}$ is an SPR cause, and a non-causal state otherwise.

\begin{example}
\label{example:MDP_SPRcause}
    We consider an MDP depicted in Fig. \ref{fig:example_6state_MDP}. The sets of actions and terminal states are $A = \{a,b\}$ and $E = \{s_5\}$, respectively. The only enabled action is $a$ except for $s_0$. $C = \{s_2, s_3\}$ forms an SPR cause in this MDP for the following reasons. First, there exist two paths $s_0 a s_2$ and $ s_0a s_1a s_3$. Thus, $C$ satisfies the condition $\mathbf{(C2)}$.
    Second, we have $\mathrm{Pr}(\eventually s_5 \mid \eventually s_2) = 0.6$ and $\mathrm{Pr}(\eventually s_5 \mid \eventually s_3) = 1$ while $\mathrm{Pr}^\mathrm{max}_\mathcal{M}(\eventually s_5 ) = 0.48$. Hence, $C$ also satisfies $\mathbf{(C1)}$.
\end{example}
Concrete application scenarios for Def. \ref{def:SPR_cause}, including permissive robotic planning with unsafe states and linear temporal logic planning against an adversary, are discussed in Appendix \ref{appendix:usecases}.

We consider lower and upper bounds for reachability probabilities.
For any MDP $\mathcal{M}$, we say that $\underline{x}^\mathrm{opt}, \overline{x}^\mathrm{opt}: S \to [0,1]$ are lower and upper bounds for the reachability probabilities for the set of terminal states $E$ if and only if
    \begin{align}
        \forall s \in S, \underline{x}^\mathrm{opt}(s) \leq \mathrm{Pr}^\mathrm{opt}_{\mathcal{M}, s}(\eventually E) \leq \overline{x}^\mathrm{opt}(s),
    \end{align}
where $\mathrm{opt} \in \{ \min, \max \}$.
Moreover, for any multi-sample of transitions $\mathcal{D}$ from $\mathcal{M}$ and any $\delta \in (0,1)$, we consider a lower bound $\hat{P}_\delta : S \times A \times S \to [0,1]$ for the transition probability $P$ as, for any $(s,a,s') \in S \times A \times S$,
\begin{align}
\label{lower_bound_transition_prob}
    \hat{P}_\delta(s,a,s') =\!
    \begin{cases}
         \max (0, \frac{N(s,a,s')}{N(s,a)} - \sqrt{\frac{- \ln{\delta}}{2 N(s,a)}}) \\
          & \hspace{-43mm} \mbox{ if } N(s,a) > 0 \land |\mathrm{Post}_{s,a}| > 1, \\
        1 & \hspace{-43mm} \mbox{ if } \mathrm{Post}_{s,a} = \{s'\}, \\
        0 & \hspace{-43mm} \mbox{ otherwise},
    \end{cases}
\end{align}
where $N(s,a,s')$ denotes the number of times the transition $(s,a,s')$ appears in $\mathcal{D}$ and $N(s,a) = \sum_{s' \in S} N(s,a,s')$. We denote $\hat{P}_\delta \leq P$ if and only if $\hat{P}_\delta(s,a,s') \leq P(s,a,s')$ for any $(s,a,s') \in S \times A \times S$. Intuitively, $\hat{P}_\delta$ is derived based on Hoeffding's inequality \citep{ashok2019pac}.

We consider $S^\mathrm{min}_{\mathcal{M}} \!=\! \{ s \in S \mid \mathrm{Pr}^\mathrm{min}_{\mathcal{M}, s}(\eventually E) = 0 \}$ and $S^\mathrm{max}_{\mathcal{M}} \!=\! \{ s \in S \mid \mathrm{Pr}^\mathrm{max}_{\mathcal{M}, s}(\eventually E) = 0 \}$ for any MDP $\mathcal{M}$.
% \begin{align}
    % S^\mathrm{min}_{\mathcal{M},0} = \{ s \in S \setminus E \;|\; \mathrm{Pr}^\mathrm{min}_\mathcal{M}(s \models \eventually E) = 0 \}, \\
    % S^\mathrm{max}_{\mathcal{M},0} = \{ s \in S \setminus E \;|\; \mathrm{Pr}^\mathrm{max}_\mathcal{M}(s \models \eventually E) = 0 \}.
% \end{align}
Note that both $S^\mathrm{min}_{\mathcal{M}}$ and $S^\mathrm{max}_{\mathcal{M}}$ can be computed in polynomial time in the number of states and actions \citep{baier2008principles}.
Based on the lower-bound estimate $\hat{P}_\delta$ of the transition probability function, we consider \textit{modified} Bellman operators that compute lower/upper bounds of the reachability probability as follows \citep{ashok2019pac}.

\begin{definition}
\label{def:bounds_operator_reachability}
    For any MDP $\mathcal{M}$, any multi-sample of transitions $\mathcal{D}$, any $\delta \in (0,1)$, we define operators $f^{\mathrm{opt}^{(1)}, \mathrm{opt}^{(2)}}_{\mathcal{M}, \delta}: [0,1]^S \to [0,1]^S$ such that $f^{\mathrm{opt}^{(1)}, \mathrm{opt}^{(2)}}_{\mathcal{M}, \delta}[x](s) = 1$ for any $s \in E$, $f^{\mathrm{opt}^{(1)}, \mathrm{opt}^{(2)}}_{\mathcal{M}, \delta}[x](s) = 0$ for any $s \in S^{\mathrm{opt}^{(1)}}_{\mathcal{M}}$, and, for any $s \in S \setminus (E \cup S^{\mathrm{opt}^{(1)}}_{\mathcal{M}})$,
    \begin{align}
        \label{opt12_update}
        &f^{\mathrm{opt}^{(1)}, \mathrm{opt}^{(2)}}_{\mathcal{M}, \delta}[x](s) := \mathrm{opt}^{(1)}_{a \in \mathcal{A}(s)} \sum_{s' \in \mathrm{Post}_{s,a}} \! \hat{P}_\delta(s,a,s') x(s') \nonumber \\
        & \hspace{5.9mm} + \mathrm{opt}^{(2)}_{s' \in \mathrm{Post}_{s,a}} x(s') (1 - \sum_{s' \in \mathrm{Post}_{s,a}}\hat{P}_\delta(s,a,s') ),
    \end{align}
    where $\hat{P}_\delta$ is defined by (\ref{lower_bound_transition_prob}) for $\mathcal{M}$ and $\mathrm{opt}^{(1)}, \mathrm{opt}^{(2)} \in \{ \min, \max \}$.
\end{definition}
Intuitively, repeatedly applying $f^{\mathrm{opt}, \mathrm{min}}_{\mathcal{M}, \delta}$ and $f^{\mathrm{opt}, \mathrm{max}}_{\mathcal{M}, \delta}$ to a function $x: S \to [0,1]$ provide lower and upper bounds for $\mathrm{Pr}^{\mathrm{opt}}_{\mathcal{M}, \cdot}(\eventually E)$.
% \setlength{\intextsep}{0pt}
% \begin{wrapfigure}{r}{0.4\textwidth}
        % \centering
        % \includegraphics[width=0.38\textwidth]{For_Paper/Figures/Example_MDP.pdf}
        % \caption{MDP in Example \ref{example:MDP_SPRcause}, where $E=\{s_5\}$. The action enabled at states except for $s_0$ is only $a$ ad hence we abbreviate it.}
        % \label{fig:example_6state_MDP}
% \end{wrapfigure}

\begin{figure}
    \centering
        \includegraphics[width=0.27\textwidth]{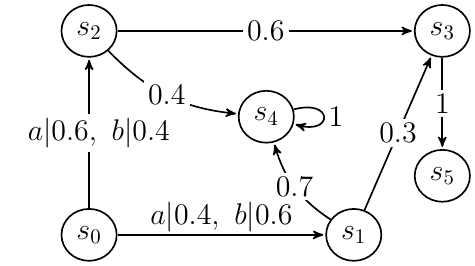}
        \caption{MDP in Example \ref{example:MDP_SPRcause} with  $E=\{s_5\}$. The action enabled at states except for $s_0$ is only $a$, and hence we abbreviate it.}
        \label{fig:example_6state_MDP}
\end{figure}

\subsection{Problem Setting}
Our goal is to identify an SPR cause that maximally covers undesired paths with high probability on a given unknown MDP. To formalize this intuition, we pose the following problem by employing the same path coverage objective as in prior works \citep{baier2022probability, oura2025probability}.
\begin{problem}
\label{problem}
    Given an unknown MDP $\mathcal{M}$, for any $\delta_\mathcal{M} > 0$, design an algorithm that finds a subset of states $C^* \subset S$ such that, w. p. at least $1-\delta_\mathcal{M}$, $C^*$ is an SPR cause for $E$ on $\mathcal{M}$ and, for any SPR cause $C$ for $E$ on $\mathcal{M}$,
    \begin{align}
    \label{PAC-PR-recall_optimal}
        \min_{\pi} \mathrm{Pr}_{\mathcal{M}}^\pi (\eventually C^* \mid \eventually E) - \mathrm{Pr}_{\mathcal{M}}^\pi (\eventually C \mid \eventually E) \!\geq\! 0.
    \end{align}
\end{problem}
Intuitively, we aim to find an SPR cause whose coverage for undesired path, that is $\mathrm{Pr}^\pi_\mathcal{M}(\eventually C^* \mid \eventually E)$, is maximal under all possible policies among all SPR causes. In the following, we do not mention the set of terminal states $E$ for any MDP when it is clear from the context.

\section{Cause Learning Methodology}
% \begin{figure}[htbp]
  	% \centering
    % \includegraphics[width=0.9\linewidth]{For_Paper/Figures/alg_flowchart_v.png}
        % \label{ex:path_planning}
  	% \caption{Overview of proposed method. This corresponds to Algorithm \ref{alg:PACIdentification} as follows: sampling (L. 4), bound computation (L. 5-13), causality check (L. 14 and 15), classification and termination (L. 17-19), and computing $C^*$ (L. 20).}
    % \label{fig:alg_summary}
% \end{figure}
In this section, we propose an anytime learning method for identifying PR causes in unknown MDPs. The method builds on a proposed MDP modification tailored to learning. For each candidate causal state, we redirect its outgoing transitions to the initial state, reducing PR-cause checking to two conditional reachability queries. This modification does not require reachability probabilities of the original MDP and makes the relevant probability gaps easier to detect from samples. Building on it, we develop a learning-and-checking loop that repeatedly tightens lower and upper bounds on reachability probabilities for both the original and modified MDPs. As these bounds tighten, more candidate states are classified into three classes (causal/non-causal/undecided), leaving fewer states unclassified, yielding an anytime procedure with confidence guarantees.
% The overview of the proposed method is summarized in Fig. \ref{fig:alg_summary}.

% \subsection{Reachability Probability Bounds and MDP Modification for SPR Cause Learning}

\subsection{MDP Modifications for SPR Cause Learning in Unknown MDPs}
We first review an existing MDP modification introduced in \citep{baier2022probability, baier2024foundations} for computing SPR causes and point out that it can be sample-inefficient when combined with MDP learning. We then propose a restart-based MDP modification that enables more sample-efficient learning of SPR causes while preserving correctness.

The existing transformation rewrites the transitions outgoing from a specified state $c$ so that the transition probability to a state in $E$ equals $\mathrm{Pr}^\mathrm{min}_{\mathcal{M}, c}(\eventually E)$, and the transition probability to an added sink state equals $1 - \mathrm{Pr}^\mathrm{min}_{\mathcal{M}, c}(\eventually E)$.
\begin{definition}[Minimal-Reachability-based Transformation for MDPs \citep{baier2022probability, baier2024foundations}]
\label{def:existing_MDP_modification}
    For any MDP $\mathcal{M}$ and any $c \in S$, we define a transformed MDP as a tuple $\mathcal{M}_\mathrm{mr}^{[c]} = (S\cup \{s_\mathrm{noeff}\}, A, s_I, P^c_\mathrm{mr})$ such that $P^c_\mathrm{mr}(s_\mathrm{noeff}, a, s_\mathrm{noeff}) = 1$ for any $a \in A$, $P^c_\mathrm{mr}(s,a,s') = P(s,a,s')$ for any $(s,a,s') \in S \setminus \{c\} \times A \times S$, and $P^c_\mathrm{mr}(c, a, s_\mathrm{eff}) = \mathrm{Pr}^\mathrm{min}_{\mathcal{M}, c}(\eventually E)$ and $P^c_\mathrm{mr}(c, a, s_\mathrm{noeff}) = 1 - \mathrm{Pr}^\mathrm{min}_{\mathcal{M}, c}(\eventually E)$ for any $a \in \mathcal{A}(c)$, where $s_\mathrm{eff}$ is a fixed state in $E$. 
\end{definition}
Def. \ref{def:existing_MDP_modification} allows us to check \textbf{(C1)} in Def. \ref{def:SPR_cause} for each $c \in S$ via $\mathrm{Pr}^\mathrm{min}_{\mathcal{M}, c}(\eventually E) - \mathrm{Pr}^\mathrm{max}_{\mathcal{M}^{[c]}_\mathrm{mr}}(\eventually E)$. To see this, recall that in $\mathcal{M}^{[c]}_\mathrm{mr}$ the outgoing transitions of $c$ are rewritten so that $c$ leads to $E$ exactly with probability $\mathrm{Pr}^\mathrm{min}_{\mathcal{M}, c}(\eventually E)$ and otherwise to the sink $s_\mathrm{noeff}$, which never reaches $E$. A policy attaining $\mathrm{Pr}^\mathrm{max}_{\mathcal{M}^{[c]}_\mathrm{mr}}(\eventually E)$ therefore reaches $E$ as much as possible before reaching $c$ and, after reaching $c$, steers to $E$ with the minimal probability. Comparing this maximal reachability against $\mathrm{Pr}^\mathrm{min}_{\mathcal{M}, c}(\eventually E)$, the reachability after a visit to $c$, realizes the test of \textbf{(C1)}. However, when $c$ is causal, a policy attaining $\mathrm{Pr}^\mathrm{max}_{\mathcal{M}^{[c]}_\mathrm{mr}}(\eventually E)$ typically steers toward $c$ with high probability, making the reachability probability gap between the original and transformed MDPs small and requiring many transition samples to determine whether $c$ is causal. Moreover, Def. \ref{def:existing_MDP_modification} requires the (unknown) maximal reachability probabilities at construction time; thus, we must, e.g., replace transition probabilities from $c$ in $\mathcal{M}_\mathrm{mr}^{[c]}$ with estimates, which in turn needs additional samples to guarantee correctness.
To address these issues, we introduce a restart-based modification that removes the effect of the probability of reaching $c$ from the initial state and avoids using estimated reachability probabilities during construction.
\begin{definition}[Restart-based Modification for MDPs]
\label{def:restart_based_modification}
    For any MDP $\mathcal{M}$ and any $c \in S$, a modified MDP is a tuple $\mathcal{M}^{[c]} = (S, A, s_I, P^c)$ such that $P^c(s,a,s') = P(s,a,s')$ for any $(s,a,s') \in S\setminus \{c\} \times A \times S$, and $P^c(c, a, s_I) = 1$ for any $a \in \mathcal{A}(c)$. 
\end{definition}
Intuitively, the transition from $c$ is reset to the initial state. Hence, this enables us to evaluate the reachability probability conditioned on bypass of $c$ in $\mathcal{M}$ by the unconditional reachability probability on $\mathcal{M}^{[c]}$ as indicated in Lemma \ref{lemma:conditioned_reachability_to_reachability_modified_MDP}.
\begin{lemma}
\label{lemma:conditioned_reachability_to_reachability_modified_MDP}
    For any MDP $\mathcal{M}$, any $c \in S \setminus E$, and any policy $\pi$ with $\mathrm{Pr}^\pi_{\mathcal{M}}(\neg \eventually c) > 0$, there exists a policy $\tilde{\pi}$ on $\mathcal{M}^{[c]}$ such that we have
    \begin{align}
        \mathrm{Pr}^{\tilde{\pi}}_{\mathcal{M}^{[c]}}(\eventually E) = \mathrm{Pr}^\pi_{\mathcal{M}}(\eventually E \mid \neg \eventually c).
    \end{align}
    Moreover, if $\pi$ is memoryless, then we can choose $\tilde{\pi} = \pi$.
\end{lemma}
The identity holds because Def. \ref{def:restart_based_modification} restarts every path reaching $c$ back to $s_I$, redistributing the contribution of $c$-visiting paths to fresh paths that never visit $c$.
% that is, we have
% \begin{align}
    % \mathrm{Pr}^\pi_\mathcal{M}(\eventually E \;|\; \neg \eventually c) = \mathrm{Pr}^\pi_{\mathcal{M}^{[c]}}(\eventually E),
% \end{align}
% for any memoryless policy $\pi$ such that $\mathrm{Pr}^\pi_\mathcal{M}(\neg \eventually c) > 0$ (see Lemma \ref{lemma:conditioned_reachability_to_reachability_modified_MDP} in Appendix \ref{appendix:proofs}).
Proof of Lemma \ref{lemma:conditioned_reachability_to_reachability_modified_MDP} is in Appendix \ref{appendix:proofs}. 
Moreover, we explicitly avoid relying on any reachability probabilities on $\mathcal{M}$ to construct $\mathcal{M}^{[c]}$.
% \begin{remark}
    % A two-step MDP modification introduced in \citep{baier2014computing} also relies on a restart-mechanism to compute conditional reachability probabilities. However, their construction requires the respective (unknown) maximal reachability probabilities of the condition and the task, but our modification does not. Hence, ours is more suitable for learning.
% \end{remark}
% We show the modified MDP with $c = s_2$ for the MDP in Example \ref{example:MDP_SPRcause} in Fig. \ref{fig:example_6state_MDP_modified}. The restart transition is illustrated by the red arrow.
% \begin{figure}[htbp]
    % \centering
        % \includegraphics[width=0.3\textwidth]{For_Paper/Figures/Example_MDP_restart.pdf}
        % \caption{The modified MDP with $c = s_2$ for Example \ref{example:MDP_SPRcause}.}
        \label{fig:example_6state_MDP_modified}
% \end{figure}

Proposition~\ref{prop:soundness_completeness_restart_modification} ensures that states are classified correctly as causal/non-causal even when using Def.~\ref{def:restart_based_modification} in place of Def.~\ref{def:existing_MDP_modification}.
Noting that $\eventually c$ equals $\neg c \mathrm{U} c$, to capture the policies for which the probability comparison in \textbf{(C1)} of Def.~\ref{def:SPR_cause} is non-trivial, for any $c \in S$, we define
\begin{align}
\label{policy_set_M_c}
    \Pi_{\mathcal{M}, c} = \{ \pi \mid \mathrm{Pr}^\pi_\mathcal{M}(\eventually c) \in (0,1)\}.
\end{align}
\begin{proposition}[Correctness of Restart-based Modification]
    \label{prop:soundness_completeness_restart_modification}
    For any MDP $\mathcal{M}$ and any $c \in S$, if $\Pi_{\mathcal{M}, c} = \emptyset$, then $c$ is not causal, otherwise we have the following properties:
    \begin{enumerate}
        \item If $\mathrm{Pr}^\mathrm{min}_{\mathcal{M}, c}(\eventually E) > \mathrm{Pr}^\mathrm{max}_{\mathcal{M}^{[c]}}(\eventually E)$, then $c$ is causal.
        \item If $\mathrm{Pr}^\mathrm{min}_{\mathcal{M}, c}(\eventually E) < \mathrm{Pr}^\mathrm{max}_{\mathcal{M}^{[c]}}(\eventually E)$, then $c$ is not causal.
        \item If $\mathrm{Pr}^\mathrm{min}_{\mathcal{M}, c}(\eventually E) = \mathrm{Pr}^\mathrm{max}_{\mathcal{M}^{[c]}}(\eventually E)$, then:
        \begin{enumerate}
            \item if $c$ is not reachable from $s_I$ in $\mathcal{M}_\mathrm{max}^{[c]}$, $c$ is causal,
            \item otherwise, $c$ is not causal,
        \end{enumerate}
    \end{enumerate}
    where $\mathcal{M}^{[c]}$ is defined by Def. \ref{def:restart_based_modification} and $\mathcal{M}_\mathrm{max}^{[c]}$ is the sub-MDP of $\mathcal{M}^{[c]}$ induced by $\mathcal{A}^\mathrm{max}(s) = \{ a \in A \mid \mathrm{Pr}^\mathrm{max}_{\mathcal{M}^{[c]}, s}(\eventually E) = \sum_{s' \in S} P^c(s,a,s') \mathrm{Pr}^\mathrm{max}_{\mathcal{M}^{[c]}, s'}(\eventually E) \}$.
\end{proposition}
At a high level, the cases contrast best-policy reachability after visiting $c$ with worst-policy reachability when bypassing $c$ to decide \textbf{(C1)}. In the equality case with $c$ reachable, some policy makes the two conditional probabilities coincide and thus violates \textbf{(C1)}; otherwise \textbf{(C1)} holds.
Proof for Proposition \ref{prop:soundness_completeness_restart_modification} is in Appendix \ref{appendix:proofs}
\begin{remark}
    We mention that the soundness of the restart-based modification remains even if we instead use an existential quantifier for policies in \textbf{(C1)} of Def. \ref{def:SPR_cause}. Please see Proposition \ref{prop:soundness_completeness_restart_modification_for_existential} in Appendix \ref{appendix:proofs}.
\end{remark}
% Proposition \ref{prop:soundness_completeness_restart_modification} allows us to judge if $c$ is causal or not using value iterations on the original MDP and the modified MDP. 

% Our MDP modification reduces the number of transition samples required for state classification compared to Def. \ref{} for the following two reasons. (i) We directly compare $\mathrm{Pr}^\pi_{\mathcal{M}}(\eventually E \mid \eventually c)$ and $ \mathrm{Pr}^\pi_{\mathcal{M}}(\eventually E \mid \neg \eventually c)$, resulting avoid the influence of $( 1 - \mathrm{Pr}^\pi_{\mathcal{M}}(\eventually c))$. (ii) Our modification explicitly avoids relying on maximal reachability probabilities.

For any $\delta \in (0,1)$, we introduce a lower bound $\hat{P}^c: S \times A \times S \to [0,1]$ for $P^c$ of $\mathcal{M}^{[c]}$ such that, for any $(s,a,s') \in S \times A \times S$, 
% $\hat{P}^c_\delta(s,a,s') = \hat{P}_\delta(s,a,s')$ if $s \neq \hat{c}$, $\hat{P}^c_\delta(s,a,s') = 1$ if $s = \hat{c}$ and $s' = s_I$ with $a \in \mathcal{A}(\hat{c})$, and otherwise $0$, where $\hat{P}_\delta$ is defined by (\ref{lower_bound_transition_prob}) for $\mathcal{M}$ with $\mathcal{D}$.
\begin{align}
\label{P_hat_modified_reset}
    \hat{P}^c_\delta(s,a,s') = 
    \begin{cases}
        \hat{P}_\delta(s,a,s')  & \mbox{if } s \neq c,\\
        1 & \mbox{if } s= c, s' = s_I, \\
        0 & \mbox{otherwise, }
    \end{cases}
\end{align}
where $\hat{P}_\delta$ is defined by (\ref{lower_bound_transition_prob}) for $\mathcal{M}$ with $\mathcal{D}$.

We will analyze the number of transition samples required to correctly classify a state as causal or non-causal under both MDP constructions. Our modification can substantially reduce the required number of samples compared with a straightforward adaptation of Def. \ref{def:existing_MDP_modification} to PAC model checking via Def. \ref{def:bounds_operator_reachability}. We first introduce the mixing time of an MDP.
\begin{definition}
\label{def:mixing_time}
    For any MDP $\mathcal{M}$, any policy $\pi$, any $s \in S$, and any $\epsilon > 0$, $\epsilon$-value mixing time from $s$ is $T^\pi_{\mathcal{M}, \epsilon, s} = \min\{ T \mid |\mathrm{Pr}^\pi_{\mathcal{M}, s}(\eventually^{\leq T} E) - \mathrm{Pr}^\pi_{\mathcal{M}, s}(\eventually E)| \leq \epsilon \}$, where $\eventually^{\leq T} E$ denotes the event of reaching $E$ within $T$ steps.
\end{definition}
Intuitively, $T^\pi_{\mathcal{M}, \epsilon, s}$ denotes the minimal time $T$ such that, under $\pi$, the difference between the probability of eventually reaching $E$ and the probability of reaching $E$ within $T$ steps is less than $\epsilon$. For simplicity, let $T^\mathrm{opt}_{\mathcal{M}, \epsilon, s} = T^\pi_{\mathcal{M}, \epsilon, s}$ for any $\pi$ with $\mathrm{Pr}^\pi_{\mathcal{M}}(\eventually E) = \mathrm{Pr}^\mathrm{opt}_{\mathcal{M}}(\eventually E)$, where $\mathrm{opt} \in \{ \mathrm{min}, \mathrm{max} \}$.

Let $\underline{x}^{\mathrm{opt}, *}_\mathcal{M}$ and $\overline{x}^{\mathrm{opt}, *}_\mathcal{M}$ be fixed points of $f^{\mathrm{opt}, \mathrm{min}}_\mathcal{M}$ and $f^{\mathrm{opt}, \mathrm{max}}_\mathcal{M}$, respectively. We denote the maximum number of successors by $d_\mathrm{max} = \max_{(s,a) \in (S \setminus E) \times A} |\mathrm{Post}_{s,a}|$ and the set of probabilistic transitions by $\mathrm{Tr}(\mathcal{M}) = \{ (s,a,s') \in S \times A \times S \mid P(s,a,s') \in (0,1) \}$, respectively. Similar to (\ref{P_hat_modified_reset}), for any $\delta \in (0,1)$, we define a plug-in lower-bound estimate $\hat{P}^c_{\mathrm{mr}, \delta}: S \times A \times S \to [0,1]$ of $P^c_\mathrm{mr}$ in $\mathcal{M}^{[c]}_\mathrm{mr}$ by setting $\hat{P}^c_{\mathrm{mr}, \delta}(s,a,s') = \hat{P}_\delta(s,a,s')$ for any $s \neq c$, and $\hat{P}^c_{\mathrm{mr}, \delta}(c, a, s_\mathrm{eff}) = \underline{x}^{\mathrm{min}, *}_\mathcal{M}(c)$ and $\hat{P}^c_{\mathrm{mr}, \delta}(c, a, s_\mathrm{noeff}) = 1 - \overline{x}^{\mathrm{min}, *}_\mathcal{M}(c)$ for any $a \in \mathcal{A}(c)$.
\begin{theorem}[Sample Complexity]
    \label{thm:sample_efficiency}
    For any MDP $\mathcal{M}$, any random $N$ transition samples $\mathcal{D}$, any $\delta_\mathcal{M}$, and any $c \in S$, if $\Delta_c := \mathrm{Pr}^\mathrm{min}_{\mathcal{M}, c}(\eventually E) - \mathrm{Pr}^\mathrm{max}_{\mathcal{M}^{[c]}}(\eventually E) \neq 0$ and, for any $(s,a) \in (S \setminus E) \times A$, 
    \begin{align}
    \label{required_samples_our_modification}
        N(s,a) \geq K \frac{d_\mathrm{max}^2 T^2}{\Delta_c^2} \ln \frac{2|\mathrm{Tr(\mathcal{M})|}}{\delta_\mathcal{M}},
    \end{align}
    then $\underline{x}^{\mathrm{min}, *}_\mathcal{M}(c) > \overline{x}^{\mathrm{max}, *}_{\mathcal{M}^{[c]}}(s_I)$ or $\overline{x}^{\mathrm{min}, *}_\mathcal{M}(c) < \underline{x}^{\mathrm{max}, *}_{\mathcal{M}^{[c]}}(s_I)$ w. p. at least $1 - \delta_\mathcal{M}$, moreover, if $\Delta_{\mathrm{mr},c} := \mathrm{Pr}^\mathrm{min}_{\mathcal{M}, c}(\eventually E) - \mathrm{Pr}^\mathrm{max}_{\mathcal{M}^{[c]}_\mathrm{mr}}(\eventually E) \neq 0$ and, for any $(s,a) \in (S \setminus E) \times A$, 
    \begin{align}
    \label{required_samples_existing_modification}
        N(s,a) \geq K \frac{d_\mathrm{max}^4 T_\mathrm{mr}^4}{\Delta_{\mathrm{mr},c}^2} \ln \frac{2|\mathrm{Tr(\mathcal{M})|}}{\delta_\mathcal{M}},
    \end{align}
    then $\underline{x}^{\mathrm{min}, *}_\mathcal{M}(c) \!>\! \overline{x}^{\mathrm{max}, *}_{\mathcal{M}^{[c]}_\mathrm{ex}}(s_I)$ or $\overline{x}^{\mathrm{min}, *}_\mathcal{M}(c) \!<\! \underline{x}^{\mathrm{max}, *}_{\mathcal{M}^{[c]}_\mathrm{mr}}(s_I)$ w. p. at least $1 - \delta_\mathcal{M}$, where $T \!=\! \max \{T^\mathrm{min}_{\mathcal{M}, \frac{|\Delta_c|}{K_1}, c}, T^\mathrm{max}_{\mathcal{M}^{[c]}, \frac{|\Delta_c|}{K_1}, s_I} \}$, 
    $T_\mathrm{mr} \!=\! \mathrm{min} \{ T \mid T \geq \mathrm{max} \{T^\mathrm{min}_{\mathcal{M}, \epsilon, c}, T^\mathrm{max}_{\mathcal{M}^{[c]}_\mathrm{mr}, \epsilon, s_I} \} \}$ with $\epsilon = |\Delta_{\mathrm{mr},c}|/(K_1 (T d_\mathrm{max} + 1))$, and $K, K_1 > 0$ are constants.
\end{theorem}
Proof for Theorem \ref{thm:sample_efficiency} is in Appendix \ref{appendix:proofs}.
Combined with Proposition \ref{prop:soundness_completeness_restart_modification}, the respective lower bound in (\ref{required_samples_our_modification}) and (\ref{required_samples_existing_modification}) acts as a number of samples required for correct classification of states as causal or non-causal.

\begin{remark}[Comparison of Sample Complexities]
% Although both bounds (\ref{required_samples_our_modification}) and (\ref{required_samples_existing_modification}) in Theorem \ref{thm:sample_efficiency} involve the common mixing time $T^{\min}_{\mathcal{M},\cdot,c}$,
The restart-based transformation $\mathcal{M}^{[c]}$ can yield more favorable sample-complexity bound than
the existing modification $\mathcal{M}^{[c]}_{\mathrm{mr}}$.
First, $\mathcal{M}^{[c]}_{\mathrm{mr}}$ embeds the unknown reachability probability
$\Pr^{\min}_{\mathcal{M},c}(\Diamond E)$ into the transition probabilities from $c$.
So, estimation errors enter twice through learning the original transitions and
through plugging an estimated reachability probability back into the transformed MDP.
This explains the stronger $d_{\max}^4T_{\mathrm{mr}}^4$ scaling,
compared with $d_{\max}^2T^2$ for $\mathcal{M}^{[c]}$.
Second, this plug-in construction imposes a stricter finite-horizon accuracy requirement,
so that $T_{\mathrm{mr}}$ is evaluated at a smaller tolerance
($( T_\mathrm{mr} d_\mathrm{max} + 1) \epsilon = |\Delta_{\mathrm{mr},c}|/K_1$ instead of $\epsilon = |\Delta_c|/K_1$), which can further increase $T_{\mathrm{mr}}$.
Finally, when $c$ is causal, a policy that achieves
$\Pr^{\max}_{\mathcal{M}^{[c]}_{\mathrm{mr}}}(\Diamond E)$ tends to steer towards $c$,
which can severely reduce $\Delta_{\mathrm{mr},c}$,
whereas $\mathcal{M}^{[c]}$ emphasizes the gap between reaching $E$ when passing through $c$ and when bypassing $c$.
% Formal details are in Appendix.
\end{remark}

% To ensure the convergence of the upper bounds in the non-trivial and non-bottom MECs to the true fixed points, we define a deflating operation as follows.
% \begin{definition}
% \label{deflating_operation}
    % For any MDP $\mathcal{M}$, any function $x : S \to [0,1]$, and a subset $S'$ of $S$, we define $d: [0,1]^S \to [0,1]^S$ such that
    % \begin{align}
        % d[x](s) = \min(x(s), \max_{s' \in \partial S'}x(s')),
    % \end{align}
    % where $\partial S' = \{s' \not\in S' \;|\; \exists (s,a,s') \mbox{ s.t. } P(s,a,s')>0 \land s \in S'\}$.
% \end{definition}

\subsection{Probably Approximately Correct SPR Cause Learning}
\label{subsection:learning_algorithm}
We provide a solution to Problem \ref{problem} in Algorithm \ref{alg:PACIdentification}. We first precompute a set of predetermined non-causal states (Line 2), then iteratively collect $k$ transition samples and update the lower/upper bounds for $\mathrm{Pr}^\mathrm{min}_{\mathcal{M}, s}(\eventually E)$ and $\mathrm{Pr}^\mathrm{max}_{\mathcal{M}^{[c]}}(\eventually E)$ (Lines 4-6). Using Proposition \ref{prop:soundness_completeness_restart_modification}, each remaining candidate state $c$ is progressively assigned to $C_\top$ or $C_\bot$ when the corresponding bound intervals certify the ordering (Lines 7-12). With introducing a margin $\tau \geq 0$, the procedure supports a three-way outcome (causal/non-causal/undecided): for $\tau > 0$, a still-unclassified state $c$ is assigned to $C_?$ once the current confidence interval of its reachability probability gap is entirely included by $[-\tau, \tau]$, capturing near-tie cases. In practice, we optionally stop bound updates for $c$ and decide its class as soon as the compared intervals become disjoint, since the ordering is then fixed by monotonicity of (\ref{opt12_update}) (see Lemma \ref{lemma:monotone_continuos} in Appendix \ref{appendix:proofs}). We describe a practical initialization and an early stopping scheme in Appendix \ref{appendix:initializtion_early_stopping}. After all candidates have been assigned, we re-check the classifications using Lines 6-12 for all classified states; % $c \in C_\top \cup C_\bot \cup C_?$; 
if $(C_\top, C_\bot, C_?)$ is unchanged, we return $C^*$ as the ``front'' of $C_\top$ to satisfy Def. \ref{def:SPR_cause}, and otherwise we resume learning.

\begin{algorithm}[h]
\caption{PAC Cause Identification}
\label{alg:PACIdentification}
\KwIn{Unknown MDP $\mathcal{M}$, $\delta_\mathcal{M} > 0$, $k > 0$, $\tau \geq 0$.}
\KwOut{$C^* \subset S$}

$ \mathcal{D}, C_\top, C_\bot, C_? \gets \emptyset $, and $ \delta \gets \frac{\delta_\mathcal{M}}{\mathrm{Tr}(\mathcal{M})} $.\\
$ C^\mathsf{pre}_\bot \gets \{ c \in S \mid \Pi_{\mathcal{M}, c} = \emptyset \} \cup E$.\\ 
% Compute MEC decomposition $S = \biguplus_{k=1}^K S_k \uplus \{t_l\}_{l=1}^L \uplus \biguplus_{m=1}^M B_m \uplus E$.
\Repeat{$S \setminus (C_\top \cup C_\bot \cup C_? \cup C^\mathrm{pre}_\bot) = \emptyset$}{
    % $ \mathcal{D}' \gets \mathrm{TransSample}(\mathcal{M}, k) $ and $ \mathcal{D} \gets \mathcal{D} \cup \mathcal{D}' $.\\
    Sample $k$ transitions in $\mathcal{M}$ and add them to $\mathcal{D}$.\\
    % \tcp{Initialize bounds}
    Initialize lower and upper bounds: $[\underline{x}^\mathrm{min}_\mathcal{M},  \overline{x}^\mathrm{min}_\mathcal{M}]$, and for all $c \in S \setminus C^\mathrm{pre}_\bot$, $[\underline{x}^\mathrm{max}_{\mathcal{M}^{[c]}},  \overline{x}^\mathrm{max}_{\mathcal{M}^{[c]}}]$. \\
    Compute the fixed points for (\ref{opt12_update}).\\
    % $\underline{x}^{\mathrm{min}, *}_\mathcal{M}, \overline{x}^{\mathrm{max}, *}_{\mathcal{M}^{[c]}}, \overline{x}^{\mathrm{min},*}_\mathcal{M}, \underline{x}^{\mathrm{max},*}_{\mathcal{M}^{[c]}} $ for (\ref{opt12_update}) with the initialized bounds.\\
    % $\forall s,c \in S,\ x^l(s) \gets 0,\ x^u(s) \gets 1, y^l_c(s) \gets 0,\ y^u_c(s) \gets 1$.\\
    % $\forall s \in E, \forall s' \in S^\mathrm{min}_{\mathcal{M}},\ x^l(s) \gets 1,\ x^u(s') \gets 0$.\\
    % $\forall c \in S, \forall s \in E, \forall s' \in S^\mathrm{max}_{\mathcal{M}^{[c]}}, \ y^l_c(s) \gets 1,\ y^u_c(s') \gets 0$.\\
    % $\forall s \in S \setminus ( E \cup S^\mathrm{min}_{\mathcal{M},0} ),\ x^l(s) \gets 0,\ x^u(s) \gets 1$.\\
    \For{$c \in S \setminus (C_\top \cup C_\bot \cup C_? \cup C^\mathrm{pre}_\bot)$}{
        Add $c$ to $C_\top$ if $\underline{x}^{\mathrm{min},*}_\mathcal{M}(c) > \overline{x}^{\mathrm{max},*}_{\mathcal{M}^{[c]}}(s_I) $. \\
        Add $c$ to $C_\bot$ if $\overline{x}^{\mathrm{min},*}_\mathcal{M}(c) < \underline{x}^{\mathrm{max},*}_{\mathcal{M}^{[c]}}(s_I) $. \\
        % \If{$\exists c_\top \in C_\top$ s.t. $\forall \rho, \rho \models \eventually c_\top \Leftrightarrow \eventually c$}{
            % Add $c$ to $C_\top$.
        % }
        \If{$[\underline{x}^{\mathrm{min},*}_\mathcal{M}(c) - \overline{x}^{\mathrm{max},*}_{\mathcal{M}^{[c]}}(s_I) , \overline{x}^{\mathrm{min},*}_\mathcal{M}(c) - \underline{x}^{\mathrm{max},*}_{\mathcal{M}^{[c]}}(s_I)] \subseteq [-\tau, \tau]$ and $c \not\in C_\top \cup C_\bot$}{
            $\tau > 0$: Add $c$ to $C_?$.\\
            $\tau = 0$: Add $c$ to $C_\top$ if 3-(a) in Proposition \ref{prop:soundness_completeness_restart_modification} is satisfied, otherwise add $c$ to $C_\bot$.
        }
    }
}
\If{ $(C_\top, C_\bot, C_?)$ remains unchanged under $\mathcal{D}$}{
    \textbf{Return} $C^* = \{ c \in C_\top \mid \exists \rho \text{ s.t. } \rho \models \neg C_\top \mathrm{U} c \}$.   
}
Go back to Line 4 with $C_\top, C_\bot, C_? \gets \emptyset$.

\end{algorithm}

In the following Example \ref{example:PACIdentification_MDP}, we show how well Algorithm \ref{alg:PACIdentification} finds an SPR cause that satisfies (\ref{PAC-PR-recall_optimal}).
\begin{example}
\label{example:PACIdentification_MDP}
    We applied Algorithm \ref{alg:PACIdentification} to the MDP in Example \ref{example:MDP_SPRcause}. Let $\delta_\mathcal{M} = 0.05$ so that the found SPR cause is guaranteed w. p. at least $0.95$, $\tau=0$ to classify all states, and sample $k=100$ transitions at each iteration by selecting $k$ state-action pairs uniformly at random. We conducted $30$ experimental runs. In every run, we consistently found $s_3$ as a causal state first and then obtained $C^* = \{s_2, s_3\}$. This is because $s_3$ has a higher difference in probabilities in terms of \textbf{(C1)} in Def. \ref{def:SPR_cause}. Algorithm \ref{alg:PACIdentification} converged within $130$ iterations per run.
\end{example}

Theorem \ref{thm:Algorithm1_is_solution} shows that Algorithm \ref{alg:PACIdentification} is a solution to Problem \ref{problem}. Theorem \ref{thm:Algorithm1_is_anytime} guarantees an anytime correctness, that is, $C_\top$ and $C_\bot$ in Algorithm \ref{alg:PACIdentification} are always a correct set of causal and non-causal states, respectively, and $C_?$ contains only states with near-tie cases, with a high probability.

\begin{theorem}
\label{thm:Algorithm1_is_solution}
    For any MDP $\mathcal{M}$, any $\delta_\mathcal{M} > 0$, and any $\tau \geq 0$, $C^*$ obtained from Algorithm \ref{alg:PACIdentification} satisfies (\ref{PAC-PR-recall_optimal}) for all SPR causes included by $S \setminus C_?$ w. p. at least $1 - \delta_\mathcal{M}$. Moreover, if $\tau = 0$, Algorithm \ref{alg:PACIdentification} is a solution to Problem \ref{problem}.
\end{theorem}
% Obviously, if $\tau = 0$, Algorithm \ref{alg:PACIdentification} is a solution to Problem \ref{problem}

\begin{theorem}[Anytime Correctness]
\label{thm:Algorithm1_is_anytime}
    For any MDP $\mathcal{M}$, any $\delta_\mathcal{M} > 0$, any $\tau \geq 0$, and any multi-sample $\mathcal{D}$ of transitions, $C_\top$, $C_\bot$, and $C_?$ computed in Algorithm \ref{alg:PACIdentification}, respectively, satisfy the following properties:
    \begin{enumerate}
        \item $\forall c \in C_\top, \{c\} \models \mathbf{(C1)}$ w. p. at least $1 - \delta_\mathcal{M}$,
        \item $\forall c \in C_\bot, \{c\} \not \models \mathbf{(C1)}$ w. p. at least $1 - \delta_\mathcal{M}$,
        \item $\forall c \in C_?, |\mathrm{Pr}^\mathrm{min}_{\mathcal{M}}(\eventually E | \eventually c) - \mathrm{Pr}^\mathrm{max}_\mathcal{M}(\eventually E | \neg \eventually c)| \leq \tau $ w. p. at least $1 - \delta_\mathcal{M}$.
        % \item $C^*$ satisfies (\ref{PAC-PR-recall_optimal}).
    \end{enumerate}
\end{theorem}
Proofs for Theorems \ref{thm:Algorithm1_is_solution} and \ref{thm:Algorithm1_is_anytime} are in Appendix \ref{appendix:proofs}.
Almost sure convergence to a true SPR cause that satisfies (\ref{PAC-PR-recall_optimal}) when the amount of data goes to infinity is also guaranteed under some conditions (see Propositions \ref{prop:fixed_point} and \ref{prop:convergence_property} in Appendix \ref{appendix:proofs}).

\begin{remark}[On Weaker Assumptions for MDPs]
The above guarantees rely on knowing the successor supports of $\mathcal{M}$, but this assumption can be relaxed. Def. \ref{def:restart_based_modification} uses none of $c$'s successor support: it resets every action at $c$ to $s_I$ regardless of $c$'s original outgoing transitions, and it uses no reachability value of $\mathcal{M}$. Hence, the construction is unaffected if we weaken ``successor supports are known'' to, e.g., ``only a lower bound on the minimum nonzero transition probability is known'' \citep{ashok2019pac}. Algorithm \ref{alg:PACIdentification} is still applicable through a black-box two-sided value iteration \citep{ashok2019pac}. In contrast, Def. \ref{def:existing_MDP_modification} requires reachability information of the original MDP, so weakening the assumption loosens its estimated bounds and further hampers the separation check for \textbf{(C1)}.
\end{remark}

\section{Examples}
\begin{figure}[htbp]
  	\centering
  	\begin{subfigure}{0.45\linewidth}
  		\includegraphics[width=\linewidth]{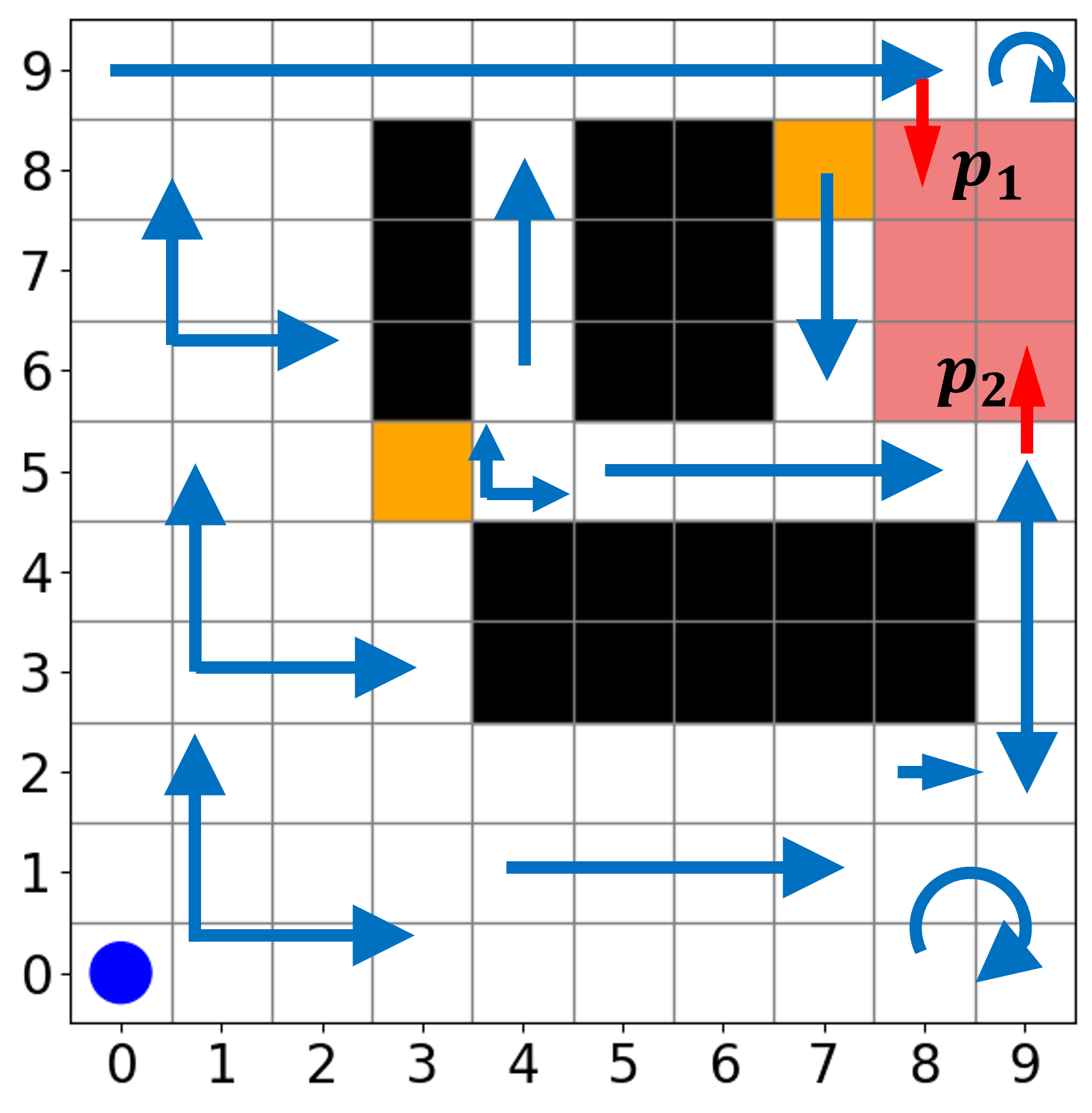}
        \caption{Nondeterministic planning}
        \label{ex:path_planning}
    \end{subfigure}
  	\begin{subfigure}{0.45\linewidth}
  		\includegraphics[width=\linewidth]{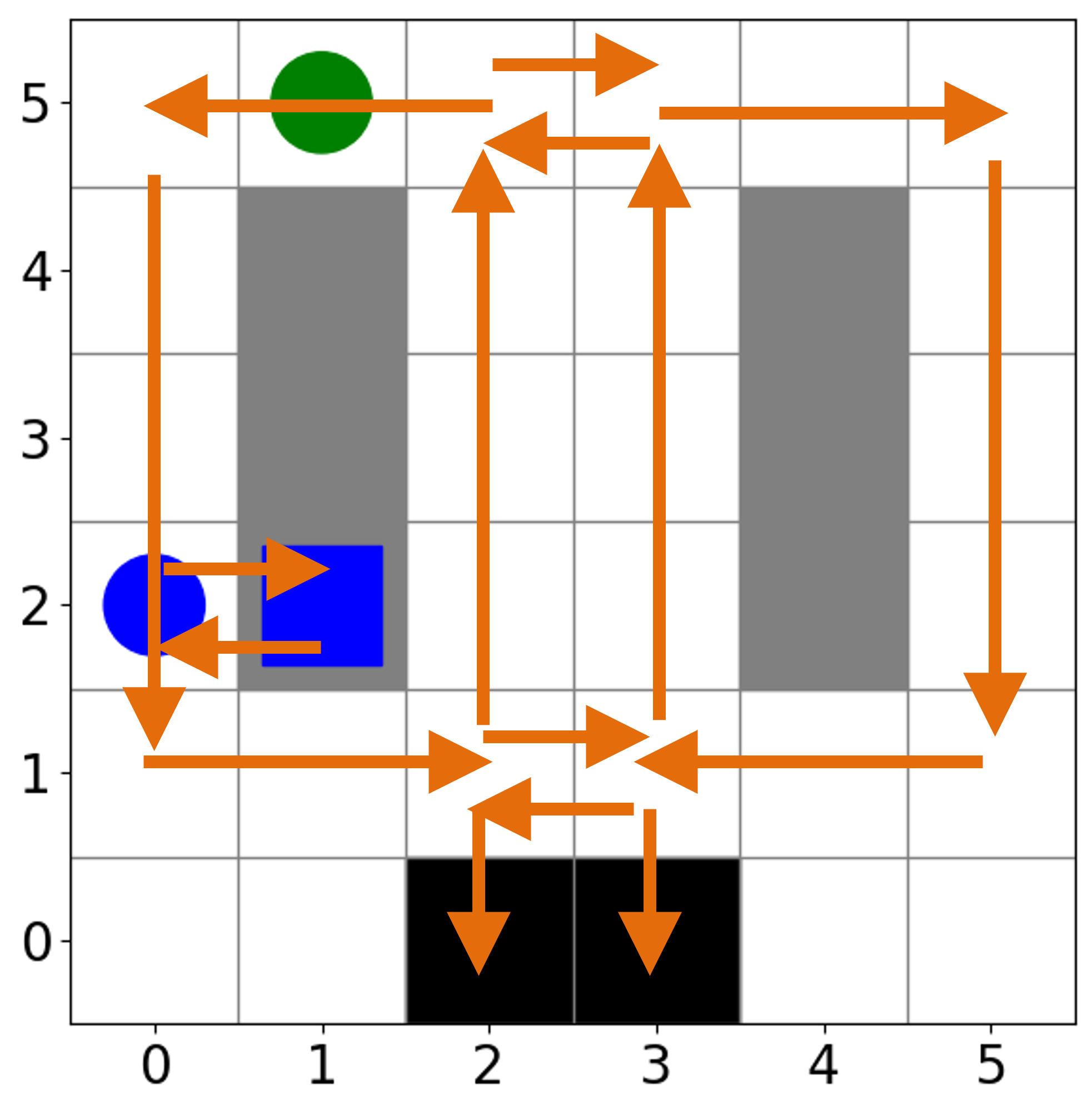}
        \caption{Warehouse delivery}
  		\label{ex:warehouse}
  	\end{subfigure}
  	\caption{ (a) Nondeterministic planning and (b) warehouse delivery. Blue/orange arrows denote enabled actions; black cells are obstacles (a) or goals (b); red cells are the undesired set $E$, entered with the probabilities $p_1$ and $p_2$. Orange cells in (a) mark the SPR cause satisfying \eqref{PAC-PR-recall_optimal}.}
    % (a) A grid world with black obstacles. The robot starts at $(0,0)$ and enters the terminal red cells via the red-arrow transitions with probabilities $p_1$ and $p_2$. The blue arrows denote the actions enabled in each region. The orange cells indicate the SPR cause satisfying (\ref{PAC-PR-recall_optimal}) when $p_1 < p_2$. (b) A warehouse with two robots. The blue robot delivers the blue item to the goal position (black cell) while avoiding collisions with the green robot. The orange arrows denote the enabled movement actions. The two robots are jointly controlled by a deterministic RL controller.}
\end{figure}

We demonstrate how correctly and quickly the proposed method identifies causal states in two environments: (i) a path planning scenario with nondeterministic actions and (ii) a warehouse delivery example under a reinforcement learning controller, as depicted in Figs.\ \ref{ex:path_planning} and \ref{ex:warehouse}. We implemented Algorithm \ref{alg:PACIdentification} in Python 3.11 with the initialization and the early stopping described in Appendix \ref{appendix:initializtion_early_stopping}, and all experiments ran on a computer equipped with Intel Core i9-14900HX and 62GB of RAM. We compared our method against two baselines.
$\mathrm{Baseline1}$ runs Algorithm 2 in \citep{baier2022probability}, but compares lower/upper bounds on reachability probabilities computed using \citep{ashok2019pac} from $k$ sampled transitions, as in Algorithm \ref{alg:PACIdentification}. It replaces the transition probabilities in Def. \ref{def:existing_MDP_modification} with the estimated lower bounds and repeats the procedure until all states are classified. An algorithmic summary of \textrm{Baseline1} is provided in Appendix \ref{appendix:baseline1}.
$\mathrm{Baseline2}$ first learns intervals for transition probabilities from $k$ transition samples such that $P$ is contained in the intervals w. p. at least $1-\delta_\mathcal{M}$ \citep{suilen2022robust}. It then runs Algorithm 1 in \citep{oura2025probability} by instantiating 10 MDPs from these intervals. 
We treat an experiment as timed out if the iteration count exceeds $2 \times 10^4$.
For all methods, we sample $k$ state–action pairs uniformly at random and draw transitions. This sampling scheme is an experimental choice, and we could instead use an exploration policy.

\subsection{Setting for Path Planning}
\label{Example:path-planning}
In the path-planning environment (Fig. \ref{ex:path_planning}), the available actions are $\mathsf{Up}$, $\mathsf{Right}$, $\mathsf{Down}$, $\mathsf{Left}$, and $\mathsf{Stay}$; entering the red cells is undesired. The robot moves in the intended direction with probability $0.9$ and moves in a perpendicular direction with probability $0.05$ each. It remains in the same cell if it chooses $\mathsf{Stay}$ or if the intended move would hit an obstacle (black cells) or leave the grid. When adjacent to walls or obstacles, actions that would move into them are disabled. When the robot is at $(9,5)$ (resp., $(8,9)$), it enters the red cells with probability $p_1$ (resp., $p_2$) and moves to the intended cell with the remaining probability. In the cells adjacent to the red cells, except for $(9,5)$ and $(8,9)$, the robot moves as intended with probability $1$. We conducted 20 runs for each baseline and our method. We set $k = 5 \times 10^4$, $\delta_\mathcal{M} = 0.05$, and $(p_1, p_2) \in {(0.2, 0.7), (0.4, 0.5)}$. Experiments with other parameter settings are in Appendix \ref{appendix:further_experimental_results}.

In this example, when $p_1 < p_2$, the true SPR cause is the orange cells shown in Fig. \ref{ex:path_planning}. Intuitively, traversing these cells makes the robot more likely to reach the red cells from below via the middle corridor than via alternative routes. Our method gradually identifies these cells by comparing upper and lower bounds on reachability probabilities.

\subsection{Setting for Warehouse Delivery}
\label{Example:warehouse}
% As shown in Fig. \ref{ex:warehouse}, two robots (blue and green disks) navigate a warehouse. Using an automata-guided approach \citep{oura2020reinforcement}, they jointly learn policies to deliver the item (blue square) assigned to the blue robot to one of the goal locations (black cells) while avoiding collisions. 
In the warehouse environment (Fig. \ref{ex:warehouse}), two robots jointly learn policies via an automata-guided approach \citep{oura2020reinforcement} to deliver the item (blue square) to a goal (black cells) while avoiding collisions, which constitute the undesired behavior.
The enabled actions $\mathsf{Up}$, $\mathsf{Right}$, $\mathsf{Down}$, $\mathsf{Left}$, and $\mathsf{PickUp}$ for each region are indicated by orange arrows. The $\mathsf{PickUp}$ action is available only when a robot is at the item location. To encourage picking up the item, the blue robot is forced to choose $\mathsf{Right}$ with probability $0.4$ when adjacent to the item location. Detailed parameter settings are provided in Appendix \ref{appendix:experimental_settings}.
The undesired behavior is a collision. When not carrying an item, each robot moves in the intended direction with probability $1$. When carrying the item, the robot moves as intended with probability $0.6$ and stays in place with probability $0.4$, reflecting the slowdown due to the item's weight. The state space consists of $2808$ states. We set $k = 5 \times 10^5$, $\delta_\mathcal{M} = 0.1$, and conducted $20$ runs for both our method and the baselines.

Intuitively, collisions are more likely when the green robot is on the left side and behind the blue robot while it is carrying the item, since the green robot can catch up with positive probability. In contrast, if the green robot stays on the right side of the warehouse while the blue robot is carrying the item, the collision probability is $0$.

\subsection{Results}
\begin{table*}[t]
  \caption{Means and standard deviations (stds.) of the accuracies for the identified causal states, non-causal states, and final SPR causes, as well as of the number of iterations each method required to obtain the final outputs. The left and middle columns report results for the first nondeterministic planning example, and the right column reports results for the warehouse example. In all nondeterministic planning experiments, \textbf{\textrm{Baseline1} timed out}.}
  \centering
  \scalebox{0.72}{
  \label{table:correctness_path_planning}
    \begin{tabular}{ c||c c c c|c c c c|c c c c} \hline
       & \multicolumn{4}{c|}{Path planning with $(p_1, p_2)=(0.4, 0.5)$} & \multicolumn{4}{c|}{Path planning with $(p_1, p_2) = (0.2, 0.7)$} & \multicolumn{4}{c}{Warehouse} \\ \hline 
      & $C_\top$ & $C_\bot$ & $C^{*}$ & Iter & $ C_\top$ & $C_\bot$ & $C^{*}$ & Iter & $ C_\top$ & $C_\bot$ & $C^{*}$ & Iter \\ \hline \hline
       Ours & $1.0 \pm 0.0$ & $1.0 \pm 0.0$ & $1.0 \pm 0.0$ & $1449 \pm 188$ & $1.0 \pm 0.0$ & $1.0 \pm 0.0$ & $1.0 \pm 0.0$ & $168 \pm 5$ & $1.0 \pm 0.0$ & $1.0 \pm 0.0$ & $1.0 \pm 0.0$ & $57 \pm 11$ \\ 
       \hline
       $\mathrm{Baseline1}$ & $1.0 \pm 0.0$ & $1.0 \pm 0.0$ & -- & TO & $1.0 \pm 0.0$ & $1.0 \pm 0.0$ & -- & TO & $1.0 \pm 0.0$ & $1.0 \pm 0.0$ & $1.0 \pm 0.0$ & $86 \pm 12$ \\ 
       \hline
       $\mathrm{Baseline2}$ & -- & -- & $0.25 \pm 0.09$ & $55 \pm 69$ & -- & -- & $0.95 \pm 0.05$ & $18 \pm 2$ & -- & -- & $1.0 \pm 0.0$ & $5 \pm 4$ \\ 
       \hline
    \end{tabular}
    }
\end{table*}
       % $(p_1, p_2)$ & $(0.7, 0.2)$ & $(0.5, 0.4)$ & $(0.7, 0.2)$ & $(0.5, 0.4)$ & $(0.7, 0.2)$ & $(0.5, 0.4)$ \\ \hline \hline
      % Ours1 & $1.0 \pm 0.0$ & $1.0 \pm 0.0$ & $1.0 \pm 0.0$ & $1.0 \pm 0.0$ & $1.0 \pm 0.0$ & $1.0 \pm 0.0$ \\ 
      % Ours2 & $1.0 \pm 0.0$ & $1.0 \pm 0.0$ & $1.0 \pm 0.0$ & $1.0 \pm 0.0$ & $1.0 \pm 0.0$ & $1.0 \pm 0.0$ \\ 
      % Baseline & $1.0 \pm 0.0$ & $1.0 \pm 0.0$ & $0.96 \pm 0.21$ & $0.33 \pm 0.47$ & $0.95 \pm 0.21$ & $0.25 \pm 0.43$ \\ 
      % \hline
    % \end{tabular}
    % }
% \end{table*}
\paragraph{Performance comparisons.} Our method successfully and quickly identified causal states during the learning process and produced SPR causes that satisfy (\ref{PAC-PR-recall_optimal}) for both examples, whereas the baselines failed to do so. 
% To evaluate the correctness of Algorithm \ref{alg:PACIdentification}, we define the indicator function $\mathbb{I}: 2^S \times 2^S \to \{0, 1\}$ as $\mathbb{I}(C, C') = 1$ if $C \subseteq C'$, and otherwise $0$. 
% To evaluate the correctness, we consider the mean number of times all states in $C_\top$ and $C_\bot$ are causal and non-causal, respectively, and $C^{*}$ is an SPR cause that satisfies (\ref{PAC-PR-recall_optimal}), across all iterations and experiments.
% , where $C^{i, j}_\top$ and $C^{i, j}_\bot$ denote the sets of states that are classified as causal and non-causal, respectively, and $C^{*, j}$ is an output from our method and the baselines, at the $i$-th iteration in the $j$-th experiment. 
Table \ref{table:correctness_path_planning} reports the accuracy that states in $C_\top / C_\bot$ are truly causal/non-causal, and that $C^*$ is an SPR cause satisfying (\ref{PAC-PR-recall_optimal}); results are from our method with $\tau=0$ and each baseline. Across all experiments, every intermediate state labeled causal/non-causal by our method was correct, and the output $C^*$ was also correct. $\mathrm{Baseline1}$ also produced correct labels, but required far more samples; in nondeterministic planning, all runs timed out and left $(3,5)$ and $(4,5)$ unclassified. This occurred because, in the transformed MDP in Def. \ref{def:existing_MDP_modification}, a policy maximizing reachability probability to $E$ tends to take paths reaching $(3,5)$ (and similarly $(4,5)$), making the reachability gap between the original and transformed MDPs extremely small; our method mitigates this issue.
$\mathrm{Baseline2}$ converged quickly, but when $(p_1,p_2)=(0.4,0.5)$ its success frequency fell below $1-\delta_\mathcal{M}$. In some generated MDPs, transitions from $(9,5)$ to the red cell had smaller/larger probabilities under \textsf{Up}/\textsf{Down} than those from $(7,9)$, leading to the incorrect SPR cause $\{(3,5),(7,5)\}$. In contrast, the warehouse delivery example does not have nondeterministic actions, and the gap of reachability probabilities under a visit to a causal state and the bypass of it is not tiny. So, causal states are easier to detect. 
Moreover, for $\mathrm{Baseline2}$, enforcing that both the lower bound of probability over the learned intervals that a returned set of states is an SPR cause and its confidence are near $1$, e.g., $1-10^{-6}$, requires about $1.4\times 10^{7}$ MDP instances by Theorem 1 in \citep{oura2025probability}, which is infeasible since checking all states for one instantiated MDP took at least $10$ seconds in both environments. Further experimental results are in Appendix \ref{appendix:further_experimental_results}.

\paragraph{Convergence comparisons.}
Figs. \ref{fig:causal_noncausal_comparison_54} and \ref{fig:causal_noncausal_comparison_72} show, for the nondeterministic planning scenario, the numbers of causal and non-causal states correctly identified at each iteration by our method and \textrm{Baseline1}. 
For readability, we plot the results only up to 2000 and 300 iterations in Figs. \ref{fig:causal_noncausal_comparison_54} and \ref{fig:causal_noncausal_comparison_72}, respectively. 
Overall, our method identifies states earlier than \textrm{Baseline1}, and the advantage is more pronounced for causal states; particularly, when $(p_1, p_2) = (0.4, 0.5)$, the causal states are identified substantially earlier. Moreover, even when $(p_1, p_2) = (0.4, 0.5)$, our method classified all but two causal states into $C_\top$ within $300$ iterations; only $(3,5)$ and $(4,5)$ remain unclassified.

\paragraph{Effects of $\tau$ and early stopping.} 
Figs. \ref{fig:effect_of_tau_54} and \ref{fig:effect_of_tau_72} show that, 
% we show how the number of undecided states, $|C_?|$, and the number of iterations to convergence vary with the margin $\tau$ in the nondeterministic planning scenario. 
as $\tau$ increases, $|C_?|$ increases while the number of iterations decreases, enabling faster convergence at the expense of more undecided states. In particular, when $\tau \geq 0.1$ and $(p_1, p_2) = (0.2, 0.7)$, our method converged in fewer than 30 iterations, with $|C_?|$ at most five at termination. For all $\tau$, the causal/non-causal state classifications were correct. In Table \ref{table:early_stopping}, we compare the wall-clock time over the first 100 iterations with and without early stopping for bound updates. Early stopping was about $1.2\times$ faster on average and preserved the correctness of causal/non-causal state classification in all experiments.

\paragraph{Scalability.} Our method is amenable to scaling. PR-cause checking for each state can be parallelized, since each candidate $c$ is classified independently, and the interval value iteration can be GPU-accelerated \citep{farrington2025going}. Both reduce wall-clock time without altering the algorithm. Realizing such parallel and GPU-accelerated implementations is left as future work.
% This indicates that several states have a large gap between the two conditional reachability probabilities.
\begin{figure}[htbp]
    \centering
    \begin{subfigure}{0.47\linewidth}
        \includegraphics[width=\linewidth]{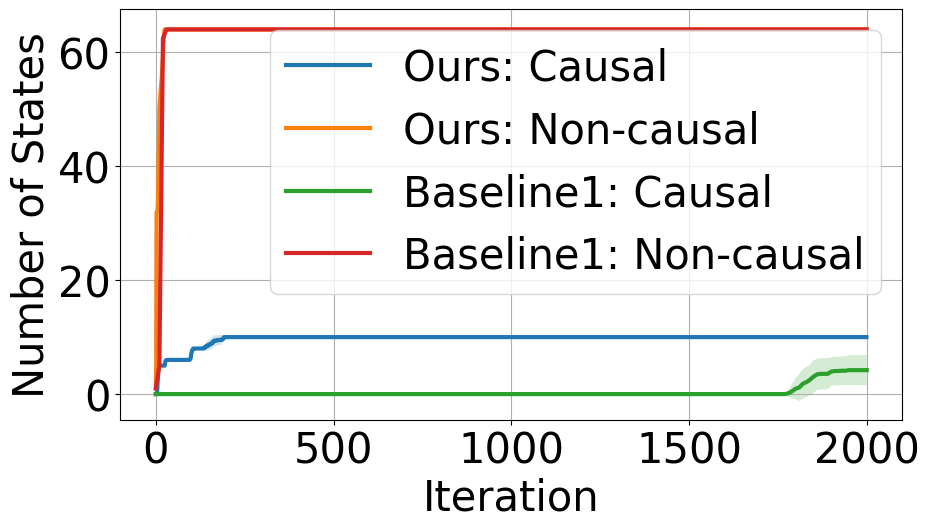}
        \caption{$(p_1, p_2) = (0.4, 0.5)$}
        \label{fig:causal_noncausal_comparison_54}
    \end{subfigure}
    \begin{subfigure}{0.47\linewidth}
        \includegraphics[width=\linewidth]{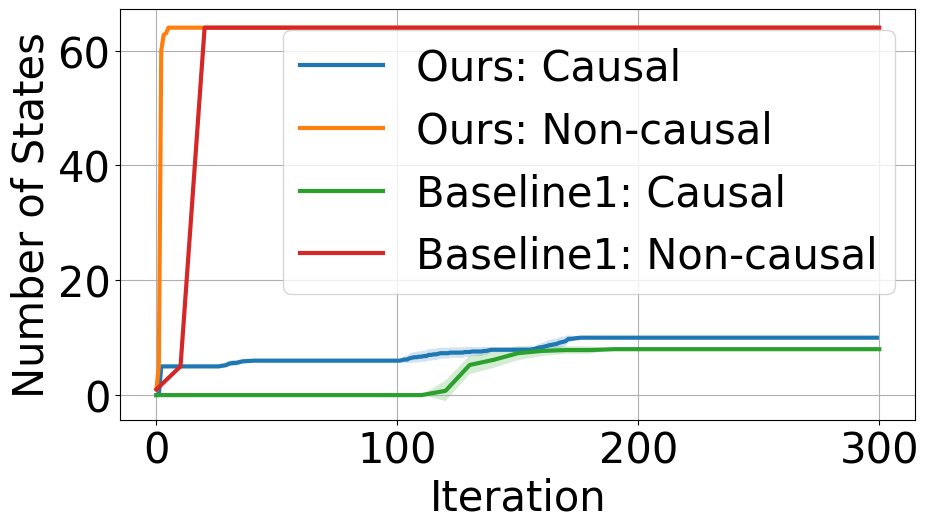}
        \caption{$(p_1, p_2) = (0.2, 0.7)$}
        \label{fig:causal_noncausal_comparison_72}
    \end{subfigure}
    \begin{subfigure}{0.49\linewidth}
        \includegraphics[width=\linewidth]{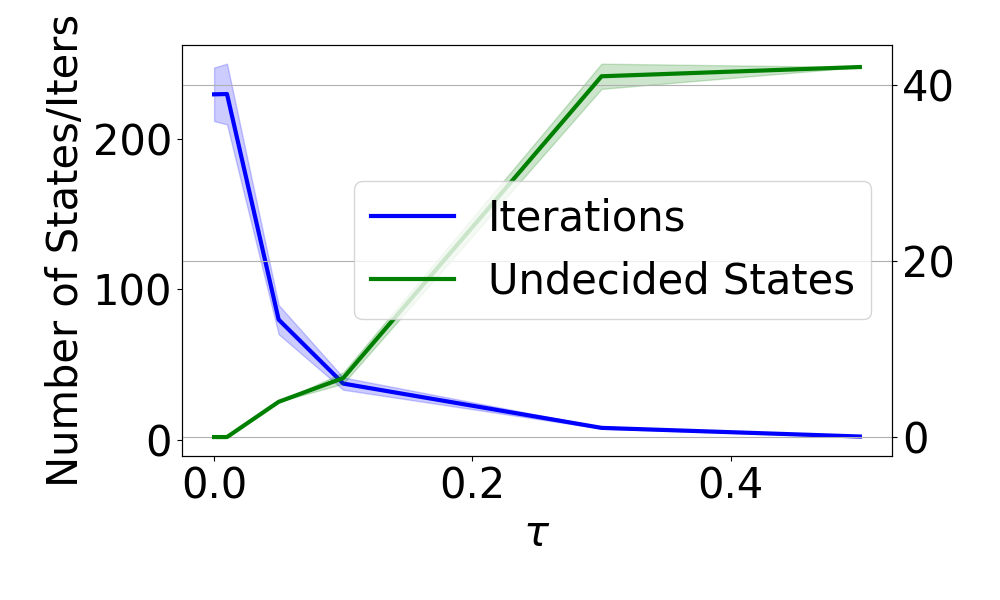}
        \caption{$(p_1, p_2) = (0.4, 0.5)$}
        \label{fig:effect_of_tau_54}
    \end{subfigure}
    \begin{subfigure}{0.49\linewidth}
        \includegraphics[width=\linewidth]{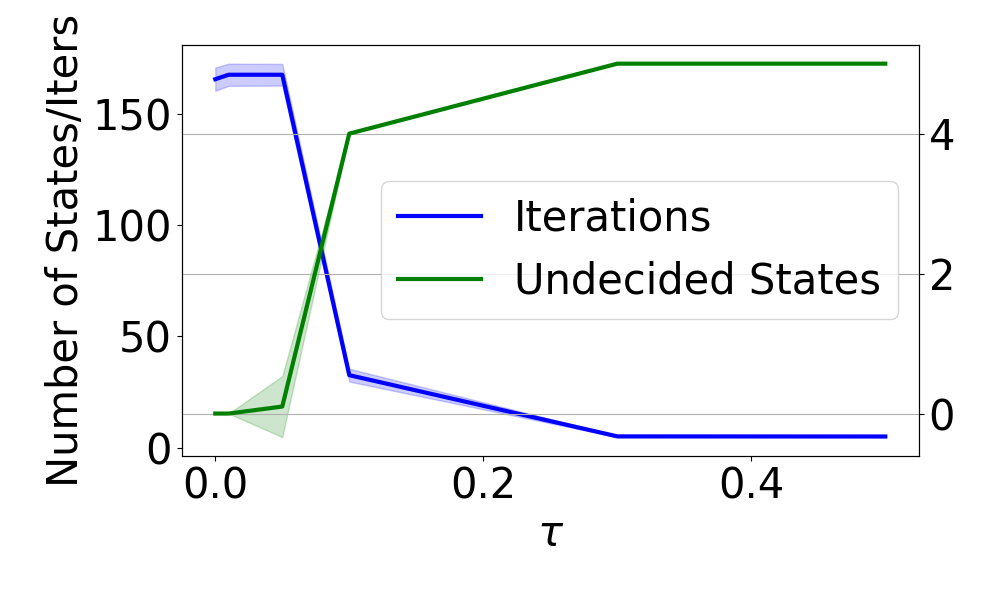}
        \caption{$(p_1, p_2) = (0.2, 0.7)$}
        \label{fig:effect_of_tau_72}
    \end{subfigure}
    \caption{Nondeterministic planning. (a) and (b): mean numbers of causal/non-causal states identified by our method and \textrm{Baseline1}. (c) and (d): mean numbers of $|C_?|$ and iteration count vs. $\tau$. Shaded areas show standard deviation; see Fig \ref{fig:causal_noncausal_comparison_54_delta_05}-\ref{fig:effect_of_tau_72_delta_05} for more details.}
    % (a) and (b) show, for the nondeterministic planning scenario, the mean numbers of causal and non-causal states identified by our method and \textrm{Baseline1}, respectively. (c) and (d) show the mean number of undecided states, $|C_?|$, and the mean number of iterations for our method as a function of $\tau$. In (a)–(d), shaded areas indicate the standard deviation.}
\end{figure}

\begin{table}[htbp]
  \caption{Means and standard deviations of the measured wall-clock times over the first 100 iterations, with and without the early stopping (ES) described in Sec. \ref{subsection:learning_algorithm}.}
  \centering
  \scalebox{0.8}{
  \label{table:early_stopping}
    \begin{tabular}{ c||c|c} \hline
       & $(p_1, p_2)=(0.4, 0.5)$ & $(p_1, p_2) = (0.2, 0.7)$ \\ \hline \hline
       Ours with ES & $2140 \pm 264$ [sec.] & $1971 \pm 701$ [sec.] \\ 
       \hline
       Ours w/o ES & $2609 \pm 747$ [sec.] & $2450 \pm 1042$ [sec.] \\ 
       \hline
    \end{tabular}
    }
\end{table}

\section{Related Works}
% We review related works on causality and explanations in decision making, MDP transformations for conditional probabilities, and learning-based verification. We use this discussion to motivate the specific gap addressed in this paper: identifying probability-raising causes for reachability properties in unknown MDPs from transition data with verifiable and probabilistic guarantees.
\subsection{Causality and Explanations for Decision Making}
\label{related_works:causal_explanation}
A line of work has studied actual causality, often defined over known system models \citep{coenen2022temporal, finkbeiner2024synthesis, dimitrova2020probabilistic, clarkson2014temporal, finkbeiner2023counterfactuals, finkbeiner2024counterfactual, carelli2025closure, oberst2019counterfactual, tsirtsis2021counterfactual}. These approaches aim to explain why a particular outcome occurred. 
\cite{rafieioskouei2025efficient} has considered causal states on tree-structured known Markov chains.
In contrast, we focus on estimating certified causal states for probabilistic behavior over unknown MDPs.
Closest to our setting are PR causes as sets of responsible states for Markovian models and computational procedures for them \citep{baier2021probabilistic, baier2022probability, baier2024foundations}. However, existing approaches assume a known MDP, and applying their methods to an unknown MDP leads to sample inefficiency.
As a context of explanations for verification, \cite{wimmer2014minimal, feng2018counterexamples, jantsch2020minimal, jansen2011hierarchical} have provided a subsystem and trajectories as evidence of a specification violation. However, they did not consider either unknown systems or causal principles.
Recent works on causal explanation for reinforcement learning are also related \citep{madumal2020explainable, madumal2020distal, yu2023explainable}, but they typically aim to explain agents' actions on fixed-length trajectories using learned or known SCM-like models. Moreover, they do not consider correctness and nondeterministic behavior.

\subsection{Transformations for Conditional Queries in Markovian Models}
\label{related_works:MDP_transformation}
Several works have studied the reduction of conditional reachability to unconditional reachability via model transformations \citep{baier2014computing, baier2017maximizing, marcker2017computing}. These methods and tools are primarily developed for known models; in the MDP setting, these pipelines require reachability probabilities in the original model during transformations, which are unavailable in unknown MDPs. Our proposed modification follows the same high-level idea but avoids querying reachability probabilities.

\subsection{Learning-based verification and PR causes under uncertainty}
\label{related_works:learning_verification}
% Learning verification, PAC model checking, PAC-MDP learning (Ufuc Topuc)
% All of the works mainly focused on verification and synthesis for the specifications, and do not consider causal explanation on why the specifications are fulfilled or violated.
PAC learning for unknown MDPs provides sampling-based verification and policy synthesis under specifications and stopping criteria \citep{ashok2019pac, agarwal2025pac,fu2014probably, wen2016probably}. While these methods underpin our bound computation, they have mainly aimed at property verification rather than causal identification.
% they do not classify states as causes/non-causes across cause-dependent variants. 
PR causes have also been studied under explicit uncertain MDPs, where causes are inferred by reasoning over a specified uncertain set for transition probabilities \citep{oura2025probability}. Our focus instead is a single unknown MDP, and we target (i) sample-efficient learning of PR causes and (ii) anytime certification. 

\section{Conclusion}
We studied PR-cause identification for reachability properties in unknown MDPs from transition samples. Our key ingredient is a restart-based MDP modification that addresses learning-specific challenges; we prove its correctness by reducing PR-cause checking to conditional reachability queries in the modified MDP, without requiring reachability values of the original MDP. We also derive and compare sample-complexity bounds for the existing and proposed MDP modifications. Building on two-sided value iteration, we develop an anytime learning-and-checking algorithm with probabilistic guarantees that progressively classifies candidate states as causal, non-causal, or undecided, and outputs a PR cause. Experiments on two benchmarks demonstrate that the approach reliably and quickly identifies the desired PR cause.
Future work includes extending the approach to temporal logic formulas as causes, investigating adaptive sampling strategies, partial observability, and multi-agent settings, and considering other policy classes.

\bibliography{reference}

@book{durrett2019probability,
  title={Probability: theory and examples},
  author={Durrett, Rick},
  volume={49},
  year={2019},
  publisher={Cambridge university press}
}

@book{baier2008principles,
  title={Principles of Model Checking},
  author={Baier, Christel and Katoen, Joost-Pieter},
  year={2008},
  publisher={MIT press}
}

@book{de1998formal,
  title={Formal verification of probabilistic systems},
  author={De Alfaro, Luca},
  year={1998},
  publisher={stanford university}
}

@inproceedings{chen2022multi,
  title={Multi-objective controller synthesis with uncertain human preferences},
  author={Chen, Shenghui and Boggess, Kayla and Parker, David and Feng, Lu},
  booktitle={Proceedings of the ACM/IEEE 13th International Conference on Cyber-Physical Systems},
  pages={170--180},
  year={2022}
}

@article{drager2015permissive,
  title={Permissive controller synthesis for probabilistic systems},
  author={Drager, Klaus and Forejt, Vojtech and Kwiatkowska, Marta and Parker, David and Ujma, Mateusz},
  journal={Logical Methods in Computer Science},
  volume={11},
  year={2015},
  publisher={Episciences. org}
}

@inproceedings{junges2018model,
  title={Model checking for safe navigation among humans},
  author={Junges, Sebastian and Jansen, Nils and Katoen, Joost-Pieter and Topcu, Ufuk and Zhang, Ruohan and Hayhoe, Mary},
  booktitle={Proceedings of the 15th International Conference on Quantitative Evaluation of Systems},
  pages={207--222},
  year={2018},
  organization={Springer}
}

@article{farrington2025going,
  title={Going faster to see further: graphics processing unit-accelerated value iteration and simulation for perishable inventory control using JAX: J. Farrington et al.},
  author={Farrington, Joseph and Wong, Wai Keong and Li, Kezhi and Utley, Martin},
  journal={Annals of Operations Research},
  volume={349},
  number={3},
  pages={1609--1638},
  year={2025},
  publisher={Springer}
}

@inproceedings{ashok2019pac,
  title={PAC statistical model checking for Markov decision processes and stochastic games},
  author={Ashok, Pranav and K{\v{r}}et{\'\i}nsk{\`y}, Jan and Weininger, Maximilian},
  booktitle={Proceedings of the 31st International Conference on Computer Aided Verification},
  pages={497--519},
  year={2019},
  organization={Springer}
}

@article{agarwal2025pac,
  title={PAC statistical model checking of mean payoff in discrete-and continuous-time MDP},
  author={Agarwal, Chaitanya and Guha, Shibashis and K{\v{r}}et{\'\i}nsk{\`y}, Jan and Pazhamalai, M},
  journal={Formal Methods in System Design},
  volume={66},
  number={2},
  pages={195--237},
  year={2025},
  publisher={Springer}
}

@inproceedings{fu2014probably,
  title={Probably approximately correct MDP learning and control with temporal logic constraints},
  author={Fu, Jie and Topcu, Ufuk},
  booktitle={Proceedings of Robotics: Science and Systems},
  year={2014}
}

@inproceedings{wen2016probably,
  title={Probably Approximately Correct Learning in Stochastic Games with Temporal Logic Specifications.},
  author={Wen, Min and Topcu, Ufuk},
  booktitle={Proceedings of the 25th International Joint Conference on Artificial Intelligence},
  pages={3630--3636},
  year={2016}
}

@inproceedings{perez2024pac,
  title={A PAC learning algorithm for LTL and omega-regular objectives in MDPs},
  author={Perez, Mateo and Somenzi, Fabio and Trivedi, Ashutosh},
  booktitle={Proceedings of the 38th Annual AAAI Conference on Artificial Intelligence},
  volume={38},
  number={19},
  pages={21510--21517},
  year={2024}
}

@article{suilen2022robust,
  title={Robust anytime learning of Markov decision processes},
  author={Suilen, Marnix and Sim{\~a}o, Thiago D and Parker, David and Jansen, Nils},
  journal={Advances in Neural Information Processing Systems},
  volume={35},
  pages={28790--28802},
  year={2022}
}

@article{oura2020reinforcement,
  title={Reinforcement learning of control policy for linear temporal logic specifications using limit-deterministic generalized B{\"u}chi automata},
  author={Oura, Ryohei and Sakakibara, Ami and Ushio, Toshimitsu},
  journal={Control Systems Letters},
  volume={4},
  number={3},
  pages={761--766},
  year={2020},
  publisher={IEEE}
}

@article{haddad2018interval,
  title={Interval iteration algorithm for MDPs and IMDPs},
  author={Haddad, Serge and Monmege, Benjamin},
  journal={Theoretical Computer Science},
  volume={735},
  pages={111--131},
  year={2018},
  publisher={Elsevier}
}

@inproceedings{baier2014computing,
  title={Computing conditional probabilities in Markovian models efficiently},
  author={Baier, Christel and Klein, Joachim and Kl{\"u}ppelholz, Sascha and M{\"a}rcker, Steffen},
  booktitle={Proceedings of the 20th International Conference on Tools and Algorithms for the Construction and Analysis of Systems},
  pages={515--530},
  year={2014},
  organization={Springer}
}

@inproceedings{baier2017maximizing,
  title={Maximizing the conditional expected reward for reaching the goal},
  author={Baier, Christel and Klein, Joachim and Kl{\"u}ppelholz, Sascha and Wunderlich, Sascha},
  booktitle={Proceedings of the 23rd International Conference on Tools and Algorithms for the Construction and Analysis of Systems},
  pages={269--285},
  year={2017},
  organization={Springer}
}

@inproceedings{marcker2017computing,
  title={Computing conditional probabilities: implementation and evaluation},
  author={M{\"a}rcker, Steffen and Baier, Christel and Klein, Joachim and Kl{\"u}ppelholz, Sascha},
  booktitle={Proceedings of the 15th International Conference on Software Engineering and Formal Methods},
  pages={349--366},
  year={2017},
  organization={Springer}
}

@article{hoeffding1963probability,
  title={Probability inequalities for sums of bounded random variables},
  author={Hoeffding, Wassily},
  journal={Journal of the American Statistical Association},
  volume={58},
  number={301},
  pages={13--30},
  year={1963},
  publisher={Taylor \& Francis}
}

@inproceedings{baier2021verification,
  title={From Verification to Causality-Based Explications},
  author={Baier, Christel and Dubslaff, Clemens and Funke, Florian and Jantsch, Simon and Majumdar, Rupak and Piribauer, Jakob and Ziemek, Robin},
  booktitle={Proceedings of the 48th International Colloquium on Automata, Languages, and Programming},
  pages={1--1},
  year={2021}
}

@inproceedings{oberst2019counterfactual,
  title={Counterfactual off-policy evaluation with Gumbel-max structural causal models},
  author={Oberst, Michael and Sontag, David},
  booktitle={Proceedings of the 15th International Conference on Machine Learning},
  pages={4881--4890},
  year={2019},
  organization={PMLR}
}

@article{tsirtsis2021counterfactual,
  title={Counterfactual Explanations in Sequential Decision Making under Uncertainty},
  author={Tsirtsis, Stratis and De, Abir and Rodriguez, Manuel},
  journal={Advances in Neural Information Processing Systems},
  volume={34},
  pages={30127--30139},
  year={2021}
}

@inproceedings{triantafyllou2022actual,
  title={Actual Causality and Responsibility Attribution in Decentralized Partially Observable {Markov} Decision Processes},
  author={Triantafyllou, Stelios and Singla, Adish and Radanovic, Goran},
  booktitle={Proceedings of the 5th AAAI/ACM Conference on AI, Ethics, and Society},
  pages={739--752},
  year={2022}
}

@article{rafieioskouei2025efficient,
  title={Efficient Discovery of Actual Causality with Uncertainty},
  author={Rafieioskouei, Arshia and Rogale, Kenneth and Bonakdarpour, Borzoo},
  journal={arXiv preprint arXiv:2507.09000},
  year={2025}
}

@inproceedings{baier2021probabilistic,
  title={Probabilistic Causes in {Markov} Chains},
  author={Baier, Christel and Funke, Florian and Jantsch, Simon and Piribauer, Jakob and Ziemek, Robin},
  booktitle={Proceedings of the 19th International Symposium on Automated Technology for Verification and Analysis},
  pages={205--221},
  year={2021}
}

@article{baier2022probability,
  title={On probability-raising causality in {Markov} decision processes},
  author={Baier, Christel and Funke, Florian and Piribauer, Jakob and Ziemek, Robin},
  journal={arXiv preprint arXiv:2201.08768},
  year={2022}
}

@article{baier2024foundations,
  title={Foundations of probability-raising causality in {Markov} decision processes},
  author={Baier, Christel and Piribauer, Jakob and Ziemek, Robin},
  journal={Logical Methods in Computer Science},
  volume={20},
  year={2024},
  publisher={Episciences. org}
}

@inproceedings{oura2025probability,
  title={Probability-raising causality for uncertain parametric Markov decision processes with PAC guarantees},
  author={Oura, Ryohei and Ito, Yuji},
  booktitle={Proceedings of the 41st Conference on Uncertainty in Artificial Intelligence},
  year={2025}
}

@inproceedings{coenen2022temporal,
  title={Temporal Causality in Reactive Systems},
  author={Coenen, Norine and Finkbeiner, Bernd and Frenkel, Hadar and Hahn, Christopher and Metzger, Niklas and Siber, Julian},
  booktitle={Proceedings of the 20th International Symposium on Automated Technology for Verification and Analysis},
  pages={208--224},
  year={2022},
  organization={Springer}
}

@inproceedings{finkbeiner2024synthesis,
  title={Synthesis of Temporal Causality},
  author={Finkbeiner, Bernd and Frenkel, Hadar and Metzger, Niklas and Siber, Julian},
  booktitle={Proceedings of the 36th International Conference on Computer Aided Verification},
  pages={87--111},
  year={2024},
  organization={Springer}
}

@inproceedings{finkbeiner2024counterfactual,
  title={Counterfactual Explanations for MITL Violations},
  author={Finkbeiner, Bernd and Jahn, Felix and Siber, Julian},
  booktitle={Proceedings of the 44th IARCS Annual Conference on Foundations of Software Technology and Theoretical Computer Science},
  pages={22--1},
  year={2024},
  organization={Schloss Dagstuhl--Leibniz-Zentrum f{\"u}r Informatik}
}

@article{carelli2025closure,
  title={Closure and Complexity of Temporal Causality},
  author={Carelli, Mishel and Finkbeiner, Bernd and Siber, Julian},
  journal={arXiv preprint arXiv:2505.10186},
  year={2025}
}

@article{kazemi2025causal,
  title={Causal temporal reasoning for Markov decision processes},
  author={Kazemi, Milad and Lally, Jessica and Paoletti, Nicola},
  journal={Research Directions: Cyber-Physical Systems},
  volume={3},
  pages={1--23},
  year={2025},
  publisher={Cambridge University Press}
}

@inproceedings{finkbeiner2023counterfactuals,
  title={Counterfactuals Modulo Temporal Logics},
  author={Finkbeiner, Bernd and Siber, Julian},
  booktitle={Proceedings of the 24th International Conference on Logic},
  volume={94},
  pages={181--204},
  year={2023}
}

@inproceedings{dimitrova2020probabilistic,
  title={Probabilistic Hyperproperties of {Markov} Decision Processes},
  author={Dimitrova, Rayna and Finkbeiner, Bernd and Torfah, Hazem},
  booktitle={Proceedings of the 18th International Symposium on Automated Technology for Verification and Analysis},
  pages={484--500},
  year={2020},
  organization={Springer}
}

@inproceedings{clarkson2014temporal,
  title={Temporal Logics for Hyperproperties},
  author={Clarkson, Michael R and Finkbeiner, Bernd and Koleini, Masoud and Micinski, Kristopher K and Rabe, Markus N and S{\'a}nchez, C{\'e}sar},
  booktitle={Proceedings of the 3rd International Conference on Principles of Security and Trust},
  pages={265--284},
  year={2014},
  organization={Springer Verlag}
}

@inproceedings{madumal2020explainable,
  title={Explainable reinforcement learning through a causal lens},
  author={Madumal, Prashan and Miller, Tim and Sonenberg, Liz and Vetere, Frank},
  booktitle={Proceedings of the 34th Annual AAAI conference on artificial intelligence},
  volume={34},
  number={03},
  pages={2493--2500},
  year={2020}
}

@article{madumal2020distal,
  title={Distal explanations for explainable reinforcement learning agents},
  author={Madumal, Prashan and Miller, Tim and Sonenberg, Liz and Vetere, Frank},
  journal={arXiv preprint arXiv:2001.10284},
  volume={2},
  year={2020}
}

@inproceedings{yu2023explainable,
  title={Explainable reinforcement learning via a causal world model},
  author={Yu, Zhongwei and Ruan, Jingqing and Xing, Dengpeng},
  booktitle={Proceedings of the 32nd International Joint Conference on Artificial Intelligence},
  pages={4540--4548},
  year={2023}
}

@article{wimmer2014minimal,
  title={Minimal counterexamples for linear-time probabilistic verification},
  author={Wimmer, Ralf and Jansen, Nils and {\'A}brah{\'a}m, Erika and Katoen, Joost-Pieter and Becker, Bernd},
  journal={Theoretical Computer Science},
  volume={549},
  pages={61--100},
  year={2014},
  publisher={Elsevier}
}

@inproceedings{feng2018counterexamples,
  title={Counterexamples for robotic planning explained in structured language},
  author={Feng, Lu and Ghasemi, Mahsa and Chang, Kai-Wei and Topcu, Ufuk},
  booktitle={Proceedings of the International Conference on Robotics and Automation},
  pages={7292--7297},
  year={2018},
  organization={IEEE}
}

@inproceedings{jantsch2020minimal,
  title={Minimal witnesses for probabilistic timed automata},
  author={Jantsch, Simon and Funke, Florian and Baier, Christel},
  booktitle={Proceedings of the 18th International Symposium on Automated Technology for Verification and Analysis},
  pages={501--517},
  year={2020},
  organization={Springer}
}

@inproceedings{jansen2011hierarchical,
  title={Hierarchical counterexamples for discrete-time Markov chains},
  author={Jansen, Nils and {\'A}brah{\'a}m, Erika and Katelaan, Jens and Wimmer, Ralf and Katoen, Joost-Pieter and Becker, Bernd},
  booktitle={Proceedings of the 9th International Symposium on Automated Technology for Verification and Analysis},
  pages={443--452},
  year={2011},
  organization={Springer}
}

\newpage

\onecolumn

\title{Supplementary Material}
\maketitle

\appendix

\section{Proofs of Theorems and Propositions}
\label{appendix:proofs}
We start with proving that, for any policy $\pi: \mathrm{FinPath} \times A \to [0,1]$, there exists a policy $\tilde{\pi}: \mathrm{FinPath} \times A \to [0,1]$ such that the reachability probability conditioned on bypass of a state $c$ on the original MDP under $\pi$ is equivalent to the reachability probability on our modified MDP under $\tilde{\pi}$, that is, $\mathrm{Pr}^\pi_{\mathcal{M}}(\eventually E \mid \neg \eventually c) = \mathrm{Pr}^{\tilde{\pi}}_{\mathcal{M}^{[c]}}(\eventually E)$. We then show that each state $c$ can be classified as causal or non-causal using the modified MDP $\mathcal{M}^{[c]}$. Let $\mathrm{InfPath}$ be the set of infinite paths on $\mathcal{M}$. Recall that we denote $\mathrm{\Pi}_{\mathcal{M},c} = \{ \pi : \mathrm{FinPath} \times A \to [0,1] \mid \mathrm{Pr}^\pi_{\mathcal{M}}(\eventually c) \in (0,1)  \}$ for any MDP $\mathcal{M}$ and any state $c \in S$.
For any $\rho \in \mathrm{FinPath}$, we denote $C_\rho = \{ \rho' \mid \rho' \in \mathrm{InfPath}, \rho \in \mathrm{pref}(\rho') \}$, where $\mathrm{pref}(\rho)$ denotes the set of prefixes of $\rho$. We define the concatenation $\rho \rho'$ of any $\rho = s_0 a_0 \ldots s_n$ and any $\rho' = s_0'a_0' \ldots s_m'$ such that $s_n = s_0'$ by $\rho \rho' = s_0 a_0 \ldots s_n a_0' \ldots s_m'$. Moreover, for any $\rho = s_0a_0\ldots s_n \in \mathrm{FinPath}$, we denote $\rho_{t} = s_t$, $\rho_{t:} = s_ta_t \ldots s_n$, $\rho_{:t} = s_0a_0 \ldots s_t$, and its length by $|\rho| = n+1$. In what follows, without loss of generality, each terminal state $s \in E$ is treated as absorbing by adding a dummy self-loop action.

\begin{proof}[Proof of Lemma \ref{lemma:conditioned_reachability_to_reachability_modified_MDP}]
Fix a state $c \in S \setminus E$ and a policy $\pi: \mathrm{FinPath} \times A \to [0,1]$. For any finite path $\rho$ on $\mathcal{M}^{[c]}$, write $\rho = s_0a_0s_1\ldots a_{n-1}s_n$. We define \textit{episode suffix} $\mathrm{suf}_c(\rho)$ as the suffix of $\rho$ that starts at the most recent occurrence of $s_I$ that is the initial or is jumped from $c$: let
    \begin{align}
        i^\star(\rho) = \max\{ i \mid s_i = s_I \mbox{ and } ( i=0 \mbox{ or } s_{i-1}=c) \},
    \end{align}
    and set
    \begin{align}
        \mathrm{suf}_c(\rho) = 
        \begin{cases}
            s_{i^\star(\rho)}a_{i^\star(\rho)}\ldots a_{n-1}s_n & \mbox{ if } i^\star(\rho) < n, \\
            s_n & \mbox{ otherwise}.
        \end{cases}
    \end{align}
    Now we define the policy $\tilde{\pi}: \mathrm{FinPath} \times A \to [0,1]$ as follows: for any $\rho \in \mathrm{FinPath}$ and any $a \in \mathcal{A}(\mathrm{last}(\rho))$,
    \begin{align}
    \label{tilde_pi}
        \tilde{\pi}(\rho, a) = \pi(\mathrm{suf}_c(\rho), a).
    \end{align}
    % Hence, the probability of reaching $E$ after visiting $c$ and going back to $s_I$ on $\mathcal{M}^{[c]}$ under $\tilde{\pi}$ equals the probability of reaching $E$ when starting from $s_I$ on $\mathcal{M}^{[c]}$ under $\tilde{\pi}$. 
    % Let $t_c$ be the first hitting time of $c$. 
    For any $t > 0$, we consider $A_t = \{ \rho = s_0a_0 \ldots s_t a_{t} s_{t+1} \in \mathrm{FinPath} \mid c \not\in \rho_{:t-1}, s_t = c, s_{t+1} = s_I \}$.
    Then, consider the event $H_t$ of the first hitting time of $c$ is $t < \infty$, that is, $H_t = \bigcup_{\rho \in A_t} C_\rho$. For any $A_t$, any $\rho = s_0a_0 \ldots s_{t+1} \in A_t$, and any $\rho' = s_{t+1}a_{t+1} \ldots s_n \in \mathrm{FinPath}$ such that $s_{t+1} = s_I$, by (\ref{tilde_pi}), we have
    \begin{align}
    \label{probability_decomposition}
        \mathrm{Pr}^{\tilde{\pi}}_{\mathcal{M}^{[c]}}(C_{\rho \rho'}) & = \Pi_{i=0}^{n-1} P^c(s_i, a_i, s_{i+1}) \tilde{\pi}(s_0a_0\ldots s_i, a_i) \nonumber \\
        & =\Pi_{i=0}^{t} P^c(s_i, a_i, s_{i+1}) \tilde{\pi}(s_0a_0\ldots s_i, a_i) \Pi_{j=t+1}^{n-1} P^c(s_j, a_j, s_{j+1}) \tilde{\pi}(s_{t+1}a_{t+1}\ldots s_j, a_j) \nonumber \\
        & = \mathrm{Pr}^{\tilde{\pi}}_{\mathcal{M}^{[c]}}(C_{\rho}) \mathrm{Pr}^{\tilde{\pi}}_{\mathcal{M}^{[c]}}(C_{\rho'}).
    \end{align}
    Let $D_t({\rho'}) = \bigcup_{\rho \in A_t} C_{\rho \rho'}$. Since $C_{\rho_1} \cap C_{\rho_2} = \emptyset$ for any ${\rho_1}, {\rho_2} \in A_t$ such that ${\rho_1} \neq {\rho_2}$, we have, when $\mathrm{Pr}^{\tilde{\pi}}_{\mathcal{M}^{[c]}}(H_t) > 0$,
    \begin{align}
    \label{cylinder_probability_after_restart}
        \mathrm{Pr}^{\tilde{\pi}}_{\mathcal{M}^{[c]}}(D_t({\rho'}) \mid H_t) & = \frac{\mathrm{Pr}^{\tilde{\pi}}_{\mathcal{M}^{[c]}}(D_t({\rho'}) \cap H_t)}{\mathrm{Pr}^{\tilde{\pi}}_{\mathcal{M}^{[c]}}(H_t)} \nonumber \\
        & = \frac{\sum_{\rho \in A_t} \mathrm{Pr}^{\tilde{\pi}}_{\mathcal{M}^{[c]}}(D_t({\rho'}) \cap C_\rho)}{\mathrm{Pr}^{\tilde{\pi}}_{\mathcal{M}^{[c]}}(H_t)} \nonumber \\
        & = \sum_{\rho \in A_t} \mathrm{Pr}^{\tilde{\pi}}_{\mathcal{M}^{[c]}}(D_t({\rho'}) \mid C_\rho) \mathrm{Pr}^{\tilde{\pi}}_{\mathcal{M}^{[c]}}(C_{\rho} \mid H_t). \nonumber \\
        \intertext{By (\ref{probability_decomposition}), we have $\mathrm{Pr}^{\tilde{\pi}}_{\mathcal{M}^{[c]}}(C_{\rho \rho'} | C_\rho) = \mathrm{Pr}^{\tilde{\pi}}_{\mathcal{M}^{[c]}}(C_{\rho'})$ for each $\rho \in A_t$. Hence,}
        & = \mathrm{Pr}^{\tilde{\pi}}_{\mathcal{M}^{[c]}}(C_{\rho'}) \sum_{\rho \in A_t} \mathrm{Pr}^{\tilde{\pi}}_{\mathcal{M}^{[c]}}(C_{\rho} \mid H_t) \nonumber \\
        & = \mathrm{Pr}^{\tilde{\pi}}_{\mathcal{M}^{[c]}}(C_{\rho'})
        % \intertext{since $C_{\rho \rho'} \subseteq D_{\rho'} := \bigcup_{\rho \in A_t} C_ $}
    \end{align}
    Consider a set of paths of length $n+1$ starting from $s_I$ that reach $E$: $\mathcal{P}_n = \{ \rho' \in \mathrm{FinPath} \mid \rho'_0 = s_I, |\rho'| = n+1, \exists i > 0 \mbox{ s.t. } \forall j \geq i, \rho_j \in E\}$, and let $B_{t, n} = \bigcup_{\rho' \in \mathcal{P}_n} D_t(\rho')$. By (\ref{cylinder_probability_after_restart}), we have 
    \begin{align}
        \mathrm{Pr}^{\tilde{\pi}}_{\mathcal{M}^{[c]}}(B_{t, n} \mid H_t)
        & = \sum_{\rho' \in \mathcal{P}_n} \mathrm{Pr}^{\tilde{\pi}}_{\mathcal{M}^{[c]}}(C_{\rho'}) \nonumber \\
        & = \mathrm{Pr}^{\tilde{\pi}}_{\mathcal{M}^{[c]}}(\bigcup_{\rho' \in \mathcal{P}_n} C_{\rho'}) \nonumber \\
        & = \mathrm{Pr}^{\tilde{\pi}}_{\mathcal{M}^{[c]}}(\eventually^{\leq n} E),
    \end{align}
    where $\eventually^{\leq T} E$ denotes the event of reaching $E$ within $T$ steps.
    Note that, since $E$ is the set of terminal states, if $m < n$, then, for any $\rho \in \mathcal{P}_m$, there exists $\rho' \in \mathcal{P}_n$ such that $\rho'_{:m} = \rho$.
    Thus, $(B_{t,n})_{n > 0}$ is increasing and $\bigcup_{n > 0} B_{t, n} = \{ \rho \in \mathrm{InfPath} \mid \rho \models \eventually E\} \cap H_t$. Hence,
    \begin{align}
        \mathrm{Pr}^{\tilde{\pi}}_{\mathcal{M}^{[c]}}(\eventually E \mid H_t) = \mathrm{Pr}^{\tilde{\pi}}_{\mathcal{M}^{[c]}}(\eventually E).
    \end{align}
    Thus, since $\eventually c$ represents $\bigcup_{t>0} H_t$ and $H_t \cap H_{t'} = \emptyset$ if $t \neq t'$, we have, when $\mathrm{Pr}^{\tilde{\pi}}_{\mathcal{M}^{[c]}}( \eventually c) > 0$,
    \begin{align}
    \label{restart_reachability}
        \mathrm{Pr}^{\tilde{\pi}}_{\mathcal{M}^{[c]}}(\eventually E \mid \eventually c) & = \sum_{t > 0} \mathrm{Pr}^{\tilde{\pi}}_{\mathcal{M}^{[c]}}(\eventually E \mid H_t) \mathrm{Pr}^{\tilde{\pi}}_{\mathcal{M}^{[c]}}(H_t \mid \eventually c) \nonumber \\
        & = \mathrm{Pr}^{\tilde{\pi}}_{\mathcal{M}^{[c]}}(\eventually E ) \sum_{t > 0} \mathrm{Pr}^{\tilde{\pi}}_{\mathcal{M}^{[c]}}(H_t \mid \eventually c) \nonumber \\
        & = \mathrm{Pr}^{\tilde{\pi}}_{\mathcal{M}^{[c]}}(\eventually E ).
    \end{align}
    % Consider $\mathcal{M}^{[c]}$ under $\tilde{\pi}$. 
    % By Def.\ \ref{def:restart_based_modification}, the only change of $\mathcal{M}^{[c]}$ from $\mathcal{M}$ is that the outgoing transitions from $c$ are replaced with the transition into $s_I$. By construction of $\tilde{\pi}$ in (\ref{tilde_pi}), the distribution of trajectories within any episode up to the first time hitting $c$ in $\mathcal{M}^{[c]}$ under $\tilde{\pi}$ is identical to the distribution in $\mathcal{M}$ under $\pi$ up to the first time hitting $c$. That is,
    By Def.\ \ref{def:restart_based_modification} and (\ref{tilde_pi}), for any $\rho = s_0a_0 \ldots s_{t+1} \in \mathrm{FinPath}$ such that $c \not\in \rho_{:t}$, we have
    \begin{align}
        \mathrm{Pr}^\pi_\mathcal{M}(C_\rho) & = \Pi_{i=0}^t P(s_i,a_i, s_{i+1})\pi(s_0a_0 \ldots s_i, a_{i}) \nonumber \\
        & = \Pi_{i=0}^t P^c(s_i,a_i, s_{i+1})\tilde{\pi}(s_0a_0 \ldots s_i, a_{i}) \nonumber \\
        & = \mathrm{Pr}^{\tilde{\pi}}_{\mathcal{M}^{[c]}}(C_\rho).
    \end{align}
    % Moreover, when we hit $c$ before $E$, i.e., the event $A_2$ happens, on $\mathcal{M}^{[c]}$, we deterministically jump back to $s_I$. 
    Hence, combined with (\ref{restart_reachability}), we have
    \begin{align}  
        \mathrm{Pr}^{\tilde{\pi}}_{\mathcal{M}^{[c]}}(\eventually E) & = \mathrm{Pr}^{\tilde{\pi}}_{\mathcal{M}^{[c]}}(\eventually E \land \neg \eventually c) + \mathrm{Pr}^{\tilde{\pi}}_{\mathcal{M}^{[c]}}(\eventually c) \mathrm{Pr}^{\tilde{\pi}}_{\mathcal{M}^{[c]}}(\eventually E ) \\
        & = \mathrm{Pr}^\pi_\mathcal{M}(\eventually E \land \neg \eventually c) + \mathrm{Pr}^\pi_\mathcal{M}(\eventually c)\mathrm{Pr}^{\tilde{\pi}}_{\mathcal{M}^{[c]}}(\eventually E) \\
        & = \mathrm{Pr}^\pi_\mathcal{M}(\eventually E \mid \neg \eventually c).
    \end{align}
    Moreover, if $\pi$ is memoryless, then $\pi(\rho, a) = \pi(\mathrm{last}(\rho), a)$, and hence $\pi(\mathrm{suf}_c(\rho), a) = \pi(\rho, a)$ for any $\rho \in \mathrm{FinPath}$ and any $a \in \mathcal{A}(\mathrm{last}(\rho))$. Thus, we have $\tilde{\pi} = \pi$.
    % we have
    % \begin{align}
        % \mathrm{Pr}^\pi_{\mathcal{M}^{[c]}}(\eventually E) & = \mathrm{Pr}^\pi_{\mathcal{M}^{[c]}}(\eventually E \;|\; \neg \eventually c) \mathrm{Pr}^\pi_{\mathcal{M}^{[c]}}(\neg \eventually c) + \mathrm{Pr}^\pi_{\mathcal{M}^{[c]}}(\eventually E \mid \eventually c) \mathrm{Pr}^\pi_{\mathcal{M}^{[c]}}(\eventually c) \nonumber \\
        % & = \mathrm{Pr}^\pi_{\mathcal{M}}(\eventually E \;|\; \neg \eventually c) \mathrm{Pr}^\pi_{\mathcal{M}}(\neg \eventually c) + \mathrm{Pr}^\pi_{\mathcal{M}^{[c]}}(\eventually E \mid \eventually c) \mathrm{Pr}^\pi_{\mathcal{M}}(\eventually c) \nonumber \\
         % & = \mathrm{Pr}^\pi_{\mathcal{M}}(\eventually E \;|\; \neg \eventually c) \mathrm{Pr}^\pi_{\mathcal{M}}(\neg \eventually c) + \mathrm{Pr}^\pi_{\mathcal{M}^{[c]}}(\eventually E) \mathrm{Pr}^\pi_{\mathcal{M}}(\eventually c) \nonumber \\
        % & =  \mathrm{Pr}^\pi_{\mathcal{M}}(\eventually E \;|\; \neg \eventually c).
    % \end{align}
\end{proof}
% Recall that the probability space for an MDP $\mathcal{M}$ under a policy $\pi$ is constructed as $(\mathrm{InfPath}, \mathcal{F}_\mathrm{InfPath}, \mathrm{Pr}^\pi_\mathcal{M})$ \citep{baier2008principles}.
\begin{lemma}
\label{lemma:policy_convex_combination}
    For any MDP $\mathcal{M}$, any policy $\pi_1, \pi_2$, and any $\lambda \in [0,1]$, there exists a policy $\pi_\lambda$ such that, for any $\rho \in \mathrm{FinPath}$,
    \begin{align}
    \label{pi_lambda_equation}
        \mathrm{Pr}^{\pi_\lambda}_{\mathcal{M}}(C_\rho) = \lambda \mathrm{Pr}^{\pi_1}_{\mathcal{M}}(C_\rho) + (1 - \lambda) \mathrm{Pr}^{\pi_2}_{\mathcal{M}}(C_\rho),
    \end{align}
    where $C_\rho = \{\rho \rho' \mid \rho' \in \mathrm{InfPath} \}$ for any $\rho \in \mathrm{FinPath}$.
\end{lemma}
\begin{proof}
    For any $\pi_1$ and $\pi_2$, and any $\lambda \in [0,1]$, we define $\pi_\lambda : \mathrm{FinPath} \times A \to [0,1]$ such that, for any $\rho \in \mathrm{FinPath}$,
    \begin{align}
        \pi_\lambda(\rho, a) = \alpha(\rho)\pi_1(\rho, a) + (1 - \alpha(\rho))\pi_2(\rho, a),
    \end{align}
    where
    \begin{align}
        \alpha(\rho) = 
        \begin{cases}
            \frac{\lambda \mathrm{Pr}^{\pi_1}_\mathcal{M}(C_\rho)}{\lambda \mathrm{Pr}^{\pi_1}_\mathcal{M}(C_\rho) + (1 - \lambda)\mathrm{Pr}^{\pi_2}_\mathcal{M}(C_\rho)} & \mbox{if } \mathrm{Pr}^{\pi_1}_\mathcal{M}(C_\rho) > 0 \lor \mathrm{Pr}^{\pi_2}_\mathcal{M}(C_\rho) > 0, \\
            1 & \mbox{otherwise}.
        \end{cases}
    \end{align}
    We show the claim by induction.
    \begin{enumerate}
        \item If $|\rho|=1$, that is the length of $\rho$ is $0$, then we have 
        \begin{align}
            \mathrm{Pr}^{\pi_\lambda}_\mathcal{M}(C_\rho) =
            \begin{cases}
                1 & \mbox{if } \rho = s_I, \\
                0 & \mbox{ otherwise}.
            \end{cases}
        \end{align}
        Likewise,
        \begin{align}
            \lambda \mathrm{Pr}^{\pi_1}_{\mathcal{M}}(C_\rho) + (1 - \lambda) \mathrm{Pr}^{\pi_2}_{\mathcal{M}}(C_\rho) =
            \begin{cases}
                1 & \mbox{if } \rho = s_I, \\
                0 & \mbox{ otherwise}.
            \end{cases}
        \end{align}
        \item Suppose that we have (\ref{pi_lambda_equation}) for any $\rho \in \mathrm{FinPath}$ with $|\rho| = n$, then, for any $\rho' = \rho as'$, we have
        \begin{align}
            \mathrm{Pr}^{\pi_\lambda}_\mathcal{M}(C_{\rho'}) & = \mathrm{Pr}^{\pi_\lambda}_\mathcal{M}(C_\rho)\pi_\lambda(\rho, a)P(s,a,s'), \nonumber \\
            \intertext{where $s$ is the last state of $\rho$, by the assumption and the definition of $\pi_\lambda$, we have}
            & = ( \lambda \mathrm{Pr}^{\pi_1}_\mathcal{M}(C_\rho) \pi_1(\rho, a) + (1 - \lambda) \mathrm{Pr}^{\pi_2}_\mathcal{M}(C_\rho) \pi_2(\rho, a)) P(s,a,s') \nonumber \\
            & = \lambda \mathrm{Pr}^{\pi_1}_{\mathcal{M}}(C_{\rho '}) + (1 - \lambda) \mathrm{Pr}^{\pi_2}_{\mathcal{M}}(C_{\rho '}).
        \end{align}
    \end{enumerate}
\end{proof}

\begin{lemma}
\label{lemma:policy_convex_combination_any_measurable}
    For any MDP $\mathcal{M}$, any policy $\pi_1, \pi_2$, and any $\lambda \in [0,1]$, there exists a policy $\pi_\lambda$ such that, for any measurable sets $\mathcal{T} \subseteq \mathrm{InfPath}$,
    \begin{align}
    \label{pi_lambda_equation_any_measurable}
        \mathrm{Pr}^{\pi_\lambda}_{\mathcal{M}}(\mathcal{T}) = \lambda \mathrm{Pr}^{\pi_1}_{\mathcal{M}}(\mathcal{T}) + (1 - \lambda) \mathrm{Pr}^{\pi_2}_{\mathcal{M}}(\mathcal{T}).
    \end{align}
\end{lemma}
\begin{proof}
    This immediately follows from Lemma \ref{lemma:policy_convex_combination} and Theorem 2.1.6 in \citep{durrett2019probability}.
\end{proof}

Based on Lemmas \ref{lemma:conditioned_reachability_to_reachability_modified_MDP} and \ref{lemma:policy_convex_combination_any_measurable}, we prove Proposition \ref{prop:soundness_completeness_restart_modification}.
\begin{proof}[Proof of Proposition \ref{prop:soundness_completeness_restart_modification}]
% Suppose $\Pi_{\mathcal{M}, c} = \emptyset$. Then, we have the two cases: (i) for any $\pi$, $\mathrm{Pr}^\pi_\mathcal{M}(\eventually c) = 0$, and (ii) for any $\pi$, $\mathrm{Pr}^\pi_\mathcal{M}(\eventually c) = 1$. In case (i), $\{c\}$ violates the condition \textbf{(C2)} in Def. \ref{def:SPR_cause}. In case (ii), $\{c\}$ violates the condition \textbf{(C1)} in Def. \ref{def:SPR_cause}. Next, we consider $\Pi_{\mathcal{M}, c} \neq \emptyset$.
Suppose $\Pi_{\mathcal{M}, c} = \emptyset$. Then, $\{c\}$ violates the condition \textbf{(C2)} in Def. \ref{def:SPR_cause}. Next, we consider $\Pi_{\mathcal{M}, c} \neq \emptyset$.
    \begin{enumerate}
        \item Suppose $\mathrm{Pr}^\mathrm{min}_{\mathcal{M}, c}(\eventually E) > \mathrm{Pr}^\mathrm{max}_{\mathcal{M}_{[c]}}(\eventually E)$. Then, by Lemma \ref{lemma:conditioned_reachability_to_reachability_modified_MDP}, for any $\pi \in \Pi_{\mathcal{M}, c}$, there exists a policy $\tilde{\pi}$ such that $\mathrm{Pr}^{\tilde{\pi}}_{\mathcal{M}^{[c]}}(\eventually E) = \mathrm{Pr}^\pi_{\mathcal{M}}(\eventually E \mid \neg \eventually c)$. Thus,
        \begin{align}
            \mathrm{Pr}^\pi_{\mathcal{M}}(\eventually E \mid \eventually c) - \mathrm{Pr}^\pi_\mathcal{M}(\eventually E)
            & = (1 - \mathrm{Pr}^\pi_\mathcal{M}(\eventually c)) (\mathrm{Pr}^\pi_{\mathcal{M}}(\eventually E \mid \eventually c) - \mathrm{Pr}^{\tilde{\pi}}_{\mathcal{M}^{[c]}}(\eventually E)), \nonumber \\
            \intertext{hence,}
            & \geq (1 - \mathrm{Pr}^\pi_\mathcal{M}(\eventually c)) (\mathrm{Pr}^\mathrm{min}_{\mathcal{M}, c}(\eventually E) - \mathrm{Pr}^\mathrm{max}_{\mathcal{M}^{[c]}}(\eventually E)) \nonumber \\
            & > 0.
        \end{align}
        \item Suppose $\mathrm{Pr}^\mathrm{min}_{\mathcal{M}, c}(\eventually E) < \mathrm{Pr}^\mathrm{max}_{\mathcal{M}^{[c]}}(\eventually E)$. Then, there exists $\pi_\mathrm{2mem} : \mathrm{FinPath} \to A$ such that $\pi_\mathrm{2mem}$ is identical to a memoryless policy $\pi_\mathrm{max}$ with  $\mathrm{Pr}^{\pi_\mathrm{max}}_{\mathcal{M}^{[c]}}(\eventually E) = \mathrm{Pr}^\mathrm{max}_{\mathcal{M}^{[c]}}(\eventually E)$ before hitting $c$, and identical to a memoryless policy $\pi_\mathrm{min}$ with $\mathrm{Pr}^{\pi_\mathrm{min}}_{\mathcal{M}, c}(\eventually E) = \mathrm{Pr}^\mathrm{min}_{\mathcal{M}, c}(\eventually E)$ after hitting $c$.
        % Hence, we have $\mathrm{Pr}^\mathrm{min}_{\mathcal{M}, c}(\eventually E) < \mathrm{Pr}^{\pi_\mathrm{2mem}}_{\mathcal{M}^{[c]}}(\eventually E) \leq \mathrm{Pr}^\mathrm{max}_{\mathcal{M}^{[c]}}(\eventually E)$. 
        Moreover, there exists a policy $\pi$ such that $\pi$ is identical to a memoryless policy $\pi'$ with $\mathrm{Pr}^{\pi'}_\mathcal{M}(\eventually c) \in (0, 1)$ before hitting $c$, and identical $\pi_\mathrm{min}$ after hitting $c$. By Lemma \ref{lemma:policy_convex_combination_any_measurable}, for any $\lambda \in (0,1)$, there exists a policy $\pi_\lambda$ such that, for any measurable set $\mathcal{T} \subseteq \mathrm{InfPath}$,
        \begin{align}
            \mathrm{Pr}^{\pi_\lambda}_{\mathcal{M}}(\mathcal{T}) = \lambda \mathrm{Pr}^{\pi_\mathrm{2mem}}_{\mathcal{M}}(\mathcal{T}) + (1 - \lambda) \mathrm{Pr}^\pi_{\mathcal{M}}(\mathcal{T}).
        \end{align}
        Hence,
        \begin{align}
            & \mathrm{Pr}^{\pi_\lambda}_{\mathcal{M}}(\eventually c) \geq (1 - \lambda) \mathrm{Pr}^{\pi}_{\mathcal{M}}(\eventually c) > 0, \\
            & \mathrm{Pr}^{\pi_\lambda}_{\mathcal{M}}(\neg \eventually c) \geq (1 - \lambda) \mathrm{Pr}^{\pi}_{\mathcal{M}}(\neg \eventually c) > 0.
        \end{align}
        % Moreover, we have
        % \begin{align}
            % \mathrm{Pr}^{\pi_\lambda}_{\mathcal{M}^{[c]}}(\eventually E) & \leq \lambda \mathrm{Pr}^{\pi_\mathrm{max}}_{\mathcal{M}^{[c]}}(\eventually E) + (1 - \lambda) \mathrm{Pr}^\pi_{\mathcal{M}^{[c]}}(\eventually E), \nonumber \\
            % \intertext{by Lemmas \ref{lemma:conditioned_reachability_to_reachability_modified_MDP} and \ref{lemma:policy_convex_combination_any_measurable},}
            % & = \lambda \mathrm{Pr}^{\pi_\mathrm{max}}_{\mathcal{M}}(\eventually E \mid \neg \eventually c) + (1 - \lambda) \mathrm{Pr}^\pi_{\mathcal{M}}(\eventually E \mid \neg \eventually c) \nonumber \\
            % & = \mathrm{Pr}^{\pi_\lambda}_{\mathcal{M}}(\eventually E \mid \neg \eventually c).
        % \end{align}
        Moreover,
        \begin{align}
            \mathrm{Pr}^{\pi_\lambda}_{\mathcal{M}}(\eventually E \mid \eventually c) & = \frac{\lambda \mathrm{Pr}^{\pi_\mathrm{2mem}}(\eventually E \land \eventually c) + (1 - \lambda)\mathrm{Pr}^{\pi}(\eventually E \land \eventually c)}{\lambda \mathrm{Pr}^{\pi_\mathrm{2mem}}(\eventually c) + (1 - \lambda)\mathrm{Pr}^{\pi}(\eventually c)} \nonumber \\
                & = \mathrm{Pr}^{\mathrm{min}}_{\mathcal{M}, c}(\eventually E).
        \end{align}
        So, we have
        \begin{align}
            \mathrm{Pr}^{\pi_\lambda}_{\mathcal{M}}(\eventually E \mid \eventually c) - \mathrm{Pr}^{\pi_\lambda}_\mathcal{M}(\eventually E) & = (1 - \mathrm{Pr}^{\pi_\lambda}_\mathcal{M}(\eventually c)) (\mathrm{Pr}^{\pi_\lambda}_{\mathcal{M}}(\eventually E \mid \eventually c) - \mathrm{Pr}^{\pi_\lambda}_{\mathcal{M}}(\eventually E \mid \neg \eventually c)) \nonumber \\
            & = (1 - \mathrm{Pr}^{\pi_\lambda}_\mathcal{M}(\eventually c)) (\mathrm{Pr}^{\mathrm{min}}_{\mathcal{M},c}(\eventually E) - \mathrm{Pr}^{\pi_\lambda}_{\mathcal{M}}(\eventually E \mid \neg \eventually c)).
            % &\hspace{5mm} \times \left\{(\mathrm{Pr}^{\mathrm{min}}_{\mathcal{M},c}(\eventually E) - \mathrm{Pr}^{\pi}_{\mathcal{M}}(\eventually E \mid \neg \eventually c)) - \lambda (\mathrm{Pr}^{\pi_\mathrm{max}}_{\mathcal{M}}(\eventually E \mid \neg \eventually c) - \mathrm{Pr}^{\pi}_{\mathcal{M}}(\eventually E \mid \neg \eventually c)) \right\} \nonumber
            % \intertext{by Lemma \ref{lemma:conditioned_reachability_to_reachability_modified_MDP},}
            % & = (1 - \mathrm{Pr}^{\pi_\lambda}_\mathcal{M}(\eventually c)) \nonumber \\
            % &\hspace{5mm} \times \left\{(\mathrm{Pr}^{\mathrm{min}}_{\mathcal{M},c}(\eventually E) - \mathrm{Pr}^{\pi}_{\mathcal{M}^{[c]}}(\eventually E)) - \lambda (\mathrm{Pr}^{\mathrm{max}}_{\mathcal{M}^{[c]}}(\eventually E) - \mathrm{Pr}^{\pi}_{\mathcal{M}^{[c]}}(\eventually E)) \right\}
        \end{align}
        We consider the two cases.
        \begin{enumerate}
            \item $\mathrm{Pr}^{\pi}_{\mathcal{M}}(\eventually E \mid \neg \eventually c) = \mathrm{Pr}^{\mathrm{max}}_{\mathcal{M}}(\eventually E \mid \neg \eventually c)$: then, we have
            \begin{align}
                \mathrm{Pr}^{\pi_\lambda}_{\mathcal{M}}(\eventually E \mid \neg \eventually c) & = \frac{\lambda \mathrm{Pr}^{\pi_\mathrm{max}}(\eventually E \land \neg \eventually c) + (1 - \lambda)\mathrm{Pr}^{\pi}(\eventually E \land \neg \eventually c)}{\lambda \mathrm{Pr}^{\pi_\mathrm{max}}(\neg \eventually c) + (1 - \lambda)\mathrm{Pr}^{\pi}(\neg \eventually c)} \nonumber \\
                & = \mathrm{Pr}^{\pi_\mathrm{max}}_{\mathcal{M}}(\eventually E \mid \neg \eventually c),
                \intertext{by Lemma \ref{lemma:conditioned_reachability_to_reachability_modified_MDP},}
                & = \mathrm{Pr}^{\mathrm{max}}_{\mathcal{M}^{[c]}}(\eventually E).
            \end{align}
            Thus, by the assumption, $\mathrm{Pr}^{\pi_\lambda}_{\mathcal{M}}(\eventually E \mid \eventually c) < \mathrm{Pr}^{\pi_\lambda}_\mathcal{M}(\eventually E)$.
            \item $\mathrm{Pr}^{\pi}_{\mathcal{M}}(\eventually E \mid \neg \eventually c) < \mathrm{Pr}^{\mathrm{max}}_{\mathcal{M}}(\eventually E \mid \neg \eventually c)$: then, 
            \begin{align}
                \mathrm{Pr}^{\pi_\lambda}_{\mathcal{M}}(\eventually E \mid \neg \eventually c) & = \frac{\lambda \mathrm{Pr}^{\pi_\mathrm{2mem}}_{\mathcal{M}}(\eventually E \land \neg \eventually c) + (1 - \lambda) \mathrm{Pr}^\pi_{\mathcal{M}}(\eventually E \land \neg \eventually c)}{\lambda \mathrm{Pr}^{\pi_\mathrm{2mem}}_{\mathcal{M}}(\neg \eventually c) + (1 - \lambda) \mathrm{Pr}^\pi_{\mathcal{M}}(\neg \eventually c)} \nonumber \\
                & = \frac{\lambda \mathrm{Pr}^{\pi_\mathrm{max}}_{\mathcal{M}}(\eventually E \land \neg \eventually c) + (1 - \lambda) \mathrm{Pr}^\pi_{\mathcal{M}}(\eventually E \land \neg \eventually c)}{\lambda \mathrm{Pr}^{\pi_\mathrm{max}}_{\mathcal{M}}(\neg \eventually c) + (1 - \lambda) \mathrm{Pr}^\pi_{\mathcal{M}}(\neg \eventually c)}
            \end{align}
            is continuous with respect to $\lambda$. Moreover, as $\lambda \to 1$, we have $\mathrm{Pr}^{\pi_\lambda}_{\mathcal{M}}(\eventually E \mid \neg \eventually c) = \mathrm{Pr}^{\pi_\mathrm{max}}_{\mathcal{M}}(\eventually E \mid \neg \eventually c)$. So, by Lemma \ref{lemma:conditioned_reachability_to_reachability_modified_MDP}, $\mathrm{Pr}^{\pi_\lambda}_{\mathcal{M}}(\eventually E \mid \neg \eventually c) = \mathrm{Pr}^{\mathrm{max}}_{\mathcal{M}^{[c]}}(\eventually E)$ as $\lambda \to 1$. Hence, combined with the assumption, there exists $\lambda < 1$ such that $\mathrm{Pr}^{\mathrm{min}}_{\mathcal{M},c}(\eventually E) < \mathrm{Pr}^{\pi_\lambda}_{\mathcal{M}}(\eventually E \mid \neg \eventually c)$. Thus, $\mathrm{Pr}^{\pi_\lambda}_{\mathcal{M}}(\eventually E \mid \eventually c) < \mathrm{Pr}^{\pi_\lambda}_\mathcal{M}(\eventually E)$.
        \end{enumerate}
        % By the construction of $\pi_\lambda$, we have $\mathrm{Pr}^{\pi_\lambda}_{\mathcal{M}^{[c]}}(\eventually E) \geq \lambda \mathrm{Pr}^{\pi_\mathrm{2mem}}_{\mathcal{M}^{[c]}}(\eventually E) > \mathrm{Pr}^\mathrm{min}_{\mathcal{M}, c}(\eventually E)$, hence
        % We consider the two cases.
        % \begin{enumerate}
            % \item $\mathrm{Pr}^{\pi}_{\mathcal{M}^{[c]}}(\eventually E) = \mathrm{Pr}^{\mathrm{max}}_{\mathcal{M}^{[c]}}(\eventually E)$: we have immediately $\mathrm{Pr}^{\pi_\lambda}_{\mathcal{M}}(\eventually E \mid \eventually c) < \mathrm{Pr}^{\pi_\lambda}_\mathcal{M}(\eventually E)$ by the assumption.
            % \item $\mathrm{Pr}^{\pi}_{\mathcal{M}^{[c]}}(\eventually E) < \mathrm{Pr}^{\mathrm{max}}_{\mathcal{M}^{[c]}}(\eventually E)$: by the assumption, we can choose $\lambda$ such that $\lambda > \frac{\mathrm{Pr}^\mathrm{min}_{\mathcal{M}, c}(\eventually E) - \mathrm{Pr}^{\pi}_{\mathcal{M}^{[c]}} (\eventually E)}{ \mathrm{Pr}^{\mathrm{max}}_{\mathcal{M}^{[c]}} (\eventually E) - \mathrm{Pr}^{\pi}_{\mathcal{M}^{[c]}} (\eventually E)}$ and $\lambda \in (0, 1)$. Thus, $\mathrm{Pr}^{\pi_\lambda}_{\mathcal{M}}(\eventually E \mid \eventually c) < \mathrm{Pr}^{\pi_\lambda}_\mathcal{M}(\eventually E)$.
        % \end{enumerate}
        \item Suppose $\mathrm{Pr}^\mathrm{min}_{\mathcal{M}, c}(\eventually E) = \mathrm{Pr}^\mathrm{max}_{\mathcal{M}_{[c]}}(\eventually E)$. Then, we consider the following two cases.
        \begin{enumerate}
            \item We suppose $c$ is not reachable from $s_I$ in $\mathcal{M}^{[c]}_\mathrm{max}$. For any $\pi$ such that $\mathrm{Pr}^\pi_{\mathcal{M}^{[c]}}(\eventually E) = \mathrm{Pr}^\mathrm{max}_{\mathcal{M}^{[c]}}(\eventually E)$, we have $\mathrm{Pr}^\pi_{\mathcal{M}^{[c]}}(\eventually c) = 0$. Thus, for any $\tilde{\pi}$ such that $\mathrm{Pr}^{\tilde{\pi}}_{\mathcal{M}^{[c]}}(\eventually c) \in (0, 1)$, we have $\mathrm{Pr}^{\tilde{\pi}}_{\mathcal{M}^{[c]}}(\eventually E) < \mathrm{Pr}^\mathrm{max}_{\mathcal{M}^{[c]}}(\eventually E)$. Thus, by Lemma \ref{lemma:conditioned_reachability_to_reachability_modified_MDP}, for any $\pi \in \Pi_{\mathcal{M}, c}$, there exists $\tilde{\pi}$ such that
            \begin{align}
                \mathrm{Pr}^\pi_\mathcal{M}(\eventually E \mid \eventually c) - \mathrm{Pr}^\pi_\mathcal{M}(\eventually E) & = (1 - \mathrm{Pr}^\pi_\mathcal{M}(\eventually c) ) (\mathrm{Pr}^\pi_{\mathcal{M}}(\eventually E \mid \eventually c) - \mathrm{Pr}^{\tilde{\pi}}_{\mathcal{M}^{[c]}}(\eventually E)) \\
                & \geq (1 - \mathrm{Pr}^\pi_\mathcal{M}(\eventually c) ) (\mathrm{Pr}^\mathrm{min}_{\mathcal{M}, c}(\eventually E) - \mathrm{Pr}^{\tilde{\pi}}_{\mathcal{M}^{[c]}}(\eventually E)) \\
                \intertext{by the assumption}
                & = (1 - \mathrm{Pr}^\pi_\mathcal{M}(\eventually c) ) (\mathrm{Pr}^\mathrm{max}_{\mathcal{M}^{{[c]}}}(\eventually E) - \mathrm{Pr}^{\tilde{\pi}}_{\mathcal{M}^{[c]}}(\eventually E)) \\
                & > 0.
            \end{align}
            Hence, the conditions \textbf{(C1)} and \textbf{(C2)} hold for $c$.
            \item We suppose $c$ is reachable in $\mathcal{M}^{[c]}_\mathrm{max}$. Then, there exists a memoryless policy $\pi$ such that $\mathrm{Pr}^\pi_{\mathcal{M}^{[c]}}(\eventually E) = \mathrm{Pr}^\mathrm{max}_{\mathcal{M}^{[c]}}(\eventually E)$ and $\mathrm{Pr}^\pi_{\mathcal{M}}(\eventually c) > 0$. We consider the two cases:
            \begin{enumerate}
                \item $\mathrm{Pr}^\pi_{\mathcal{M}}(\eventually c) = 1$: by Def. \ref{def:restart_based_modification}, we have $\mathrm{Pr}^{\mathrm{max}}_{\mathcal{M}^{[c]}}(\eventually E) = 0$. Hence, $\mathrm{Pr}^{\mathrm{min}}_{\mathcal{M}, c}(\eventually E) = 0$. Thus, by Lemma \ref{lemma:conditioned_reachability_to_reachability_modified_MDP}, there exists a memoryless policy $\pi'$ such that $\mathrm{Pr}^{\pi'}_{\mathcal{M}, c}(\eventually E) = \mathrm{Pr}^{\mathrm{min}}_{\mathcal{M}, c}(\eventually E)$,
                \begin{align}
                    \mathrm{Pr}^{\pi'}_\mathcal{M}(\eventually E \mid \eventually c) - \mathrm{Pr}^{\pi'}_\mathcal{M}(\eventually E) & = (1 - \mathrm{Pr}^{\pi'}_\mathcal{M}(\eventually c) ) (\mathrm{Pr}^{\pi'}_{\mathcal{M}, c}(\eventually E) - \mathrm{Pr}^{\pi'}_{\mathcal{M}^{[c]}}(\eventually E)) \\
                    & = 0.
                \end{align}
                \item $\mathrm{Pr}^\pi_{\mathcal{M}}(\eventually c) \in (0,1)$: by Lemma \ref{lemma:conditioned_reachability_to_reachability_modified_MDP}, we have $\mathrm{Pr}^\pi_{\mathcal{M}}(\eventually E \mid \neg \eventually c) = \mathrm{Pr}^\mathrm{max}_{\mathcal{M}^{[c]}}(\eventually E)$. Thus, there exists a policy $\pi' \in \Pi_{\mathcal{M}, c}$ such that $\mathrm{Pr}^{\pi'}_{\mathcal{M}}(\eventually E \mid \neg \eventually c) = \mathrm{Pr}^\mathrm{max}_{\mathcal{M}^{[c]}}(\eventually E)$ and $\mathrm{Pr}^{\pi'}_{\mathcal{M}, c}(\eventually E) = \mathrm{Pr}^\mathrm{min}_{\mathcal{M}, c}(\eventually E)$, hence we have
                \begin{align}
                    \mathrm{Pr}^{\pi'}_\mathcal{M}(\eventually E \mid \eventually c) - \mathrm{Pr}^{\pi'}_\mathcal{M}(\eventually E) & = (1 - \mathrm{Pr}^{\pi'}_\mathcal{M}(\eventually c) ) (\mathrm{Pr}^{\pi'}_{\mathcal{M}, c}(\eventually E) - \mathrm{Pr}^{\pi'}_{\mathcal{M}}(\eventually E \mid \neg \eventually c)) \\
                    & = (1 - \mathrm{Pr}^{\pi'}_\mathcal{M}(\eventually c) ) (\mathrm{Pr}^\mathrm{min}_{\mathcal{M}, c}(\eventually E) - \mathrm{Pr}^{\mathrm{max}}_{\mathcal{M}^{[c]}}(\eventually E)) \\
                    & = 0.
                \end{align}
                % $\{c\}$ violates \textbf{(C1)} in Def. \ref{def:SPR_cause} under a policy $\pi{''}$ such that .
            % \end{enumerate}Hence, the condition \textbf{(C1)} in Def. \ref{def:SPR_cause} does not hold for $c$ under $\pi$.
            \end{enumerate}
        \end{enumerate}
    \end{enumerate}
\end{proof}

% Moreover, $\mathcal{M}^{[c]}$ contains at most one non-trivial and non-bottom MEC, and the state $c$ is always included in the MEC.

To prove Theorems \ref{thm:sample_efficiency}, \ref{thm:Algorithm1_is_solution}, and \ref{thm:Algorithm1_is_anytime}, we additionally show the following Lemmas. Let $\mathrm{Tr}(\mathcal{M})$ be the number of transitions of an MDP $\mathcal{M}$ whose probabilities are less than $1$ and more than $0$, that is, $\mathrm{Tr}(\mathcal{M}) = |\{ (s,a,s') \in S \times A \times S \mid P(s,a,s') \in (0,1) \}|$.

\begin{lemma}
\label{lemma:transition_lower_bound}
    For any MDP, any random samples $\mathcal{D} = (s_i, a_i, s_{i+1})_{i}$, and any $\delta \in (0,1)$, we have, for any $(s,a,s') \in S \times A \times S$,
    \begin{align}
        & \mathbb{P}_\mathcal{D}\left( \hat{P}_\delta(s,a,s') \leq P(s,a,s') \right) \geq 1 - \delta,
        \intertext{and}
        & \mathbb{P}_\mathcal{D}\left( \hat{P}_\delta \leq P \right) \geq 1 - \delta \mathrm{Tr}(\mathcal{M}),
    \end{align}
    where $\mathbb{P}_\mathcal{D}$ is the probability measure for $\mathcal{D}$.
\end{lemma}
\begin{proof}
    For any $\mathcal{D}$, any $\delta \in (0,1)$, and any transition $(s,a,s')$, we have
    \begin{align}
        \mathbb{P}_\mathcal{D}\left( \hat{P}_\delta(s,a,s') \geq P(s,a,s') \right) & = \mathbb{P}_\mathcal{D}\left( \frac{N(s,a,s')}{N(s,a)} - P(s,a,s') \geq \sqrt{\frac{- \ln{\delta}}{2 N(s,a)}} \right) \nonumber \\
        \intertext{by Hoeffding's inequality \citep{hoeffding1963probability}, }
        & \leq \delta.
    \end{align}
    Moreover, by Bool's inequality, we have
    \begin{align}   & \mathbb{P}_\mathcal{D}\left (\exists(s,a,s')\in S \times A \times S \mbox{ s.t. } \hat{P}_\delta(s,a,s') \geq P(s,a,s') \right) \\
    & = \mathbb{P}_\mathcal{D}\left(\exists(s,a,s')\in \mathrm{Tr}(\mathcal{M}) \mbox{ s.t. } \hat{P}_\delta(s,a,s') \geq P(s,a,s') \right) \\
    & \leq \sum_{(s,a,s')\in \mathrm{Tr}(\mathcal{M})}\mathbb{P}_\mathcal{D}\left(\hat{P}_\delta(s,a,s') \geq P(s,a,s') \right) \nonumber \\
    &\leq \delta \mathrm{Tr}(\mathcal{M}).
    \end{align}
\end{proof}

For any function $x: S \to \mathbb{R}$, we consider the infinite norm $|| \cdot ||_\infty$ such that $||x||_\infty = \max_{s \in S} |x(s)|$. The distance of functions $x$ and $x'$ under $|| \cdot ||_\infty$ is hence $||x - x'||_\infty = \max_{s \in S} |x(s) - x'(s)|$.
\begin{lemma}
\label{lemma:true_and_approximate_operator_bound}
    For any MDP $\mathcal{M}$ and any $\delta, \varepsilon \in (0,1)$, if $||P - \hat{P}_\delta||_\infty \leq \varepsilon$, then, for any $x:S \to [0,1]$ and any $\mathrm{opt}^{(1)}, \mathrm{opt}^{(2)} \in \{ \mathrm{min}, \mathrm{max}\}$, we have $|| f^{\mathrm{opt}^{(1)}, \mathrm{opt}^{(2)}}_{\mathcal{M}, \delta}[x] - f^{\mathrm{opt}^{(1)}}_{\mathcal{M}}[x]||_\infty \leq d_\mathrm{max} \varepsilon$.
\end{lemma}
\begin{proof}
    Let $\Delta P(s,a,s') = P(s,a,s') - \hat{P}_\delta(s,a,s')$ for any $(s,a,s') \in S \times A \times S$. For any $x: S \to [0,1]$ and any $s \in S$, we have
    \begin{align}
        & f^{\mathrm{opt}^{(1)}}_{\mathcal{M}}[x](s) = \mathrm{opt}^{(1)}_{a \in \mathcal{A}(s)} \{ \sum_{s' \in \mathrm{Post}(s,a)} \hat{P}_\delta(s,a,s') x(s') + \sum_{s' \in \mathrm{Post}(s,a)}\Delta P(s,a,s')x(s')\}, \\
        & f^{\mathrm{opt}^{(1)}, \mathrm{opt}^{(2)}}_{\mathcal{M}, \delta}[x](s) = \mathrm{opt}^{(1)}_{a \in \mathcal{A}(s)} \{ \sum_{s' \in \mathrm{Post}(s,a)} \hat{P}_\delta(s,a,s') x(s') \nonumber \\
        & \hspace{55mm} + \mathrm{opt}^{(2)}_{s' \in \mathrm{Post}(s,a)}x(s') \Delta P(s,a,s') \}.
    \end{align}
    Hence, noting that $\mathrm{opt}^{(1)}_i x_i - \mathrm{opt}^{(1)}_i y_i \leq \max_i |x_i - y_i|$ for any sequences of real numbers $(x_i)_i$ and $(y_i)_i$, we have
    \begin{align}
        & | f^{\mathrm{opt}^{(1)}}_{\mathcal{M}}[x](s) - f^{\mathrm{opt}^{(1)}, \mathrm{opt}^{(2)}}_{\mathcal{M}, \delta}[x](s)| \nonumber \\
        & \leq \max_{a} \{ \sum_{s'} \Delta P(s,a,s') x(s') - \mathrm{opt}^{(2)}_{s'} x(s') \sum\Delta P(s,a,s')\} \nonumber \\
        & \leq \max_{a} \sum_{s'} \Delta P(s,a,s') \nonumber \\
        & \leq d_\mathrm{max} \varepsilon.
    \end{align}
\end{proof}

\begin{lemma}
\label{lemma:T_step_operators_bound}
    For any MDP $\mathcal{M}$ and any $\delta, \varepsilon \in (0,1)$, if $||P - \hat{P}_\delta||_\infty \leq \varepsilon$, then, for any $x: S \to [0,1]$, any $T > 0$, and any $\mathrm{opt}^{(1)}, \mathrm{opt}^{(2)} \in \{ \mathrm{min}, \mathrm{max}\}$, we have $|| (f^{\mathrm{opt}^{(1)}, \mathrm{opt}^{(2)}}_{\mathcal{M}, \delta})^T[x] - (f^{\mathrm{opt}^{(1)}}_{\mathcal{M}})^T[x] ||_\infty \leq T d_\mathrm{max} \varepsilon$.
\end{lemma}
\begin{proof}
    For any $x: S \to [0,1]$, we define $x_t: S \to [0,1]$ and $z_t: S \to [0,1]$ as follows:
    \begin{align}
        & x_0 = z_0 = x, \\
        & x_{t+1} = f^{\mathrm{opt}^{(1)}, \mathrm{opt}^{(2)}}_{\mathcal{M}, \delta}[x_t], \\
        & z_{t+1} = f^{\mathrm{opt}^{(1)}}_{\mathcal{M}}[z_t].
    \end{align}
    Then, we have
    \begin{align}
        x_{t+1} - z_{t+1} & = f^{\mathrm{opt}^{(1)}, \mathrm{opt}^{(2)}}_{\mathcal{M}, \delta}[x_t] - f^{\mathrm{opt}^{(1)}}_{\mathcal{M}}[z_t], \nonumber \\
        & = f^{\mathrm{opt}^{(1)}, \mathrm{opt}^{(2)}}_{\mathcal{M}, \delta}[x_t] - f^{\mathrm{opt}^{(1)}}_{\mathcal{M}}[x_t] + f^{\mathrm{opt}^{(1)}}_{\mathcal{M}}[x_t] - f^{\mathrm{opt}^{(1)}}_{\mathcal{M}}[z_t].
    \end{align}
    Note that $f^{\mathrm{opt}^{(1)}}_{\mathcal{M}}[x_t] - f^{\mathrm{opt}^{(1)}}_{\mathcal{M}}[z_t] \leq ||x_t - z_t||_\infty$. Hence, by Lemma \ref{lemma:true_and_approximate_operator_bound},
    \begin{align}
        ||x_{T} - z_{T}||_\infty &\leq d_\mathrm{max} \varepsilon + ||x_{T-1} - z_{T-1}||_\infty \nonumber \\
        & \leq T d_\mathrm{max} \varepsilon.
    \end{align}
\end{proof}

\begin{lemma}
\label{lemma:monotone_continuos}
    For any MDP $\mathcal{M}$, any multi-sample $\mathcal{D}$ of transitions, and any $\delta \in (0,1)$, if $\hat{P}_\delta \leq P$ holds, the operators $f^{\mathrm{opt}^{(1)}, \mathrm{opt}^{(2)}}_{\mathcal{M}, \delta}$ defined by (\ref{opt12_update}) with $\mathrm{opt}^{(1)}, \mathrm{opt}^{(2)} \in \{ \mathrm{min}, \mathrm{max}\}$ are monotone.
\end{lemma}
\begin{proof}
    For any $x, x': S \to [0,1]$ and any $s \in S$, we have
    \begin{align}
        &f^{\mathrm{opt}^{(1)}, \mathrm{opt}^{(2)}}_{\mathcal{M}, \delta}[x'](s) - f^{\mathrm{opt}^{(1)}, \mathrm{opt}^{(2)}}_{\mathcal{M}, \delta}[x](s)\\
        & = \mathrm{opt}^{(1)}_{a \in \mathcal{A}(s)} \sum_{s' \in \mathrm{Post}(s,a)} \hat{P}_\delta(s,a,s') ( x'(s') - x(s') ) \nonumber \\
        & \hspace{10mm} + ( \mathrm{opt}^{(2)}_{s'' \in \mathrm{Post}(s,a)} x'(s'') - \mathrm{opt}^{(2)}_{s'' \in \mathrm{Post}(s,a)} x(s'') )\left(1 - \sum_{s' \in \mathrm{Post}(s,a)}\hat{P}_\delta(s,a,s') \right).
    \end{align}
    Clearly $1 \geq \sum_{s' \in \mathrm{Post}(s,a)}\hat{P}_\delta(s,a,s')$. Thus, if $x \leq x'$, then we have $f^{\mathrm{opt}^{(1)}, \mathrm{opt}^{(2)}}_{\mathcal{M}, \delta}[x'](s) - f^{\mathrm{opt}^{(1)}, \mathrm{opt}^{(2)}}_{\mathcal{M}, \delta}[x](s) \geq 0$, otherwise we have $f^{\mathrm{opt}^{(1)}, \mathrm{opt}^{(2)}}_{\mathcal{M}, \delta}[x'](s) - f^{\mathrm{opt}^{(1)}, \mathrm{opt}^{(2)}}_{\mathcal{M}, \delta}[x](s) \leq 0$.
\end{proof}

\begin{lemma}
\label{lemma:operators_monotonically_error_decreasing}
    For any MDP $\mathcal{M}$ and any $\delta, \varepsilon \in (0,1)$, if $||P - \hat{P}_\delta||_\infty \leq \varepsilon$, then, for any $x: S \to [0,1]$, any $T > 0$, and any $\mathrm{opt} \in \{ \mathrm{min}, \mathrm{max}\}$, we have $|| \underline{x}^{\mathrm{opt}, *}_{\mathcal{M}, \delta} - \mathrm{Pr}^\mathrm{opt}_{\mathcal{M}, \cdot}(\eventually E) ||_\infty \leq || (f^{\mathrm{opt}, \mathrm{min}}_{\mathcal{M}, \delta})^T[x] - \mathrm{Pr}^\mathrm{opt}_{\mathcal{M}, \cdot}(\eventually E) ||_\infty$ if $x \leq f^{\mathrm{opt}, \mathrm{min}}_{\mathcal{M}, \delta}[x]$, and $|| \overline{x}^{\mathrm{opt}, *}_{\mathcal{M}, \delta} - \mathrm{Pr}^\mathrm{opt}_{\mathcal{M}, \cdot}(\eventually E) ||_\infty \leq || (f^{\mathrm{opt}, \mathrm{max}}_{\mathcal{M}, \delta})^T[x] - \mathrm{Pr}^\mathrm{opt}_{\mathcal{M}, \cdot}(\eventually E) ||_\infty$ if $x \geq f^{\mathrm{opt}, \mathrm{max}}_{\mathcal{M}, \delta}[x]$, where $\underline{x}^{\mathrm{opt}, *}_{\mathcal{M}, \delta}$ and $\overline{x}^{\mathrm{opt}, *}_{\mathcal{M}, \delta}$ are fixed points of $f^{\mathrm{opt}, \mathrm{min}}_{\mathcal{M}, \delta}$ and $f^{\mathrm{opt}, \mathrm{max}}_{\mathcal{M}, \delta}$ defined in (\ref{opt12_update}), respectively.
\end{lemma}
\begin{proof}
    Both inequalities immediately follow from Lemma \ref{lemma:monotone_continuos}.
\end{proof}

\begin{lemma}
\label{lemma:transition_probability_bound_modified_MDPs}
    For MDP $\mathcal{M}$, any $c \in S$, and any $\delta, \varepsilon, \epsilon > 0$, if $||P - \hat{P}_\delta||_\infty \leq \varepsilon$, then we have
    \item \begin{align}
        \label{Pc_error_bound}
        & || P^c - \hat{P}^c_\delta ||_\infty \leq \varepsilon, \\
        \label{Pc_rm_error_bound}
        & || P^c_\mathrm{mr} - \hat{P}^c_{\mathrm{mr}, \delta} ||_\infty \leq T^\mathrm{min}_{\mathcal{M}, \epsilon, c} d_\mathrm{max} \varepsilon + \epsilon,
    \end{align}
    where $T^\mathrm{min}_{\mathcal{M}, \epsilon, c}$ is the minimal mixing time for $\epsilon$ defined by (\ref{def:mixing_time}).
\end{lemma}
\begin{proof}
    For any $\mathcal{M}$ and any $c \in S$, we have $P^c(s,a,s') = P(s,a,s')$ and $ \hat{P}^c_\delta(s,a,s') = \hat{P}_\delta(s,a,s')$ for any $(s,a,s') \in S \setminus \{c\} \times A \times S$ and $P^c(c,a,s_I) = \hat{P}^c_\delta(c,a,s_I) = 1$ for any $a \in \mathcal{A}(c)$, thus
    \begin{align}
        || P^c - \hat{P}^c_\delta ||_\infty = || P - \hat{P}_\delta ||_\infty \leq \varepsilon.
    \end{align}
    For any $(s,a,s') \in S \setminus \{c\} \times A \times S$, we have $P^c_\mathrm{mr}(s,a,s') = P(s,a,s')$ and $ \hat{P}^c_\mathrm{mr}(s,a,s') = \hat{P}(s,a,s')$, and, for any $a \in \mathcal{A}(c)$, we have $P^c_\mathrm{mr}(c, a, s_\mathrm{eff}) = \mathrm{Pr}^\mathrm{min}_{\mathcal{M}, c}(\eventually E)$, $P^c_\mathrm{mr}(c, a, s_\mathrm{noeff}) = 1 - \mathrm{Pr}^\mathrm{min}_{\mathcal{M}, c}(\eventually E)$, $\hat{P}^c_\mathrm{mr}(c, a, s_\mathrm{eff}) = \underline{x}^{\mathrm{min}, *}_{\mathcal{M}, \delta}(c)$, and $\hat{P}^c_\mathrm{mr}(c, a, s_\mathrm{noeff}) = 1 - \overline{x}^{\mathrm{min}, *}_{\mathcal{M}, \delta}(c)$. Thus, for any $\epsilon > 0$, there is a mixing time $T = T^\mathrm{min}_{\mathcal{M}, \epsilon, c}$ such that, by Lemma \ref{lemma:T_step_operators_bound},
    \begin{align}
        || P^c_\mathrm{mr} - \hat{P}^c_\mathrm{mr} ||_\infty & \leq \max \{ |\mathrm{Pr}^\mathrm{min}_{\mathcal{M}, c}(\eventually E) - \underline{x}^{\mathrm{min}, *}_{\mathcal{M}, \delta}(c)|, |\mathrm{Pr}^\mathrm{min}_{\mathcal{M}, c}(\eventually E) - (1 - \overline{x}^{\mathrm{min}, *}_{\mathcal{M}, \delta}(c))| \} \nonumber \\
        & \leq T d_\mathrm{max} \varepsilon + \epsilon.
    \end{align}
\end{proof}

\begin{proof}[Theorem \ref{thm:sample_efficiency}]
    We first show (\ref{required_samples_our_modification}). By Hoeffding's inequality \citep{hoeffding1963probability} and Bool's inequality, for any $\delta_\mathcal{M} \in (0,1)$ and any $(s,a) \in S \setminus E \times A$, if 
    \begin{align}
        N(s,a) \geq \frac{2}{\varepsilon^2} \ln \frac{2|\mathrm{Tr(\mathcal{M}})|}{\delta_\mathcal{M}},
    \end{align}
    then we have $\mathbb{P}_\mathcal{D}(||P - \hat{P}_\delta||_\infty \leq \varepsilon) \geq 1 - \delta_\mathcal{M}$, where $\delta = \delta_\mathcal{M} / |\mathrm{Tr}(\mathcal{M})|$.
    We consider the two cases $\Delta_c > 0$ and $\Delta_c < 0$:
    \begin{enumerate}
        \item Suppose $\Delta_c > 0$. Then, let $\hat{\Delta}_c = \underline{x}^{\mathrm{min}, *}_\mathcal{M}(c) - \overline{x}^{\mathrm{max}, *}_{\mathcal{M}^{[c]}}(s_I)$. We have
        \begin{align}
            |\hat{\Delta}_c - \Delta_c| & \leq |\underline{x}^{\mathrm{min}, *}_\mathcal{M}(c) - \mathrm{Pr}^\mathrm{min}_{\mathcal{M}, c}(\eventually E)| + |\overline{x}^{\mathrm{max}, *}_{\mathcal{M}^{[c]}}(s_I) - \mathrm{Pr}^\mathrm{max}_{\mathcal{M}^{[c]}}(\eventually E)|, \nonumber \\
            \intertext{by Lemma \ref{lemma:operators_monotonically_error_decreasing}, for any $\epsilon > 0$, there is $T = \mathrm{max} \{T^\mathrm{min}_{\mathcal{M}, \epsilon, c}, T^\mathrm{max}_{\mathcal{M}^{[c]}, \epsilon, s_I} \}$ such that}
            & \leq |(f^{\mathrm{min}, \mathrm{min}}_{\mathcal{M}, \delta})^T[\underline{x}_0](c) - \mathrm{Pr}^\mathrm{min}_{\mathcal{M}, c}(\eventually E)| + |(f^{\mathrm{max}, \mathrm{max}}_{\mathcal{M}^{[c]}, \delta})^T[\overline{x}_0](s_I) - \mathrm{Pr}^\mathrm{max}_{\mathcal{M}^{[c]}}(\eventually E)| \nonumber \\
            & \leq |(f^{\mathrm{min}, \mathrm{min}}_{\mathcal{M}, \delta})^T[\underline{x}_0](c) - \mathrm{Pr}^{\pi_\mathrm{min}}_{\mathcal{M}, c}(\eventually^{\leq T} E) + \mathrm{Pr}^{\pi_\mathrm{min}}_{\mathcal{M}, c}(\eventually^{\leq T} E) - \mathrm{Pr}^\mathrm{min}_{\mathcal{M}, c}(\eventually E)| \nonumber \\
            & \hspace{5mm} + |(f^{\mathrm{max}, \mathrm{max}}_{\mathcal{M}^{[c]}, \delta})^T[\overline{x}_0](s_I) - \mathrm{Pr}^{\pi_\mathrm{max}}_{\mathcal{M}^{[c]}}(\eventually^{\leq T} E) + \mathrm{Pr}^{\pi_\mathrm{max}}_{\mathcal{M}^{[c]}}(\eventually^{\leq T} E) - \mathrm{Pr}^{\mathrm{max}}_{\mathcal{M}^{[c]}}(\eventually E)|, \nonumber \\
            \intertext{by Lemma \ref{lemma:T_step_operators_bound} and (\ref{Pc_error_bound}) in Lemma \ref{lemma:transition_probability_bound_modified_MDPs}, we have}
            & \leq 2 T d_\mathrm{max} \varepsilon + 2 \epsilon,
        \end{align}
        where $\underline{x}_0:S \to [0,1]$ and $\overline{x}_0:S \to [0,1]$ are lower and upper bounds for $\mathrm{Pr}^\mathrm{min}_{\mathcal{M}, c}(\eventually E)$ and $\mathrm{Pr}^\mathrm{max}_{\mathcal{M}^{[c]}}(\eventually E)$, respectively, such that $\underline{x}_0 \leq f^{\mathrm{min}, \mathrm{min}}_{\mathcal{M}, \delta}[\underline{x}_0]$ and $\overline{x}_0 \geq f^{\mathrm{max}, \mathrm{max}}_{\mathcal{M}, \delta}[\overline{x}_0]$, and $\pi_\mathrm{min}$ and $\pi_\mathrm{max}$ are the policies such that $\mathrm{Pr}^{\pi_\mathrm{min}}_{\mathcal{M}, \cdot}(\eventually E) = \mathrm{Pr}^\mathrm{min}_{\mathcal{M}, \cdot}(\eventually E)$ and $\mathrm{Pr}^{\pi_\mathrm{max}}_{\mathcal{M}^{[c]}, \cdot}(\eventually E) = \mathrm{Pr}^\mathrm{max}_{\mathcal{M}^{[c]}, \cdot}(\eventually E)$, respectively.
        Since $|\hat{\Delta}_c - \Delta_c| < |\Delta_c|$ implies that the sign of $\hat{\Delta}_c$ is equal to the sign of $\Delta_c$, we choose $\epsilon = \frac{\Delta_c}{8}$ and $\varepsilon = \frac{\Delta_c}{8 T d_\mathrm{max}}$, resulting $N(s,a)$ that satisfies (\ref{required_sample_number_ours_exact_coefficients}) implies $\hat{\Delta}_c > 0$:
        \begin{align}
            \label{required_sample_number_ours_exact_coefficients}
            N(s,a) \geq 128 \frac{d_\mathrm{max}^2 T^2}{\Delta_c^2} \ln \frac{2|\mathrm{Tr}(\mathcal{M})|}{\delta_\mathcal{M}}.
        \end{align}
        \item We can prove that (\ref{required_sample_number_ours_exact_coefficients}) is a sufficient condition for $\overline{x}^{\mathrm{min}, *}_\mathcal{M}(c) < \underline{x}^{\mathrm{max}, *}_{\mathcal{M}^{[c]}}(s_I)$ when $\Delta_c < 0$ in the same way as in case 1.
    \end{enumerate}
    Next, we show (\ref{required_samples_existing_modification}). We consider the  two cases $\Delta_{\mathrm{mr}, c} > 0$ and $\Delta_{\mathrm{mr}, c} < 0$:
    \begin{enumerate}
        \item Suppose $\Delta_{\mathrm{mr}, c} > 0$. Then, let $\hat{\Delta}_{\mathrm{mr}, c} = \underline{x}^{\mathrm{min}, *}_\mathcal{M}(c) - \overline{x}^{\mathrm{max}, *}_{\mathcal{M}^{[c]}}(s_I)$. By Lemmas \ref{lemma:T_step_operators_bound} and \ref{lemma:operators_monotonically_error_decreasing} and (\ref{Pc_rm_error_bound}) in Lemma \ref{lemma:transition_probability_bound_modified_MDPs}, for any $\epsilon > 0$, there is $T_\mathrm{mr} = \min \{ T > 0 \mid T \geq \mathrm{max} \{T^\mathrm{min}_{\mathcal{M}, \epsilon, c}, T^\mathrm{max}_{\mathcal{M}^{[c]}_\mathrm{mr}, \epsilon, s_I} \} \}$ with $ \epsilon = \frac{|\Delta_{\mathrm{mr},c}|}{8(T d_\mathrm{max} + 1)}$ such that
        \begin{align}
            |\Delta_{\mathrm{mr}, c} - \hat{\Delta}_{\mathrm{mr}, c}| & \leq 2T_\mathrm{mr} d_\mathrm{max} (T^\mathrm{min}_{\mathcal{M}, \epsilon, c} d_\mathrm{max} \varepsilon + \epsilon) + 2 \epsilon \nonumber \\
            & \leq 2T_\mathrm{rm}^2 d^2_\mathrm{max} \varepsilon + 2(T_\mathrm{mr} d_\mathrm{max} + 1)\epsilon.
        \end{align}
        Since $|\hat{\Delta}_{\mathrm{mr},c} - \Delta_{\mathrm{rm},c}| < |\Delta_{\mathrm{mr},c}|$ implies that the sign of $\hat{\Delta}_{\mathrm{mr},c}$ is equal to the sign of $\Delta_{\mathrm{mr},c}$, we choose $\varepsilon = \frac{\Delta_{\mathrm{mr}, c}}{8 T_\mathrm{rm}^2 d^2_\mathrm{max}}$, resulting $N(s,a)$ that satisfies (\ref{required_sample_number_existing_exact_coefficients}) implies $\hat{\Delta}_c > 0$:
        \begin{align}
            \label{required_sample_number_existing_exact_coefficients}
            N(s,a) \geq 128 \frac{d_\mathrm{max}^4 T_\mathrm{rm}^4}{\Delta_{\mathrm{rm}, c}^2} \ln \frac{2|\mathrm{Tr}(\mathcal{M})|}{\delta_\mathcal{M}}.
        \end{align}
        \item We can prove that (\ref{required_sample_number_existing_exact_coefficients}) is a sufficient condition for $\overline{x}^{\mathrm{min}, *}_\mathcal{M}(c) < \underline{x}^{\mathrm{max}, *}_{\mathcal{M}^{[c]}_\mathrm{rm}}(s_I)$ when $\Delta_{\mathrm{rm},c} < 0$ in the same way as in case 1.
    \end{enumerate}
\end{proof}

\begin{lemma}
\label{lemma:correct_update}
    For any MDP $\mathcal{M}$, any $\delta > 0$, any multi-sample $\mathcal{D}$ of transitions, and any lower and upper bounds $\underline{x}, \overline{x}: S \to [0,1]$ for $\mathrm{Pr}^{\mathrm{opt}}_{\mathcal{M}, s}(\eventually E)$ with $\mathrm{opt} \in \{ \mathrm{min}, \mathrm{max} \}$, if $P \geq \hat{P}_{\delta}$ holds, then the updated functions $f^{\mathrm{opt}, \mathrm{min}}_{\mathcal{M}, \delta}[\underline{x}]$ and $f^{\mathrm{opt}, \mathrm{max}}_{\mathcal{M}, \delta}[\overline{x}]$ obtained from (\ref{opt12_update}), respectively, are also lower and upper bounds for $\mathrm{Pr}^\mathrm{opt}_{\mathcal{M}, s}(\eventually E)$. 
\end{lemma}
\begin{proof}
    For any $s \in S$, we have
    \begin{align}
        \mathrm{Pr}^\mathrm{opt}_{\mathcal{M}, s}(\eventually E) & \geq \mathrm{opt}_{a \in \mathcal{A}(s)} \sum_{s' \in \mathrm{Post}(s,a)} P(s,a,s') \underline{x}(s') \nonumber \\
        & = \mathrm{opt}_{a \in \mathcal{A}(s)} \sum_{s' \in \mathrm{Post}(s,a)} \hat{P}_\delta(s,a,s') \underline{x}(s') + \sum_{s' \in \mathrm{Post}(s,a)} \left( P(s,a,s') - \hat{P}_\delta(s,a,s') \right) \underline{x}(s'),
        \intertext{since $P(s,a,s') > \hat{P}_\delta(s,a,s')$, we have}
        & \geq \mathrm{opt}_{a \in \mathcal{A}(s)} \sum_{s' \in \mathrm{Post}(s,a)} \hat{P}_\delta(s,a,s') \underline{x}(s') + \min_{s' \in \mathrm{Post}(s,a)} \underline{x}(s') \left( 1 - \sum_{s' \in \mathrm{Post}(s,a)} \hat{P}_\delta(s,a,s') \right).
    \end{align}
    Likewise, for any $s \in S$, we have $\mathrm{Pr}^\mathrm{opt}_{\mathcal{M}, s}(\eventually E) \leq f^{\mathrm{opt}, \mathrm{max}}_{\mathcal{M}, \delta}[\overline{x}](s)$.
    % \begin{align}
        % \mathrm{Pr}^\mathrm{min}_\mathcal{M}(s \models \eventually E) & \leq \min_{a \in \mathcal{A}(s)} \sum_{s' \in \mathrm{Post}(s,a)} P(s,a,s') x^u(s') \nonumber \\
        % & \leq \min_{a \in \mathcal{A}(s)} \sum_{s' \in \mathrm{Post}(s,a)} \hat{P}_\delta(s,a,s') x^u(s') + \max_{s' \in \mathrm{Post}(s,a)} x^u(s') \left( 1 - \sum_{s' \in \mathrm{Post}(s,a)} \hat{P}_\delta(s,a,s') \right).
    % \end{align}
\end{proof}

% \begin{lemma}
    % For any MDP $\mathcal{M}$, any $c \in S$, and any fixed points $\underline{L}_\mathcal{M}$, $\underline{U}_\mathcal{M}$, $\overline{L}_{\mathcal{M}_L^{[c]}}$ and $ \overline{U}_{\mathcal{M}_U^{[c]}}$ of (\ref{under_lower_update}), (\ref{under_upper_update}) (\ref{over_lower_update}), and (\ref{over_upper_update}), respectively, and any $p_c \in [\underline{L}_\mathcal{M}(c), \underline{U}_\mathcal{M}(c)]$, we have that $\overline{L}_{\mathcal{M}_L^{[c]}}(s_I) \leq \overline{L}_{\mathcal{M}^{[c]}(p_c)}(s_I)$ and $\overline{U}_{\mathcal{M}_U^{[c]}}(s_I) \geq \overline{U}_{\mathcal{M}^{[c]}(p_c)}(s_I)$.
% \end{lemma}

In the following, we provide a key lemma to guarantee the correctness of Algorithm \ref{alg:PACIdentification}.

\begin{lemma}
\label{lemma:Algorithm1_is_anytime_under_probability_bound}
    For any MDP $\mathcal{M}$, any $\delta_\mathcal{M} > 0$, and any multi-sample $\mathcal{D}$ of transitions, if $\hat{P}_\delta \leq P$, then $C_\top$, $C_\bot$, $C_?$ obtained from each iteration in Algorithm \ref{alg:PACIdentification}, respectively, satisfy the following properties:
    \begin{enumerate}
        \item $\forall c \in C_\top$, $\{c\} \models \mathbf{(C1)}$,
        \item $\forall c \in C_\bot$, $\{c\} \not \models \mathbf{(C1)}$.
        \item $\forall c \in C_?,  |\mathrm{Pr}^\mathrm{min}_{\mathcal{M}}(\eventually E \mid \eventually c) - \mathrm{Pr}^\mathrm{max}_\mathcal{M}(\eventually E \mid \neg \eventually c)| \leq \tau $.
        % \item for any SPR cause $C \subseteq C_\top$, $\min_\pi \mathrm{Pr}^\pi_\mathcal{M}(\eventually E \;|\; \eventually C^*) - \mathrm{Pr}^\pi_\mathcal{M}(\eventually E \;|\; \eventually C) \geq 0$.
    \end{enumerate}
    Moreover, if $\hat{P}_\delta \leq P$, $C^* = \{ c \in C_\top \mid \exists \rho \text{ s.t. } \rho \models \neg C_\top \mathrm{U} c \}$ obtained from Line 15 in Algorithm \ref{alg:PACIdentification} satisfies that, for any SPR cause $C \subseteq S \setminus C_?$, 
    \begin{align}
    \label{recall_optimality}
        \min_\pi \mathrm{Pr}^\pi_\mathcal{M}(\eventually C^* \mid \eventually E) - \mathrm{Pr}^\pi_\mathcal{M}(\eventually C \mid \eventually E) \geq 0.
    \end{align}
\end{lemma}

\begin{proof}
First, we show properties 1 and 2.
    Under $\hat{P}_\delta \leq P$, by Lemma \ref{lemma:correct_update}, if $\underline{x}^\mathrm{min}_\mathcal{M}(c) > \overline{x}^\mathrm{max}_{\mathcal{M}^{[c]}}(s_I)$ (resp, $\overline{x}^\mathrm{min}_\mathcal{M}(c) < \underline{x}^\mathrm{max}_{\mathcal{M}^{[c]}}(s_I)$) in Line 7, then we have $\mathrm{Pr}^\mathrm{min}_{\mathcal{M}, c}(\eventually E) > \mathrm{Pr}^\mathrm{max}_{\mathcal{M}^{[c]}}(\eventually E)$ (resp., $\mathrm{Pr}^\mathrm{min}_{\mathcal{M}, c}(\eventually E) < \mathrm{Pr}^\mathrm{max}_{\mathcal{M}^{[c]}}(\eventually E)$).
    Thus, by Proposition \ref{prop:soundness_completeness_restart_modification}, $c$ is a causal state (resp., a non-causal state).
    Hence, we have $\{c\} \models \mathbf{(C1)}$ for any $c \in C_\top$ and $\{c\} \not\models \mathbf{(C1)}$ for any $c \in C_\bot$. 
    Next, we show property 3.
    For any $\tau > 0$ and any $c \in C_?$, under under $\hat{P}_\delta \leq P$, if we have 
    \begin{align}
        [\underline{x}^{\mathrm{min},*}_\mathcal{M}(c) - \overline{x}^{\mathrm{max},*}_{\mathcal{M}^{[c]}}(s_I) , \overline{x}^{\mathrm{min},*}_\mathcal{M}(c) - \underline{x}^{\mathrm{max},*}_{\mathcal{M}^{[c]}}(s_I)] \subseteq [-\tau, \tau],
    \end{align}
    then, by Lemma \ref{lemma:correct_update}, we have 
    \begin{align}
        |\mathrm{Pr}^\mathrm{min}_{\mathcal{M}}(\eventually E | \eventually c) - \mathrm{Pr}^\mathrm{max}_\mathcal{M}(\eventually E | \neg \eventually c)| \leq \tau.
    \end{align}
    Then we show (\ref{recall_optimality}), under $\hat{P}_\delta \leq P$. In Line 14 in Algorithm \ref{alg:PACIdentification}, we have $S\setminus C_? = C_\top \cup (C_\bot \cup C_\bot^{\mathrm{pre}})$. Consider any SPR cause $C \subseteq S \setminus C_?$. We then have $C \subseteq C_\top$ by property 2 and Proposition \ref{prop:soundness_completeness_restart_modification}. Thus, by the construction of $C^*$, we have 
    \begin{align}
        \forall \rho \in \mathrm{InfPath},\; \rho \models \eventually C \implies \eventually C^*,
    \end{align}
    and hence, for any policy $\pi$,
    \begin{align}
        \mathrm{Pr}^\pi_\mathcal{M}(\eventually C \land \eventually E) \leq \mathrm{Pr}^\pi_\mathcal{M}(\eventually C^* \land \eventually E).
    \end{align}
    Therefore, $C^*$ satisfies (\ref{recall_optimality}).
\end{proof}

\begin{proof}[Proof of Theorem \ref{thm:Algorithm1_is_anytime}] For any $c \in C_\top$, there exists an iteration number $i$ such that $c$ is assigned to $C_\top$ at $i$-th iteration and transition data $\mathcal{D}_i$ sampled up to $i$-th iteration. Then, by Lemma \ref{lemma:Algorithm1_is_anytime_under_probability_bound}, we have 
\begin{align}
    \mathbb{P}_{\mathcal{D}_i}(\{c\} \models \mathbf{(C1)} ) & \geq \mathbb{P}_{\mathcal{D}_i}(\{c\} \models \mathbf{(C1)} \;|\; \hat{P}_\delta \leq P) \mathbb{P}_\mathcal{D}( \hat{P}_\delta \leq P ) \nonumber \\
    & = \mathbb{P}_{\mathcal{D}_i}( \hat{P}_\delta \leq P ),
    \intertext{where $\mathbb{P}_{\mathcal{D}_i}$ is the probability measure for $\mathcal{D}_i$. By Lemma \ref{lemma:transition_lower_bound}, we have $\mathbb{P}_{\mathcal{D}_i}(\hat{P}_\delta \leq P) \geq 1 - \delta_\mathcal{M}$ and hence, }
    & \geq 1 - \delta_\mathcal{M}.
\end{align}
We analogously derive properties 2 and 3 in Theorem \ref{thm:Algorithm1_is_anytime} by Lemmas \ref{lemma:transition_lower_bound} and \ref{lemma:Algorithm1_is_anytime_under_probability_bound}.
% --- Note: (i) $x^l_\mathcal{M}(c) > y^u_{\mathcal{M}_u^{[c]}}(s_I)$ and (ii) $x^u_\mathcal{M}(c) < y^l_{\mathcal{M}_l^{[c]}}(s_I)$ implies $\mathrm{Pr}^\mathrm{min}_\mathcal{M}(c \models \eventually E) > \mathrm{Pr}^\mathrm{max}_{\mathcal{M}^{[c]}(p_c)}(s_I \models \eventually E)$ and $\mathrm{Pr}^\mathrm{min}_\mathcal{M}(c \models \eventually E) < \mathrm{Pr}^\mathrm{max}_{\mathcal{M}^{[c]}(p_c)}(s_I \models \eventually E)$, respectively. Hence, $\{c\}$ is SPR cause and not if (i) and (ii) hold, respectively, by Lemma 4 in \citep{}.
\end{proof}

\begin{proof}[Proof of Theorem \ref{thm:Algorithm1_is_solution}]
By Lemma \ref{lemma:Algorithm1_is_anytime_under_probability_bound} and Def. \ref{def:SPR_cause}, each SPR cause included by $S \setminus C_?$ is a subset of $C_\top$ under $\hat{P}_\delta \leq P$. Hence, for any multi-sample $\mathcal{D}$ of transitions, we have
\begin{align}
    \mathbb{P}_{\mathcal{D}}(C^* \mbox{ satisfies } (\ref{PAC-PR-recall_optimal}) ) & \geq \mathbb{P}_\mathcal{D}( C^* \mbox{ satisfies } (\ref{PAC-PR-recall_optimal}) \mid \hat{P}_\delta \leq P) \mathbb{P}_\mathcal{D}( \hat{P}_\delta \leq P ) \nonumber \\
    & = \mathbb{P}_\mathcal{D}( \hat{P}_\delta \leq P ),
    \intertext{where $\mathbb{P}_\mathcal{D}$ is the probability measure for $\mathcal{D}$. By Lemma \ref{lemma:transition_lower_bound}, we have $\mathbb{P}_\mathcal{D}(\hat{P}_\delta \leq P) \geq 1 - \delta_\mathcal{M}$ and hence, }
    & \geq 1 - \delta_\mathcal{M}.
\end{align}
\end{proof}

Moreover, the following Propositions \ref{prop:fixed_point} and \ref{prop:convergence_property} guarantee that fixed point operations in Lines 10 to 16 in Algorithm \ref{alg:PACIdentification} converge if $\hat{P}_\delta \leq P$, and Algorithm \ref{alg:PACIdentification} converges as the number of transition samples goes to infinity under mild conditions.

For any MDP $\mathcal{M}$, a pair $(S', A')$ of a subset of states $S' \subseteq S$ and a subset of actions $A' \subseteq \cup_{s\in S'}\mathcal{A}(s)$ such that $A' \neq \emptyset$ is called an end component (EC) if and only if, (i) for any $s \in S'$ and any $a \in A' \cap \mathcal{A}(s)$, we have $\{ s' \in S \mid P(s,a, s') > 0 \} \subseteq S'$, and (ii) for any $s, s' \in S'$, there exists a path from $s$ to $s'$ under $A'$.
An EC $\mathcal{E} = (S', A')$ is called a maximal EC (MEC) if there is no other EC $\mathcal{E}' = (S'', A'')$ such that $S' \subset S''$ or $A' \subset A''$. We say that an MEC $(S', A')$ is a singleton MEC if and only if $|S'| = 1$. Moreover, an MEC $(S', A')$ is called a bottom MEC (BMEC) if and only if $A' = \cup_{s \in S'}\mathcal{A}(s)$. We assume that all MECs of $\mathcal{M}$ are composed of singleton MECs and BMECs\footnote{Any non-bottom MEC can be replaced with a single state \citep{de1998formal, haddad2018interval}}.

Our restart-based modification for MDPs, as defined in Def. \ref{def:restart_based_modification}, may introduce a non-singleton and non-bottom MEC. So, we compress the MEC into a single state in a standard way \citep{de1998formal} if such an MEC exists in $\mathcal{M}^{[c]}$.
The following proposition characterizes the MECs in $\mathcal{M}^{[c]}$.
\begin{proposition}
\label{prop:MEC_characterization_modifiedMDP}
    For any MDP $\mathcal{M}$ and any $c \in S$, if there exists a non-bottom and non-singleton MEC $(S', A')$ in the modified MDP $\mathcal{M}^{[c]}$, then we have $s_I, c \in S'$.
\end{proposition}
\begin{proof}
    We prove by contradiction. Suppose that there exists a non-bottom MEC $(S', A')$ in $\mathcal{M}^{[c]}$ and $A' \neq \bigcup_{s \in S'} \mathcal{A}(s)$ such that $|S'| > 1$ and (1) $c \not\in S'$ or (2) $s_I \not\in S'$ but $c \in S'$.
    \begin{enumerate}
        \item Since every MEC in $\mathcal{M}$ is either a singleton MEC or a BMEC, and all transitions from states other than $c$ on $\mathcal{M}^{[c]}$ are identical to those on $\mathcal{M}$ by Def. \ref{def:restart_based_modification}, for any $S'' \subseteq S\setminus \{c\}$ and any $A' \subseteq \bigcup_{s \in S''} \mathcal{A}(s)$ such that $|S''| > 1$, if $(S'', A'')$ is an MEC, then it is a BMEC. This contradicts the assumption. 
        \item By Def. \ref{def:restart_based_modification}, if $c \in S'$, then $s_I \in S'$. However, this contradicts the assumption.
    \end{enumerate}
\end{proof}

% By Proposition \ref{prop:MEC_characterization_modifiedMDP}, $\mathcal{M}^{[c]}$ contains at most one nontrivial and non-bottom MEC, and it always contains $c$. 
% We denote the compressed single state from the appeared MEC by $\hat{c}$ if such an MEC exists, and otherwise, $\hat{c} = c$. 
With a slight abuse of notation, we also denote the modified MDP after the state compression by $\mathcal{M}^{[c]} = (S, A, s_I, P^c)$. 
In Propositions \ref{prop:fixed_point} and \ref{prop:convergence_property}, we treat the reduced version as the modified MDP.

\begin{proposition}
\label{prop:fixed_point}
    For any MDP $\mathcal{M}$, any multi-sample $\mathcal{D}$ of transitions, and any $c \in S$, $\underline{x}^\mathrm{min}_\mathcal{M}$, $\overline{x}^\mathrm{min}_\mathcal{M}$, $\underline{x}^\mathrm{max}_{\mathcal{M}^{[c]}}$, and $\overline{x}^\mathrm{max}_{\mathcal{M}^{[c]}}$ converge to fixed points with probability at least $ 1- \delta_\mathcal{M}$, respectively.
\end{proposition}

\begin{assumption}
    \label{assume:transition_sampling}
    For any MDP $\mathcal{M}$, any transition $(s,a,s') \in S \times A \times S$, if $s$ is reachable from $s_I$, $a \in \mathcal{A}(s)$, and $P(s,a,s') < 1$, then, in Algorithm \ref{alg:PACIdentification}, $|\mathcal{D}_{s,a,s'}| \to \infty$ with probability $1$ as the number of times transitions are sampled goes to $\infty$, where $\mathcal{D}_{s,a,s'}$ denotes the number of samples of $(s,a,s')$.
\end{assumption}

\begin{proposition}
\label{prop:convergence_property}
    For any MDP $\mathcal{M}$ and any $\delta > 0$, under Assumption \ref{assume:transition_sampling}, as the number of multi-sample $\mathcal{D}$ goes to infinity, with probability $1$, Algorithm \ref{alg:PACIdentification} converges and $C^*$ that satisfies (\ref{PAC-PR-recall_optimal}) when $\tau=0$.
    % \begin{enumerate}
        % \item Algorithm \ref{alg:PACIdentification} converges with probability $1$.
        % \item $C^*$ obtained from Algorithm \ref{alg:PACIdentification} satisfies (\ref{PAC-PR-recall_optimal}) with probability $1$.
    % \end{enumerate}
\end{proposition}

\begin{proof}[Proof of Proposition \ref{prop:fixed_point}]
Let $\{ \underline{x}^{\mathrm{min}, (n)}_\mathcal{M} \}^\infty_{n=0}$ and $\{\overline{x}^{\mathrm{min}, (n)}_\mathcal{M} \}^\infty_{n=0}$ be the sequences derived by repeatedly applying $f^{\mathrm{min}, \mathrm{min}}_{\mathcal{M}, \delta}$ and $f^{\mathrm{min}, \mathrm{max}}_{\mathcal{M}, \delta}$ defined by (\ref{opt12_update}), respectively, to the functions initialized in Lines 6 and 7, that is, $\underline{x}^{\mathrm{min}, (i+1)}_\mathcal{M} = f^{\mathrm{min}, \mathrm{min}}_{\mathcal{M}, \delta}[\underline{x}^{\mathrm{min}, (i)}_\mathcal{M}]$ and $\overline{x}^{\mathrm{min}, (i+1)}_\mathcal{M} = f^{\mathrm{min}, \mathrm{max}}_{\mathcal{M}, \delta}[\overline{x}^{\mathrm{min}, (i)}_\mathcal{M}]$ for any $i = 0, \ldots$.
By Lemma \ref{lemma:monotone_continuos}, under $\hat{P}_\delta \leq P$, there exists $j > 0$ such that the sequences $\{ \underline{x}^{\mathrm{min}, (n)}_\mathcal{M} \}^\infty_{n=j}$ and $\{ \overline{x}^{\mathrm{min}, (n)}_\mathcal{M} \}^\infty_{n=j}$ are monotonically increasing and decreasing, that is, $\underline{x}^{\mathrm{min}, (i)}_\mathcal{M} \leq \overline{x}^{\mathrm{min}, (i+1)}_\mathcal{M}$ and $\overline{x}^{\mathrm{min}, (i)}_\mathcal{M} \geq \overline{x}^{\mathrm{min}, (i+1)}_\mathcal{M}$ for any $i \geq j$. Thus, $\underline{x}^{\mathrm{min}, (\infty)}_\mathcal{M}$ and $\overline{x}^{\mathrm{min}, (\infty)}_\mathcal{M}$ are fixed points $\underline{x}^{\mathrm{min}, *}_\mathcal{M}$ and $\overline{x}^{\mathrm{min}, *}_\mathcal{M}$ of $f^{\mathrm{min}, \mathrm{min}}_{\mathcal{M}, \delta}$ and $f^{\mathrm{min}, \mathrm{max}}_{\mathcal{M}, \delta}$, respectively, since $S$ is finite and $[0,1]^{|S|}$ are closed. Thus, by Lemma \ref{lemma:transition_lower_bound}, for each sequence $\{x^{\mathrm{min}, (n)}_\mathcal{M}\}_n = \{\underline{x}^{\mathrm{min}, (n)}_\mathcal{M}\}_n, \{\overline{x}^{\mathrm{min}, (n)}_\mathcal{M}\}_n$, we have
\begin{align}
    \mathbb{P}_\mathcal{D}(\lim_{n \to \infty} x^{\mathrm{min}, (\infty)}_\mathcal{M} = x^{\mathrm{min}, *}_\mathcal{M}) & \geq \mathbb{P}_\mathcal{D}(\lim_{n \to \infty} x^{\mathrm{min}, (\infty)}_\mathcal{M} = x^{\mathrm{min}, *}_\mathcal{M} \mid \hat{P}_\delta \leq P) \mathbb{P}_\mathcal{D}(\hat{P}_\delta \leq P) \nonumber \\
    & \geq 1 - \delta_\mathcal{M},
\end{align}
where $\mathbb{P}_\mathcal{D}$ is the probability measure for $\mathcal{D}$.

In the same way, we can show that, for each sequence $\{x^{\mathrm{max}, (n)}_{\mathcal{M}^{[c]}}\}_n  = \{\underline{x}^{\mathrm{max}, (n)}_{\mathcal{M}^{[c]}}\}_n, \{\overline{x}^{\mathrm{max}, (n)}_{\mathcal{M}^{[c]}}\}_n $, we have
\begin{align}
    \mathbb{P}_\mathcal{D}(\lim_{n \to \infty} x^{\mathrm{max}, (n)}_{\mathcal{M}^{[c]}} = x^{\mathrm{max}, *}_{\mathcal{M}^{[c]}} & \geq \mathbb{P}_\mathcal{D}(\lim_{n \to \infty} x^{\mathrm{max}, (n)}_{\mathcal{M}^{[c]}} = x^{\mathrm{max}, *}_{\mathcal{M}^{[c]}} \mid \hat{P}_\delta \leq P) \mathbb{P}_\mathcal{D}(\hat{P}_\delta \leq P) \\
    & \geq 1 -\delta_\mathcal{M}.
\end{align}
    % - By Lemma \ref{lemma:monotone_continuos}, the sequences obtained from applying the operators to the initial functions converge.\\
    % --- Point memo: Since $[0,1]^{S}$ is closed and the operators are monotonic, the sequences converge limit functions $x^{l,(\infty)}, y^{l,(\infty)}, x^{u,(\infty)}, y^{u,(\infty)}$. By the continuity of the operators, the limit functions are fixed points, because $f(x^{(\infty)}) = \lim_{n \to \infty} f(x^{(n)}) = \lim_{n \to \infty} x^{(n+1)} = x^{(\infty)}$.\\
    % - By Tarski's fixed point theorem, the limit function $x^{(\infty)}$ is a fixed point in $[0,1]^S$: \\
    % --- Monotonicity and continuity of update operators.\\
    % --- MDP is finite and hence $[0,1]^{S}$ is a complete lattice.\\
    % --- If an operator over a complete lattice is monotone increasing and continuous, then the operator has a fixed point in $[0,1]^S$. \\
\end{proof}

\begin{proof}[Proof of Proposition \ref{prop:convergence_property}]
    By law of large numbers and Assumption \ref{assume:transition_sampling}, $\hat{P}_\delta$ converges pointwise to $P$ with probability $1$ as $|\mathcal{D}|$ goes to $\infty$.
    Hence, as $|\mathcal{D}|$ goes to $\infty$, $f^{\mathrm{min}, \mathrm{min}}_{\mathcal{M}, \delta}$ and $f^{\mathrm{min}, \mathrm{max}}_{\mathcal{M}, \delta}$ defined by (\ref{opt12_update}), respectively, converge to the following operator $f^\mathrm{min}_{\mathcal{M}} : [0,1]^S \to [0,1]^S$ with probability $1$: for any $x: S \to [0,1]$ and any $s \in S$,
    \begin{align}
        f^\mathrm{min}_{\mathcal{M}}[x](s) = \min_{a \in \mathcal{A}(s)} \sum_{s' \in \mathrm{Post}(s,a)} P(s,a,s')x(s').
    \end{align}
    Likewise, when $|\mathcal{D}|$ goes to $\infty$, for any $c \in S$,  $f^{\mathrm{max}, \mathrm{min}}_{\mathcal{M}^{[c]}, \delta}$ and $f^{\mathrm{max}, \mathrm{max}}_{\mathcal{M}^{[c]}, \delta}$ defined by (\ref{opt12_update}), respectively, are identical to the following operator $f^\mathrm{max}_{\mathcal{M}^{[c]}} : [0,1]^S \to [0,1]^S$ with probability $1$: for any $x: S \to [0,1]$ and any $s \in S$,
    \begin{align}
        f^\mathrm{max}_{\mathcal{M}^{[c]}}[x](s) = \max_{a \in \mathcal{A}(s)} \sum_{s' \in \mathrm{Post}(s,a)} P^c(s,a,s')x(s'),
    \end{align}
    where $P^c$ is the transition probability function of $\mathcal{M}^{[c]}$. 
    Since the sets of states of $\mathcal{M}$ are composed of non-ECs and BMECs, $\underline{x}^\mathrm{min}_\mathcal{M}(c)$ and $\overline{x}^\mathrm{min}_\mathcal{M}(c)$ converge to $\mathrm{Pr}^\mathrm{min}_{\mathcal{M}, c}(\eventually E)$ by Propositions 2 and 3 in \citep{haddad2018interval}. Likewise, $\underline{x}^\mathrm{max}_{\mathcal{M}^{[c]}}(s_I)$ and $\overline{x}^\mathrm{max}_{\mathcal{M}^{[c]}}(s_I)$ converge to $\mathrm{Pr}^\mathrm{max}_{\mathcal{M}^{[c]}}(\eventually E)$ by Propositions 4 and 5 in \citep{haddad2018interval}. Thus, when $\tau = 0$ and $|\mathcal{D}|$ goes to $\infty$, by Proposition \ref{prop:soundness_completeness_restart_modification}, we have,
    \begin{align}
        & C_\top = \{ c \in S \setminus C^\mathsf{pre}_\bot \mid \{c\} \models \mathbf{(C1)} \}, \\
        & C_\bot = \{ c \in S \setminus C^\mathsf{pre}_\bot \mid \{c\} \not\models \mathbf{(C1)} \} \cup C^\mathsf{pre}_\bot,
    \end{align}
    with probability $1$, where $C^\mathsf{pre}_\bot = \{ c \in S \mid \Pi_{\mathcal{M}, c} = \emptyset \} \cup E$.
    Therefore, with probability $1$, Algorithm \ref{alg:PACIdentification} converges and, by Lemma \ref{lemma:Algorithm1_is_anytime_under_probability_bound}, outputs $C^*$ that satisfies (\ref{PAC-PR-recall_optimal}).
\end{proof}

We show that the soundness of our restart-based MDP modification remains intact even when considering the existential-quantifier variant of SPR causes.
\begin{definition}[Existential SPR Cause]
    \label{def:eSPR_cause}
    For any MDP $\mathcal{M}$, the set of terminal states $E \subseteq S$, and a nonempty subset $C \subseteq S \setminus E$, we say that $C$ is an existential strict probability-raising (eSPR) cause for $E$ on $\mathcal{M}$ if and only if the following condition hold:
    \begin{align}
    \label{eSPR:probability_rasiing}
    \forall c \in C, \exists \pi \mbox{ s. t. } \mathrm{Pr}_{\mathcal{M}}^\pi(\neg C \mathrm{U} c) \in (0, 1) \mbox{ and }
    \mathrm{Pr}_{\mathcal{M}}^\pi(\eventually E \;|\; \neg C \mathrm{U} c) > \mathrm{Pr}_{\mathcal{M}}^\pi(\eventually E ).
    \end{align}
    % \item[(M)] $\forall c \in C, \exists \pi$ s. t.
\end{definition}
The difference of Def. \ref{def:eSPR_cause} from Def. \ref{def:SPR_cause} comes from only the existential quantifier for the probability-raising condition (\ref{eSPR:probability_rasiing}). Intuitively, the condition requires that, for each state $c \in C$, there exists a policy under which traversal of $c$ increases the likelihood of reaching $E$.
% SPR cause is defined as a nonredundant subset $C$ of states that act as waypoints increasing the probability of reaching terminal states $E$ under any policy.

In the following, we denote the minimal and maximal probability on $\mathcal{M}$ over a set of policies $\Pi$ by $\mathrm{Pr}^{\mathrm{min}_{\Pi}}_{\mathcal{M}, s}$ and $\mathrm{Pr}^{\mathrm{max}_{\Pi}}_{\mathcal{M}, s}$, respectively. 
% We denote the set of finite paths on $\mathcal{M}^{[c]}$ for each $s \in S \setminus E$ by $\mathrm{FinPath}^{[c]}$

\begin{proposition}[Soundness of Restart-based Modification for Def. \ref{def:eSPR_cause}]
   \label{prop:soundness_completeness_restart_modification_for_existential}
    For any MDP $\mathcal{M}$ and any $c \in S$, if $\Pi_{\mathcal{M}, c} = \emptyset$, then $c$ is not causal, otherwise, for any $\Pi \subseteq \{ \pi: S \times A \to [0,1] \mid \mathrm{Pr}^\pi_{\mathcal{M}}(\neg \eventually c) > 0 \}$, we have the following properties:
    \begin{enumerate}
        \item If $\mathrm{Pr}^\mathrm{max}_{\mathcal{M}, c}(\eventually E) > \mathrm{Pr}^{\mathrm{min}_{\Pi}}_{\mathcal{M}^{[c]}}(\eventually E)$, then $\{c\}$ is an eSPR cause.
        \item If $\mathrm{Pr}^\mathrm{max}_{\mathcal{M}, c}(\eventually E) < \mathrm{Pr}^{\mathrm{min}}_{\mathcal{M}^{[c]}}(\eventually E)$, then $\{c\}$ is not an eSPR cause.
        % \item If $\mathrm{Pr}^\mathrm{max}_{\mathcal{M}, c}(\eventually E) = \mathrm{Pr}^{\mathrm{min}_{\overline{A}_c}}_{\mathcal{M}^{[c]}}(\eventually E)$, then we have the following properties:
        % \begin{enumerate}
            % \item if $c$ is not reachable from $s_I$ in $\mathcal{M}_\mathrm{min}^{[c]}$, $c$ is causal,
            % \item otherwise, $c$ is not causal,
        % \end{enumerate}
    \end{enumerate}
    where $\mathcal{M}^{[c]}$ is defined by Def. \ref{def:restart_based_modification} and $\Pi_{\mathcal{M}, c}$ is defined by (\ref{policy_set_M_c}).
    % $\mathcal{M}_\mathrm{max}^{[c]}$ is the sub-MDP of $\mathcal{M}^{[c]}$ induced by $\mathcal{A}^\mathrm{min}(s) = \{ a \in A \mid \mathrm{Pr}^{\mathrm{min}_{\overline{A}_c}}_{\mathcal{M}^{[c]}, s}(\eventually E) = \sum_{s' \in S} P^c(s,a,s') \mathrm{Pr}^{\mathrm{min}_{\overline{A}_c}}_{\mathcal{M}^{[c]}, s'}(\eventually E) \}$,.
\end{proposition}
\begin{proof}
% Suppose $\Pi_{\mathcal{M}, c} = \emptyset$. Then, we have the two cases: (i) for any $\pi$, $\mathrm{Pr}^\pi_\mathcal{M}(\eventually c) = 0$, and (ii) for any $\pi$, $\mathrm{Pr}^\pi_\mathcal{M}(\eventually c) = 1$. In case (i), $\{c\}$ violates the condition \textbf{(C2)} in Def. \ref{def:SPR_cause}. In case (ii), $\{c\}$ violates the condition \textbf{(eS)} in Def. \ref{def:SPR_cause}. Next, we consider $\Pi_{\mathcal{M}, c} \neq \emptyset$.
Suppose $\Pi_{\mathcal{M}, c} = \emptyset$. Then, $\{c\}$ does not satisfy (\ref{eSPR:probability_rasiing}) in Def.~\ref{def:eSPR_cause}. Next, we consider $\Pi_{\mathcal{M}, c} \neq \emptyset$.
    \begin{enumerate}
        \item Suppose $\mathrm{Pr}^\mathrm{max}_{\mathcal{M}, c}(\eventually E) > \mathrm{Pr}^{\mathrm{min}_{\Pi}}_{\mathcal{M}^{[c]}}(\eventually E)$. Then, there exists a two-memory policy $\pi_\mathrm{2mem} : \mathrm{FinPath} \to A$ such that $\pi_\mathrm{2mem}$ is identical to a memoryless policy $\pi_\mathrm{max}$ with $\mathrm{Pr}^{\mathrm{max}}_{\mathcal{M}, c}(\eventually E) = \mathrm{Pr}^\mathrm{max}_{\mathcal{M}, c}(\eventually E)$ after hitting $c$, and identical to a memoryless policy $\pi_\mathrm{min}$ with $\mathrm{Pr}^{\pi_\mathrm{min}}_{\mathcal{M}^{[c]}}(\eventually E) = \mathrm{Pr}^{\mathrm{min}_{\Pi}}_{\mathcal{M}^{[c]}}(\eventually E)$ before hitting $c$. Moreover, there exists a policy $\pi$ such that $\pi$ is identical to a memoryless policy $\pi' \in \Pi_{\mathcal{M}, c}$ before hitting $c$, and identical $\pi_\mathrm{max}$ after hitting $c$.
        % and hence we have $\mathrm{Pr}^\mathrm{max}_{\mathcal{M}, c}(\eventually E) > \mathrm{Pr}^{\pi_\mathrm{2mem}}_{\mathcal{M}^{[c]}}(\eventually E) \geq \mathrm{Pr}^{\mathrm{min}_{\Pi}}_{\mathcal{M}^{[c]}}(\eventually E)$.
        By Lemma \ref{lemma:policy_convex_combination_any_measurable}, for any $\lambda \in (0,1)$, there exists a policy $\pi_\lambda$ such that, for any measurable set $\mathcal{T} \subseteq \mathrm{InfPath}$, we have
        \begin{align}
            & \mathrm{Pr}^{\pi_\lambda}_{\mathcal{M}}(\mathcal{T}) = \lambda \mathrm{Pr}^{\pi_\mathrm{2mem}}_{\mathcal{M}}(\mathcal{T}) + (1 - \lambda) \mathrm{Pr}^\pi_{\mathcal{M}}(\mathcal{T}),
        \end{align}
         % where $\mathrm{InfPath}$ is the set of infinite path from $s_I$.
        Hence,
        \begin{align}
            & \mathrm{Pr}^{\pi_\lambda}_{\mathcal{M}}(\eventually c) \geq (1-\lambda) \mathrm{Pr}^{\pi}_{\mathcal{M}}(\eventually c) > 0, \\
            & \mathrm{Pr}^{\pi_\lambda}_{\mathcal{M}}(\neg \eventually c) \geq (1-\lambda) \mathrm{Pr}^{\pi}_{\mathcal{M}}(\neg \eventually c) > 0.
        \end{align}
        % Moreover, we have
        % \begin{align}
            % \mathrm{Pr}^{\pi_\lambda}_{\mathcal{M}^{[c]}}(\eventually E) & \geq \lambda \mathrm{Pr}^{\mathrm{min}_{\Pi}}_{\mathcal{M}^{[c]}}(\eventually E) + (1 - \lambda) \mathrm{Pr}^\pi_{\mathcal{M}^{[c]}}(\eventually E), \nonumber \\
            % \intertext{by Lemmas \ref{lemma:conditioned_reachability_to_reachability_modified_MDP} and \ref{lemma:policy_convex_combination_any_measurable},}
            % & = \lambda \mathrm{Pr}^{\mathrm{min}_{\Pi}}_{\mathcal{M}}(\eventually E \mid \neg \eventually c) + (1 - \lambda) \mathrm{Pr}^\pi_{\mathcal{M}}(\eventually E \mid \neg \eventually c) \nonumber \\
            % & = \mathrm{Pr}^{\pi_\lambda}_{\mathcal{M}}(\eventually E \mid \neg \eventually c).
        % \end{align}
        We have 
        \begin{align}
            \mathrm{Pr}^{\pi_\lambda}_{\mathcal{M}}(\eventually E \mid \eventually c) - \mathrm{Pr}^{\pi_\lambda}_\mathcal{M}(\eventually E) & = (1 - \mathrm{Pr}^{\pi_\lambda}_\mathcal{M}(\eventually c)) (\mathrm{Pr}^{\pi_\lambda}_{\mathcal{M},c}(\eventually E) - \mathrm{Pr}^{\pi_\lambda}_{\mathcal{M}}(\eventually E \mid \neg \eventually c)) \nonumber \\
            & = (1 - \mathrm{Pr}^{\pi_\lambda}_\mathcal{M}(\eventually c)) (\mathrm{Pr}^{\mathrm{max}}_{\mathcal{M},c}(\eventually E) - \mathrm{Pr}^{\pi_\lambda}_{\mathcal{M}}(\eventually E \mid \neg \eventually c))
        \end{align}
        We consider the two cases.
        \begin{enumerate}
            \item $\mathrm{Pr}^{\pi}_{\mathcal{M}}(\eventually E \mid \neg \eventually c) = \mathrm{Pr}^{\pi_\mathrm{min}}_{\mathcal{M}}(\eventually E \mid \neg \eventually c)$: note that $\mathrm{Pr}^\pi_{\mathcal{M}}$ we have 
            \begin{align}
                \mathrm{Pr}^{\pi_\lambda}_{\mathcal{M}}(\eventually E \mid \neg \eventually c) & = \mathrm{Pr}^{\pi_\mathrm{min}}_{\mathcal{M}}(\eventually E \mid \neg \eventually c),
                \intertext{by Lemma \ref{lemma:conditioned_reachability_to_reachability_modified_MDP},}
                & = \mathrm{Pr}^{\mathrm{min}_\Pi}_{\mathcal{M}^{[c]}}(\eventually E).
            \end{align}
            Thus, by the assumption, $\mathrm{Pr}^{\pi_\lambda}_{\mathcal{M}}(\eventually E \mid \eventually c) > \mathrm{Pr}^{\pi_\lambda}_\mathcal{M}(\eventually E)$.
            \item $\mathrm{Pr}^{\pi}_{\mathcal{M}}(\eventually E \mid \neg \eventually c) > \mathrm{Pr}^{\pi_\mathrm{min}}_{\mathcal{M}}(\eventually E \mid \neg \eventually c)$: then, 
            \begin{align}
                \mathrm{Pr}^{\pi_\lambda}_{\mathcal{M}}(\eventually E \mid \neg \eventually c) & = \frac{\lambda \mathrm{Pr}^{\pi_\mathrm{2mem}}_{\mathcal{M}}(\eventually E \land \neg \eventually c) + (1 - \lambda) \mathrm{Pr}^\pi_{\mathcal{M}}(\eventually E \land \neg \eventually c)}{\lambda \mathrm{Pr}^{\pi_\mathrm{2mem}}_{\mathcal{M}}(\neg \eventually c) + (1 - \lambda) \mathrm{Pr}^\pi_{\mathcal{M}}(\neg \eventually c)} \nonumber \\
                & = \frac{\lambda \mathrm{Pr}^{\pi_\mathrm{min}}_{\mathcal{M}}(\eventually E \land \neg \eventually c) + (1 - \lambda) \mathrm{Pr}^\pi_{\mathcal{M}}(\eventually E \land \neg \eventually c)}{\lambda \mathrm{Pr}^{\pi_\mathrm{min}}_{\mathcal{M}}(\neg \eventually c) + (1 - \lambda) \mathrm{Pr}^\pi_{\mathcal{M}}(\neg \eventually c)}
            \end{align}
            is continuous with respect to $\lambda$. Moreover, as $\lambda \to 1$, we have $\mathrm{Pr}^{\pi_\lambda}_{\mathcal{M}}(\eventually E \mid \neg \eventually c) = \mathrm{Pr}^{\pi_\mathrm{min}}_{\mathcal{M}}(\eventually E \mid \neg \eventually c)$. So, by Lemma \ref{lemma:conditioned_reachability_to_reachability_modified_MDP}, $\mathrm{Pr}^{\pi_\lambda}_{\mathcal{M}}(\eventually E \mid \neg \eventually c) = \mathrm{Pr}^{\mathrm{min}_\Pi}_{\mathcal{M}^{[c]}}(\eventually E)$ as $\lambda \to 1$. Hence, combined with the assumption, there exists $\lambda < 1$ such that $\mathrm{Pr}^{\mathrm{max}}_{\mathcal{M},c}(\eventually E) > \mathrm{Pr}^{\pi_\lambda}_{\mathcal{M}}(\eventually E \mid \neg \eventually c)$. Thus, $\mathrm{Pr}^{\pi_\lambda}_{\mathcal{M}}(\eventually E \mid \eventually c) > \mathrm{Pr}^{\pi_\lambda}_\mathcal{M}(\eventually E)$.

            % by the assumption, we can choose $\lambda$ such that $\lambda > \frac{\mathrm{Pr}^{\pi}_{\mathcal{M}^{[c]}}(\eventually E) - \mathrm{Pr}^{\mathrm{max}}_{\mathcal{M},c}(\eventually E)}{\mathrm{Pr}^{\pi}_{\mathcal{M}^{[c]}}(\eventually E) - \mathrm{Pr}^{\mathrm{min}_\Pi}_{\mathcal{M}^{[c]}}(\eventually E)}$ and $\lambda \in (0, 1)$. Thus, $\mathrm{Pr}^{\pi_\lambda}_{\mathcal{M}}(\eventually E \mid \eventually c) > \mathrm{Pr}^{\pi_\lambda}_\mathcal{M}(\eventually E)$ .
        \end{enumerate}
        % By construction of $\pi_\lambda$, we have $\mathrm{Pr}^{\pi_\lambda}_{\mathcal{M}^{[c]}}(\eventually E) \leq \mathrm{Pr}^\mathrm{max}_{\mathcal{M}, c}(\eventually E)$, hence,
        % \begin{align}
            % \mathrm{Pr}^{\pi_\lambda}_{\mathcal{M}}(\eventually E \mid \eventually c) - \mathrm{Pr}^{\pi_\lambda}_\mathcal{M}(\eventually E) > 0.
        % \end{align}
        \item Suppose that $\mathrm{Pr}^\mathrm{max}_{\mathcal{M}, c}(\eventually E) < \mathrm{Pr}^{\mathrm{min}}_{\mathcal{M}^{[c]}}(\eventually E)$. Then, by Lemma $\ref{lemma:conditioned_reachability_to_reachability_modified_MDP}$, for any $\pi \in \Pi_{\mathcal{M}, c}$, there exists a policy $\tilde{\pi}$ such that $\mathrm{Pr}^{\tilde{\pi}}_{\mathcal{M}^{[c]}}(\eventually E) = \mathrm{Pr}^\pi_{\mathcal{M}}(\eventually E \mid \neg \eventually c)$
        \begin{align}
            \mathrm{Pr}^\pi_{\mathcal{M}}(\eventually E \mid \eventually c) - \mathrm{Pr}^\pi_\mathcal{M}(\eventually E)
            & = (1 - \mathrm{Pr}^\pi_\mathcal{M}(\eventually c)) (\mathrm{Pr}^\pi_{\mathcal{M}, c}(\eventually E) - \mathrm{Pr}^{\tilde{\pi}}_{\mathcal{M}^{[c]}}(\eventually E)), \nonumber \\
            \intertext{hence,}
            & \leq (1 - \mathrm{Pr}^\pi_\mathcal{M}(\eventually c)) (\mathrm{Pr}^\mathrm{max}_{\mathcal{M}, c}(\eventually E) - \mathrm{Pr}^\mathrm{min}_{\mathcal{M}^{[c]}}(\eventually E)) \nonumber \\
            & < 0.
        \end{align}
    \end{enumerate}
\end{proof}
\paragraph{Example of $\Pi$ in Proposition \ref{prop:soundness_completeness_restart_modification_for_existential} based on MEC decomposition}
We can provide an example of $\Pi$ for case 2 in Proposition \ref{prop:soundness_completeness_restart_modification_for_existential} as follows. We assume that every MEC in $\mathcal{M}^{[c]}$ is either a singleton MEC or a BMEC. The goal is to exclude policies $\pi$ such that $\mathrm{Pr}^\pi_{\mathcal{M}}(\eventually c) = 1$ by forbidding actions that keep the process inside the MEC containing $c$ in $\mathcal{M}^{[c]}$. By Proposition \ref{prop:MEC_characterization_modifiedMDP}, any non-singleton, non-bottom MEC in $\mathcal{M}^{[c]}$ must contains both $s_I$ and $c$. Hence, preventing the policy from choosing the ''staying" actions in the MEC rules out the almost-sure reachability to $c$ from $s_I$.
For each $c \in S$, we denote by $(S_c, A_c)$ the MEC in $\mathcal{M}^{[c]}$ such that $c \in S_c$. Then, we choose $\Pi$ as the set of policies that never choose actions in $A_c$ when the current state is in $S_c$:
\begin{align}
    \Pi = \{ \pi : \mathrm{FinPath} \times A \to [0,1] \mid \forall \rho \in \mathrm{FinPath}, \mathrm{last}(\rho) \in S_c \implies \sum_{a \in A_c \cap \mathcal{A}(s)} \pi(\rho, a) = 0\}.
\end{align}
% Intuitively, $\Pi$ denotes the set of policies that prohibit the actions in $A_c$ when staying in $S_c$. 
% We can compute or learn $\mathrm{Pr}^{\mathrm{min}_\Pi}_{\mathcal{M}^{[c]}}(\eventually E)$ by disabling all actions in $A_c$ at any state in $S_c$. 
% Note that, if $\Pi = \emptyset$, then $\Pi_{\mathcal{M}, c} = \emptyset$. Hence, $c$ is not causal in this case by Proposition \ref{prop:soundness_completeness_restart_modification_for_existential}.

\section{Initialization and Early Stopping for Algorithm \ref{alg:PACIdentification}}
\label{appendix:initializtion_early_stopping}
Our restart-based modification defined by Def. \ref{def:restart_based_modification} sometimes induces a slow computation for the upper bounds of $\mathrm{Pr}^\mathrm{max}_{\mathcal{M}^[[c]}(\eventually E)$ if the initialized function $x$ for $f^{\mathrm{max}, \mathrm{max}}_{\mathcal{M}^{[c]}, \delta}$ defined by (\ref{opt12_update}) is loose, e.g., $x(s) = 0$ for any $s \in S^\mathrm{max}_{\mathcal{M}}$ and $x(s) = 1$ for any other states $s \not\in S^\mathrm{max}_\mathcal{M}$. This is because the optimized policy under a loosely initialized function tends to make a behavior: repeatedly reaching $c$ and resetting to $s_I$ with a high probability. To mitigate this issue, we introduce a two-step initialization using a loosely initialized function.
Given an MDP $\mathcal{M}$ and $\delta \in (0, 1)$, we consider the following two operators $g_{\mathcal{M}^{[c]}, \delta}, F_{\mathcal{M}^{[c]}, \delta}: [0, 1]^S \to [0, 1]^S$. For any $x \in [0,1]^S$ and any $c \in S \setminus E$, $g_{\mathcal{M}^{[c]}, \delta}[x](s) = 0$ for any $s \in S^\mathrm{max}_{\mathcal{M}} \cup \{c \}$, $g_{\mathcal{M}^{[c]}, \delta}[x](s) = 1$ for any $s \in E$, and $g_{\mathcal{M}^{[c]}, \delta}[x](s) = f^{\mathrm{max}, \mathrm{max}}_{\mathcal{M}^{[c]}, \delta}[x](s)$ defined by (\ref{opt12_update}) for any $s \in S \setminus (E \cup S^\mathrm{max}_\mathcal{M} \cup \{ c \})$. For any $x \in [0, 1]^S$ and any $s \in S$, $F_{\mathcal{M}^{[c]}, \delta}[x](s) = \max \{x(s), f^{\mathrm{max}, \mathrm{max}}_{\mathcal{M}^{c]}, \delta}[x](s) \}$.
By definition, for any $x \in [0, 1]^S$, we have
\begin{align}
    g_{\mathcal{M}^{[c]}, \delta}[x](s) &\leq f^{\mathrm{max}, \mathrm{max}}_{\mathcal{M}^{[c]}, \delta}[x](s), \\
    F_{\mathcal{M}^{[c]}, \delta}[x](s) &\geq x(s).
\end{align}
Intuitively, the fixed point of $g$ for an initial function $x$ acts as an upper bound of $\mathrm{Pr}^\mathrm{max}_{\mathcal{M}^{[c]}, \cdot}(\eventually E \land \neg \eventually c)$. This alleviates the aforementioned issue. 
However, the fixed point of $g$ is smaller than the fixed point of $ f^{\mathrm{max}, \mathrm{max}}_{\mathcal{M}, \delta}$. So, we increase the fixed point function $x'$ of $g_{\mathcal{M}^{[c]}, \delta}$ so that we have $x' \geq  f^{\mathrm{max}, \mathrm{max}}_{\mathcal{M}, \delta}[x']$, i.e., a pre-fixpoint of $f^{\mathrm{max}, \mathrm{max}}_{\mathcal{M}, \delta}$ by applying $F_{\mathcal{M}^{[c]}, \delta}$.
We summarize the initialization with an early stopping in Algorithm \ref{alg:Initialization_EarlyStopping}. Algorithm \ref{alg:Initialization_EarlyStopping} can be used to implement Line 5 in Algorithm \ref{alg:PACIdentification}. After this computation, we resume Line 6 in Algorithm \ref{alg:PACIdentification}.

\begin{algorithm}[H]
\caption{Initialization and Early Stopping for Algorithm \ref{alg:PACIdentification}}
\label{alg:Initialization_EarlyStopping}
\KwIn{Unknown MDP $\mathcal{M}$, $\delta_\mathcal{M} > 0$, $k > 0$, $\tau \geq 0$, $N > 0$.}
    \tcp{Initialize the lower and upper bounds}
    $\forall s,c \in S,\ \underline{x}^\mathrm{min}_\mathcal{M}(s) \gets 0,\ \overline{x}^\mathrm{min}_\mathcal{M}(s) \gets 1, \underline{x}^\mathrm{max}_{\mathcal{M}^{[c]}}(s) \gets 0,\ \overline{x}^\mathrm{max}_{\mathcal{M}^{[c]}}(s) \gets 1$.\\
    $\forall s \in E, \forall s' \in S^\mathrm{min}_{\mathcal{M}},\ \underline{x}^\mathrm{min}_\mathcal{M}(s) \gets 1,\ \overline{x}^\mathrm{min}_\mathcal{M}(s') \gets 0$.\\
    $\forall c \in S, \forall s \in E, \forall s' \in S^\mathrm{max}_{\mathcal{M}^{[c]}}, \ \underline{x}^\mathrm{max}_{\mathcal{M}^{[c]}}(s) \gets 1,\ \overline{x}^\mathrm{max}_{\mathcal{M}^{[c]}}(s') \gets 0$.\\
    \tcp{Bound updates and early stopping for non-causal states}
    \For{$i=0,\ldots,N$}{
        $\underline{x}^\mathrm{min}_\mathcal{M} \gets f^{\mathrm{min}, \mathrm{min}}_{\mathcal{M}, \delta}[\underline{x}^\mathrm{min}_\mathcal{M}]$\\
        $\overline{x}^\mathrm{min}_\mathcal{M} \gets f^{\mathrm{min}, \mathrm{max}}_{\mathcal{M}, \delta}[\overline{x}^\mathrm{min}_\mathcal{M}]$\\
        \For{$c \in S \setminus (E \cup C_\bot \cup C^\mathrm{pre}_\bot)$}{
            $\underline{x}^\mathrm{max}_{\mathcal{M}^{[c]}} \gets f^{\mathrm{max}, \mathrm{min}}_{\mathcal{M}^{[c]}, \delta}[\underline{x}^\mathrm{max}_{\mathcal{M}^{[c]}}]$\\
            $\overline{x}^\mathrm{max}_{\mathcal{M}^{[c]}} \gets g_{\mathcal{M}^{[c]}, \delta}[\overline{x}^\mathrm{max}_{\mathcal{M}^{[c]}}]$\\
            \If{$\overline{x}^\mathrm{min}_\mathcal{M}(c) < \underline{x}^\mathrm{max}_{\mathcal{M}^{[c]}}(s_I)$}{
                Add $c \to C_\bot$.
            }
        }
    }
    \tcp{Re-compute initial upper bound for maximal reachability in $\mathcal{M}^{[c]}$}
    \Repeat{$\forall s \in S, \overline{x}^\mathrm{max}_{\mathcal{M}^{[c]}}(s) \geq f^{\mathrm{max}, \mathrm{max}}_{\mathcal{M}, \delta}[\overline{x}^\mathrm{max}_{\mathcal{M}^{[c]}}](s) $}{
        $\overline{x}^\mathrm{max}_{\mathcal{M}^{[c]}} \gets F_{\mathcal{M}^{[c]}, \delta}[\overline{x}^\mathrm{max}_{\mathcal{M}^{[c]}}]$.
    }
    \tcp{Bound updates and early stopping for causal and non-causal states}
    \For{$i=0,\ldots,N$}{
        $\underline{x}^\mathrm{min}_\mathcal{M} \gets f^{\mathrm{min}, \mathrm{min}}_{\mathcal{M}, \delta}[\underline{x}^\mathrm{min}_\mathcal{M}]$\\
        $\overline{x}^\mathrm{min}_\mathcal{M} \gets f^{\mathrm{min}, \mathrm{max}}_{\mathcal{M}, \delta}[\overline{x}^\mathrm{min}_\mathcal{M}]$\\
        \For{$c \in S \setminus (E \cup C_\bot \cup C^\mathrm{pre}_\bot)$}{
            $\underline{x}^\mathrm{max}_{\mathcal{M}^{[c]}} \gets f^{\mathrm{max}, \mathrm{min}}_{\mathcal{M}^{[c]}, \delta}[\underline{x}^\mathrm{max}_{\mathcal{M}^{[c]}}]$\\
            $\overline{x}^\mathrm{max}_{\mathcal{M}^{[c]}} \gets g_{\mathcal{M}^{[c]}, \delta}[\overline{x}^\mathrm{max}_{\mathcal{M}^{[c]}}]$\\
            \If{$\overline{x}^\mathrm{min}_\mathcal{M}(c) < \underline{x}^\mathrm{max}_{\mathcal{M}^{[c]}}(s_I)$}{
                Add $c \to C_\bot$.
            }
            \ElseIf{$\underline{x}^\mathrm{min}_\mathcal{M}(c) > \overline{x}^\mathrm{max}_{\mathcal{M}^{[c]}}(s_I)$}{
                Add $c \to C_\top$.
            }
        }
    }
\end{algorithm}

\section{Details of Baseline 1}
\label{appendix:baseline1}
We provide an algorithmic summary of $\mathrm{Baseline1}$, which adapts the model-checking procedure of \citep{baier2022probability} to the unknown-MDP setting by combining the transformed MDP of Def.~\ref{def:existing_MDP_modification} with the interval (two-sided) value iteration of \citep{ashok2019pac}. The procedure shares the same interval value iteration as our method (Algorithm~\ref{alg:PACIdentification}) and differs only in the transformation used to check the probability-raising condition \textbf{(C1)}. $\mathcal{M}^{[c]}_\mathrm{mr}$ requires the computation of $\mathrm{Pr}^\mathrm{min}_{\mathcal{M}, c}(\eventually E)$ for the transition from $c$ to $E$. So, we estimate the lower bound $\hat{P}^c_{\mathrm{mr}, \delta}: S \times A \times S \to [0,1]$ of the transition probability function of $\mathcal{M}^{[c]}_\mathrm{mr}$ as
\begin{align}
\label{P_hat_transformed}
    \hat{P}^c_{\mathrm{mr}, \delta}(s,a,s') = 
    \begin{cases}
        \hat{P}_\delta(s,a,s')  & \mbox{if } s \neq c,\\
        \underline{x}^{\mathrm{min}, *}_\mathcal{M}(c) & \mbox{if } s= c, s' = s_\mathrm{eff}, \\
        1 - \underline{x}^{\mathrm{min}, *}_\mathcal{M}(c) & \mbox{if } s= c, s' = s_\mathrm{noeff}, \\
        0 & \mbox{otherwise, }
    \end{cases}
\end{align}

For each candidate state $c \in S \setminus E$, $\mathrm{Baseline1}$ proceeds as Algorithm \ref{alg:baseline1}.

\begin{algorithm}[h]
\caption{Baseline 1}
\label{alg:baseline1}
\KwIn{Unknown MDP $\mathcal{M}$, $\delta_\mathcal{M} > 0$, $k > 0$, $\tau \geq 0$.}
\KwOut{$C^* \subset S$}

$ \mathcal{D}, C_\top, C_\bot, C_? \gets \emptyset $, and $ \delta \gets \frac{\delta_\mathcal{M}}{\mathrm{Tr}(\mathcal{M})} $.\\
% $ C^\mathsf{pre}_\bot \gets \{ c \in S \mid \Pi_{\mathcal{M}, c} = \emptyset \} \cup E$.\\ 
% Compute MEC decomposition $S = \biguplus_{k=1}^K S_k \uplus \{t_l\}_{l=1}^L \uplus \biguplus_{m=1}^M B_m \uplus E$.
\Repeat{$S \setminus (C_\top \cup C_\bot \cup C_? \cup E) = \emptyset$}{
    % $ \mathcal{D}' \gets \mathrm{TransSample}(\mathcal{M}, k) $ and $ \mathcal{D} \gets \mathcal{D} \cup \mathcal{D}' $.\\
    Sample $k$ transitions in $\mathcal{M}$, add them to $\mathcal{D}$, and compute $\hat{P}_\delta$.\\
    % \tcp{Initialize bounds}
    Initialize lower and upper bounds of $\mathrm{Pr}^\mathrm{min}_{\mathcal{M}}(\eventually E)$: $[\underline{x}^\mathrm{min}_\mathcal{M},  \overline{x}^\mathrm{min}_\mathcal{M}]$.\\
    Compute the fixed points of $\underline{x}^\mathrm{min}_\mathcal{M}$ and $ \overline{x}^\mathrm{min}_\mathcal{M}$ for (\ref{opt12_update}).\\
    For all $c \in S \setminus E$, construct $\mathcal{M}^{[c]}_\mathrm{mr}$ by \eqref{P_hat_transformed} and initialize lower and upper bounds of $\mathrm{Pr}^\mathrm{max}_{\mathcal{M}^{[c]}_\mathrm{mr}}(\eventually E)$: $[\underline{x}^\mathrm{max}_{\mathcal{M}^{[c]}_\mathrm{mr}},  \overline{x}^\mathrm{max}_{\mathcal{M}^{[c]}_\mathrm{mr}}]$. \\
    Compute the fixed points of $[\underline{x}^\mathrm{max}_{\mathcal{M}^{[c]}_\mathrm{mr}}$ and $  \overline{x}^\mathrm{max}_{\mathcal{M}^{[c]}_\mathrm{mr}}]$ for (\ref{opt12_update}).\\
    % $\underline{x}^{\mathrm{min}, *}_\mathcal{M}, \overline{x}^{\mathrm{max}, *}_{\mathcal{M}^{[c]}}, \overline{x}^{\mathrm{min},*}_\mathcal{M}, \underline{x}^{\mathrm{max},*}_{\mathcal{M}^{[c]}} $ for (\ref{opt12_update}) with the initialized bounds.\\
    % $\forall s,c \in S,\ x^l(s) \gets 0,\ x^u(s) \gets 1, y^l_c(s) \gets 0,\ y^u_c(s) \gets 1$.\\
    % $\forall s \in E, \forall s' \in S^\mathrm{min}_{\mathcal{M}},\ x^l(s) \gets 1,\ x^u(s') \gets 0$.\\
    % $\forall c \in S, \forall s \in E, \forall s' \in S^\mathrm{max}_{\mathcal{M}^{[c]}}, \ y^l_c(s) \gets 1,\ y^u_c(s') \gets 0$.\\
    % $\forall s \in S \setminus ( E \cup S^\mathrm{min}_{\mathcal{M},0} ),\ x^l(s) \gets 0,\ x^u(s) \gets 1$.\\
    \For{$c \in S \setminus (C_\top \cup C_\bot \cup C_? \cup E)$}{
        Add $c$ to $C_\top$ if $\underline{x}^{\mathrm{min},*}_\mathcal{M}(c) > \overline{x}^{\mathrm{max},*}_{\mathcal{M}^{[c]}}(s_I) $. \\
        Add $c$ to $C_\bot$ if $\overline{x}^{\mathrm{min},*}_\mathcal{M}(c) < \underline{x}^{\mathrm{max},*}_{\mathcal{M}^{[c]}}(s_I) $. \\
        % \If{$\exists c_\top \in C_\top$ s.t. $\forall \rho, \rho \models \eventually c_\top \Leftrightarrow \eventually c$}{
            % Add $c$ to $C_\top$.
        % }
        \If{$[\underline{x}^{\mathrm{min},*}_\mathcal{M}(c) - \overline{x}^{\mathrm{max},*}_{\mathcal{M}^{[c]}}(s_I) , \overline{x}^{\mathrm{min},*}_\mathcal{M}(c) - \underline{x}^{\mathrm{max},*}_{\mathcal{M}^{[c]}}(s_I)] \subseteq [-\tau, \tau]$ and $c \not\in C_\top \cup C_\bot$}{
            Add $c$ to $C_\top$ if 3-(a) in Lemma 2 in \citep{baier2022probability} is satisfied, otherwise add $c$ to $C_\bot$.
        }
    }
}
\If{ $(C_\top, C_\bot, C_?)$ remains unchanged under $\mathcal{D}$}{
    \textbf{Return} $C^* = \{ c \in C_\top \mid \exists \rho \text{ s.t. } \rho \models \neg C_\top \mathrm{U} c \}$.   
}
Go back to Line 4 with $C_\top, C_\bot, C_? \gets \emptyset$.

\end{algorithm}

\section{Detailed Experimental Results}
\label{appendix:further_experimental_results}
Table \ref{table:correctness_path_planning_all_params} reports the accuracies of states identified as causal/non-causal during learning, that is, $C_\top$ and $C_\bot$, and the final outputs $C^*$ for the nondeterministic path-planning examples under $\delta_\mathcal{M} \in \{0.05, 0.1\}$ and $(p_1, p_2) \in \{(0.4, 0.5), (0.3, 0.6), (0.2, 0.7)\}$. The results are obtained for our method and the baselines. Our method and \textrm{Baseline1} produced correct results; however, \textrm{Baseline1} left $(3,5)$ and $(4,5)$ unclassified. The number of iterations required by our method increases as the gap between $p_1$ and $p_2$ decreases, because the gap between reachability probabilities conditioned on visiting a candidate state and bypassing it also decreases. \textrm{Baseline2} converged fastest in all experiments, but its accuracy dropped below $1-\delta_\mathcal{M}$ when $(p_1, p_2) \in \{(0.4, 0.5), (0.3, 0.6)\}$.

Figs. \ref{fig:causal_noncausal_comparison_54_delta_05}-\ref{fig:causal_noncausal_comparison_72_delta_1} show, for the same parameter settings, the numbers of causal and non-causal states correctly identified at each iteration by our method and \textrm{Baseline1}. In all cases, our method identified both causal and non-causal states more quickly. \textrm{Baseline1} left $(5,3)$ and $(5,4)$ unclassified in all experiments, and it also left $(6,7)$, $(7,7)$, and $(7,8)$ unclassified when $(p_1, p_2) = (0.4, 0.5)$.

Figs. \ref{fig:effect_of_tau_54_delta_05}-\ref{fig:effect_of_tau_72_delta_05} show the number of states labeled as undecided at termination and the number of iterations required by our method. As expected, the number of undecided states increases with $\tau$. Nevertheless, the required number of iterations, and thus the number of transition samples, is substantially reduced at the expense of more undecided states.

\begin{table*}[htbp]
  \caption{Means and standard deviations (stds.) of the accuracies of the identified causal states, non-causal states, and final SPR causes for the nondeterministic path-planning example, as well as of the number of iterations each method required to obtain the final outputs. In all nondeterministic planning experiments, \textrm{Baseline1} timed out.}
  \centering
  \scalebox{0.7}{
  \label{table:correctness_path_planning_all_params}
    \begin{tabular}{ c||c c c c|c c c c|c c c c} \hline
       & \multicolumn{4}{c|}{$(p_1, p_2)=(0.4, 0.5), \delta_\mathcal{M}=0.05$} & \multicolumn{4}{c|}{$(p_1, p_2) = (0.3, 0.6), \delta_\mathcal{M}=0.05$} & \multicolumn{4}{c}{$(p_1, p_2) = (0.2, 0.7), \delta_\mathcal{M}=0.05$} \\ \hline 
       & $C_\top$ & $C_\bot$ & $C^{*}$ & Iter & $ C_\top$ & $C_\bot$ & $C^{*}$ & Iter & $ C_\top$ & $C_\bot$ & $C^{*}$ & Iter \\ \hline \hline
       Ours & $1.0 \pm 0.0$ & $1.0 \pm 0.0$ & $1.0 \pm 0.0$ & $1449 \pm 188$ & $1.0 \pm 0.0$ & $1.0 \pm 0.0$ & $1.0 \pm 0.0$ & $284 \pm 33$ & $1.0 \pm 0.0$ & $1.0 \pm 0.0$ & $1.0 \pm 0.0$ & $168 \pm 5$ \\ 
       \hline
       $\mathrm{Baseline1}$ & $1.0 \pm 0.0$ & $1.0 \pm 0.0$ & -- & TO & $1.0 \pm 0.0$ & $1.0 \pm 0.0$ & -- & TO & $1.0 \pm 0.0$ & $1.0 \pm 0.0$ & -- & TO \\ 
       \hline
       $\mathrm{Baseline2}$ & -- & -- & $0.25 \pm 0.09$ & $55 \pm 69$ & -- & -- & $0.75 \pm 0.10$ & $23 \pm 14$ & -- & -- & $0.95 \pm 0.05$ & $18 \pm 2$ \\ 
       \hline \hline
       & \multicolumn{4}{c|}{$(p_1, p_2)=(0.4, 0.5), \delta_\mathcal{M}=0.1$} & \multicolumn{4}{c|}{$(p_1, p_2) = (0.3, 0.6), \delta_\mathcal{M}=0.1$} & \multicolumn{4}{c}{$(p_1, p_2) = (0.2, 0.7), \delta_\mathcal{M}=0.1$} \\ \hline
       & $C_\top$ & $C_\bot$ & $C^{*}$ & Iter & $ C_\top$ & $C_\bot$ & $C^{*}$ & Iter & $ C_\top$ & $C_\bot$ & $C^{*}$ & Iter \\ \hline \hline
       Ours & $1.0 \pm 0.0$ & $1.0 \pm 0.0$ & $1.0 \pm 0.0$ & $1338 \pm 148$ & $1.0 \pm 0.0$ & $1.0 \pm 0.0$ & $1.0 \pm 0.0$ & $267 \pm 334$ & $1.0 \pm 0.0$ & $1.0 \pm 0.0$ & $1.0 \pm 0.0$ & $156 \pm 5$ \\ 
       \hline
       $\mathrm{Baseline1}$ & $1.0 \pm 0.0$ & $1.0 \pm 0.0$ & -- & TO & $1.0 \pm 0.0$ & $1.0 \pm 0.0$ & -- & TO & $1.0 \pm 0.0$ & $1.0 \pm 0.0$ & -- & TO \\ 
       \hline
       $\mathrm{Baseline2}$ & -- & -- & $0.20 \pm 0.09$ & $44 \pm 61$ & -- & -- & $0.85 \pm 0.08$ & $31 \pm 18$ & -- & -- & $1.0 \pm 0.0$ & $18 \pm 6$ \\ 
       \hline
    \end{tabular}
    }
\end{table*}
\begin{figure}[htbp]
    \centering
    \begin{subfigure}{0.33\linewidth}
        \includegraphics[width=\linewidth]{For_Paper/Figures/MDP2d_2hour_TO_comparison_p_5_and_4_delta_05.png}
        \caption{$(p_1, p_2) = (0.4, 0.5), \delta_\mathcal{M}=0.05$}
        \label{fig:causal_noncausal_comparison_54_delta_05}
    \end{subfigure}
    \begin{subfigure}{0.33\linewidth}
        \includegraphics[width=\linewidth]{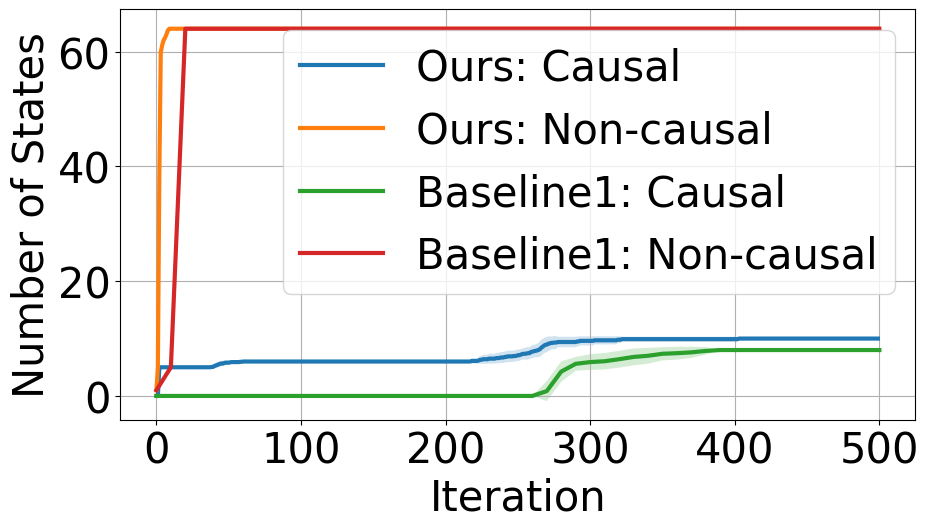}
        \caption{$(p_1, p_2) = (0.3, 0.6), \delta_\mathcal{M}=0.05$}
        \label{fig:causal_noncausal_comparison_63_delta_05}
    \end{subfigure}
    \begin{subfigure}{0.33\linewidth}
        \includegraphics[width=\linewidth]{For_Paper/Figures/MDP2d_2hour_TO_comparison_p_7_and_2_delta_05.png}
        \caption{$(p_1, p_2) = (0.2, 0.7), \delta_\mathcal{M}=0.05$}
        \label{fig:causal_noncausal_comparison_72}
    \end{subfigure}
    \begin{subfigure}{0.33\linewidth}
        \includegraphics[width=\linewidth]{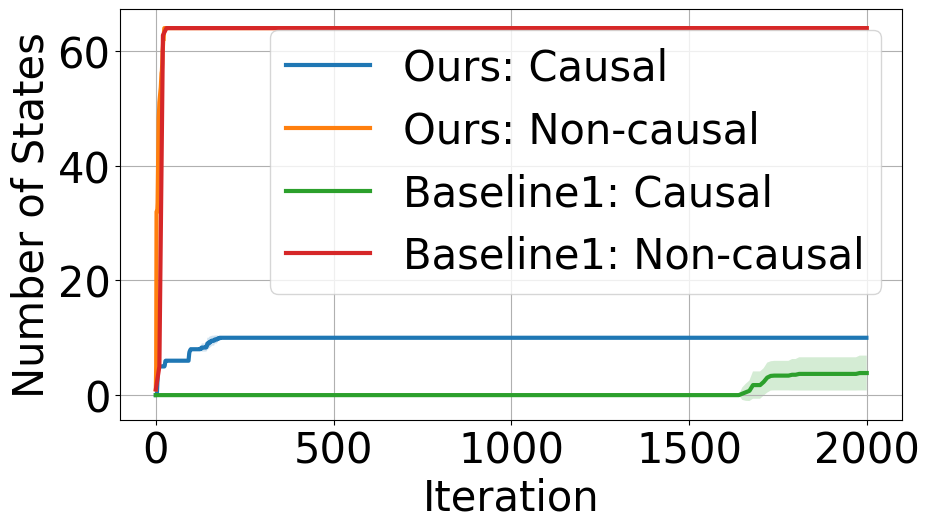}
        \caption{$(p_1, p_2) = (0.4, 0.5), \delta_\mathcal{M}=0.1$}
        \label{fig:causal_noncausal_comparison_54_delta_1}
    \end{subfigure}
    \begin{subfigure}{0.33\linewidth}
        \includegraphics[width=\linewidth]{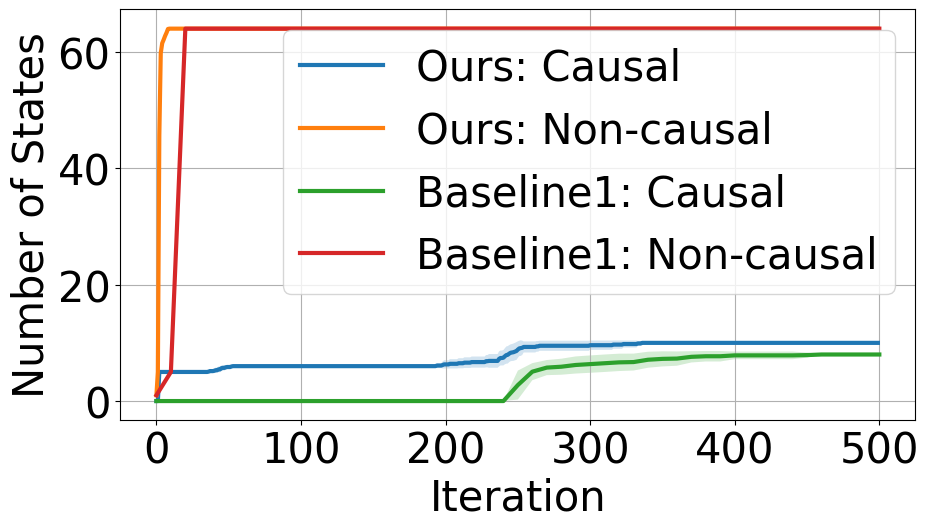}
        \caption{$(p_1, p_2) = (0.3, 0.6), \delta_\mathcal{M}=0.1$}
        \label{fig:causal_noncausal_comparison_63_delta_1}
    \end{subfigure}
    \begin{subfigure}{0.33\linewidth}
        \includegraphics[width=\linewidth]{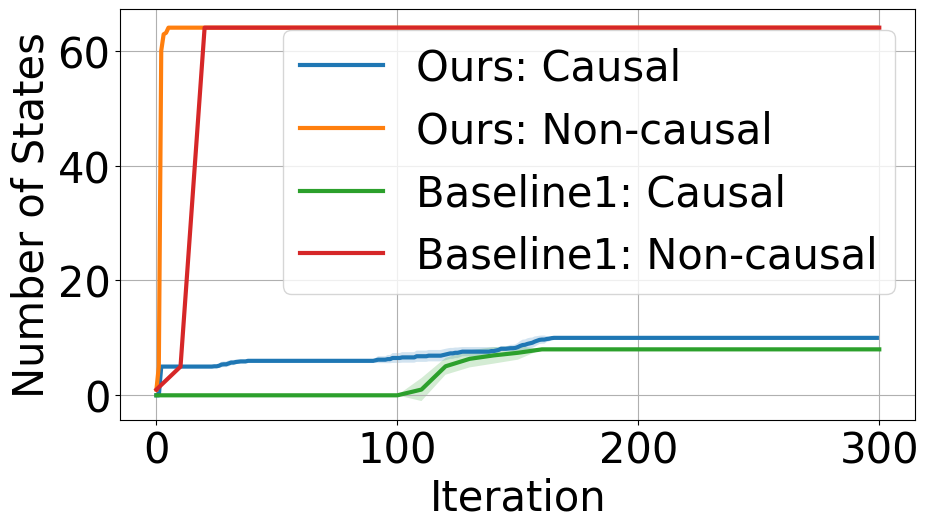}
        \caption{$(p_1, p_2) = (0.2, 0.7), \delta_\mathcal{M}=0.1$}
        \label{fig:causal_noncausal_comparison_72_delta_1}
    \end{subfigure}
    \begin{subfigure}{0.33\linewidth}
        \includegraphics[width=\linewidth]{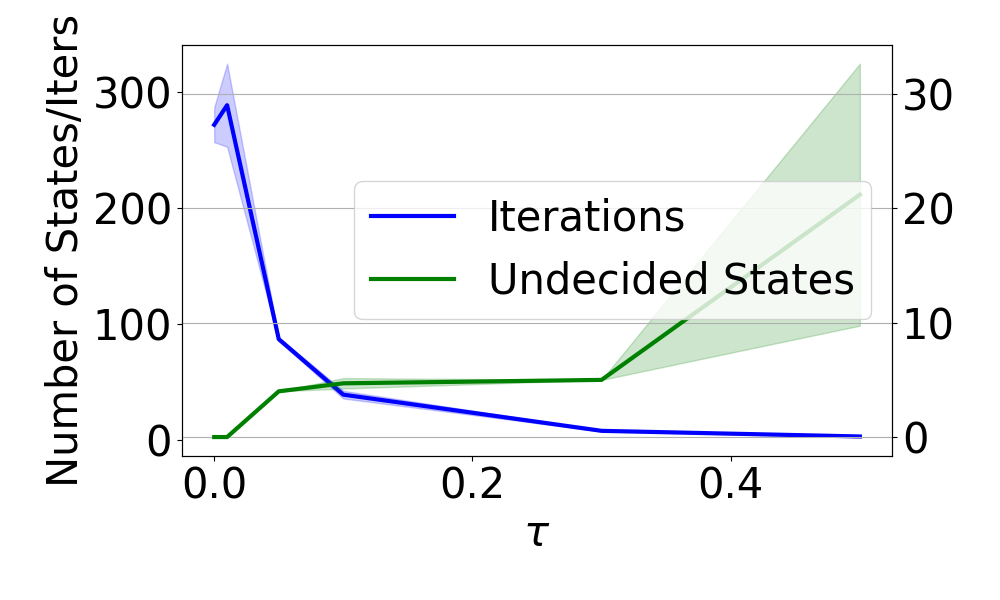}
        \caption{$(p_1, p_2) = (0.4, 0.5)$}
        \label{fig:effect_of_tau_54_delta_05}
    \end{subfigure}
    \begin{subfigure}{0.33\linewidth}
        \includegraphics[width=\linewidth]{For_Paper/Figures/MDP2d_efficient_p_6_and_3_delta_05_tau_effect.png}
        \caption{$(p_1, p_2) = (0.3, 0.6)$}
        \label{fig:effect_of_tau_63_delta_05}
    \end{subfigure}
    \begin{subfigure}{0.33\linewidth}
        \includegraphics[width=\linewidth]{For_Paper/Figures/MDP2d_efficient_p_7_and_2_delta_05_tau_effect.png}
        \caption{$(p_1, p_2) = (0.2, 0.7)$}
        \label{fig:effect_of_tau_72_delta_05}
    \end{subfigure}
    \caption{(a)–(f) show, for the nondeterministic planning setting, the mean numbers of causal and non-causal states identified by our method and \textrm{Baseline1}, respectively. (g)–(i) show, under $\delta_\mathcal{M}=0.05$, the mean number of undecided states, $|C_?|$, and the mean number of iterations for our method. In (a)–(i), shaded areas indicate the standard deviation.}
\end{figure}

To show how causal states affect the system behavior corresponding to Fig. \ref{ex:warehouse}, we depict two states in the SPR cause obtained from our method in Fig. \ref{fig:warehouse_cause}. Each arrow represents the executed actions at the locations. Fig. \ref{fig:warehouse_cause0} shows that the collision probability is increased if the blue robot tries to pick up the item when the green robot is still in the top-left side. Moreover, Fig. \ref{fig:warehouse_cause1} indicates that the collision probability is increased if the green robot reenters the left area when the blue robot is delivering the item.

\begin{figure}[htbp]
    \centering
    \begin{subfigure}{0.25\linewidth}
        \includegraphics[width=\linewidth]{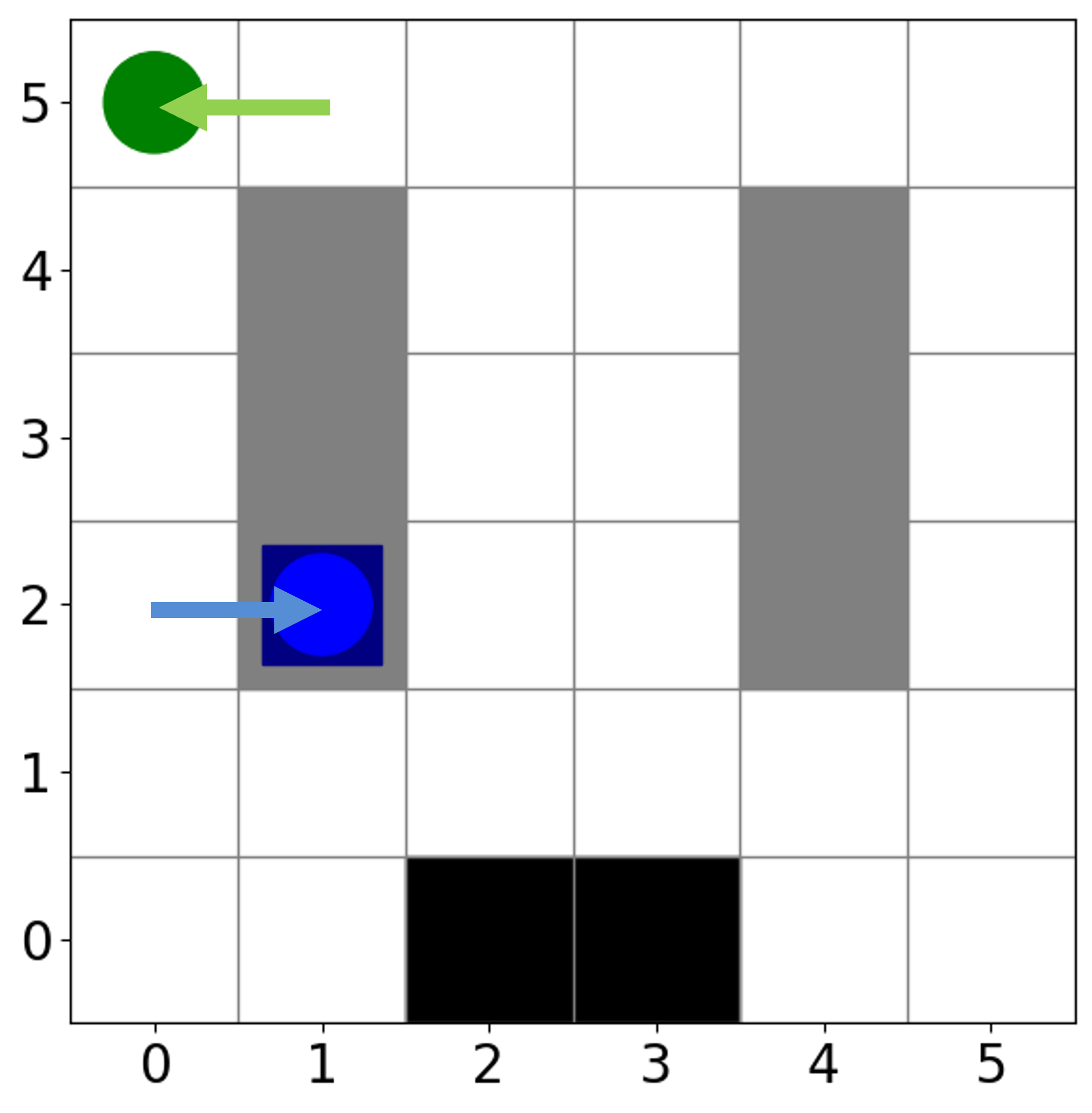}
        \caption{}
        \label{fig:warehouse_cause0}
    \end{subfigure}
    \begin{subfigure}{0.25\linewidth}
        \includegraphics[width=\linewidth]{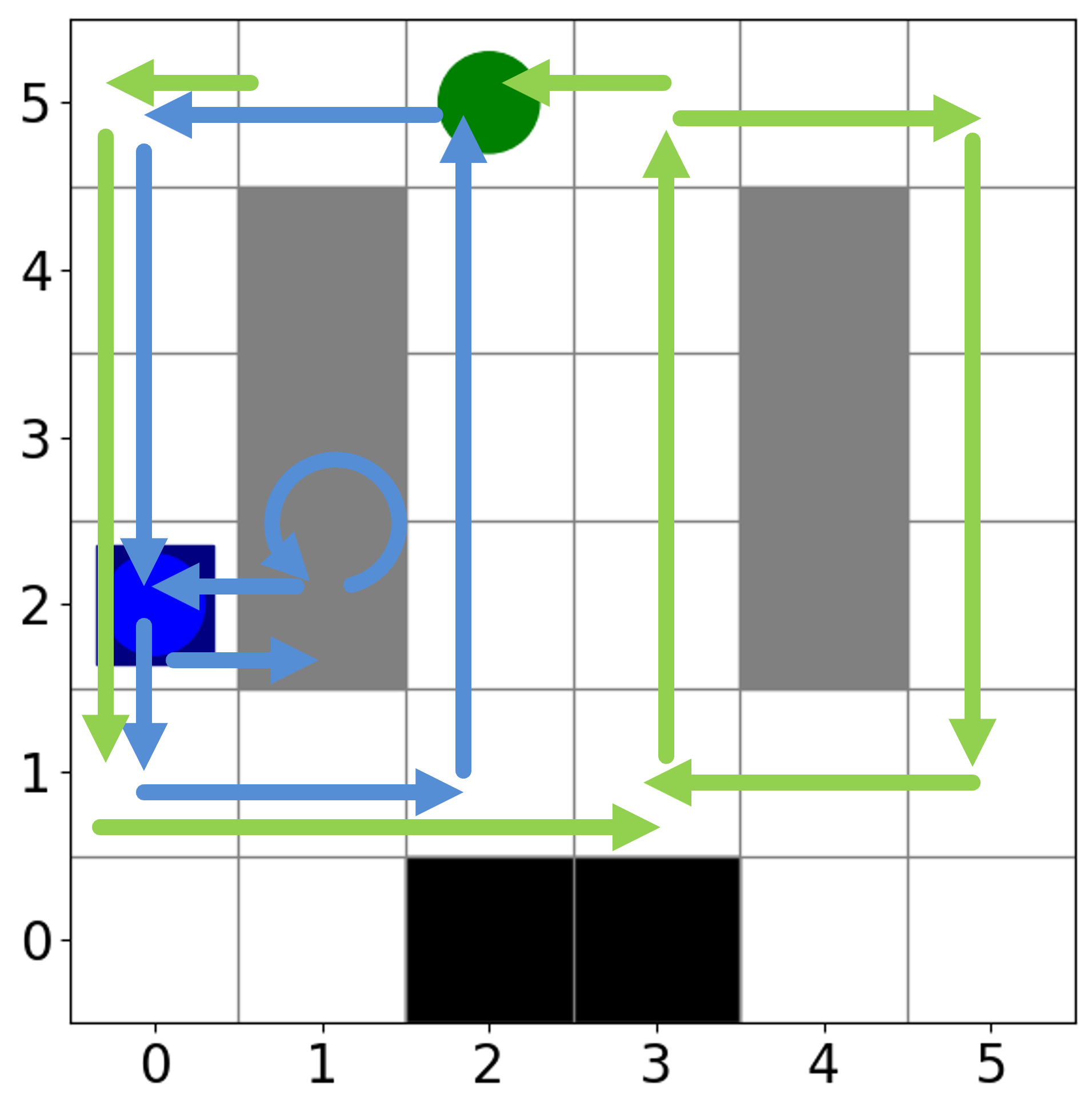}
        \caption{}
        \label{fig:warehouse_cause1}
    \end{subfigure}
    \caption{Two example states in the obtained SPR cause on the scenario shown in Fig. \ref{ex:warehouse} identified by our method.}
    \label{fig:warehouse_cause}
\end{figure}

\section{Parameters for Experiments and Reinforcement Learning}
\label{appendix:experimental_settings}
We summarize the parameter settings in our experiments.
\subsection{Nondeterministic planning scenario}
Environmental parameters are as follows:
\begin{itemize}
    \item $(p_1, p_2)$: $(0.2, 0.7), (0.3, 0.6), (0.4, 0.5)$.
    \item Transition probability to the intended direction: $0.9$.
    \item Transition probability to the perpendicular directions: $0.05$.
    \item Number of states: $81$.
    \item Number of experimental runs: $20$.
\end{itemize}
Hyperparameters in Algorithms \ref{alg:PACIdentification} and \ref{alg:Initialization_EarlyStopping} are as follows.
\begin{itemize}
    \item $k$: $5 \times 10^4$.
    \item $\delta_\mathcal{M}$: $0.05, 0.1$.
    \item $\tau$: $0, 0.01, 0.05, 0.1, 0.3, 0.5$.
    \item $N$: $500$.
\end{itemize}

\subsection{Delivery in warehouse}
Environmental parameters are as follows:
\begin{itemize}
    \item Transition probability to the intended direction when not carrying the item: $1$.
    \item Transition probability to the intended direction when carrying the item: $0.6$.
    \item Transition probability to the same cell when not carrying the item: $0$.
    \item Transition probability to the same cell when carrying the item: $0.4$.
    \item Number of states: $2808$.
    \item Number of experimental runs: $20$.
\end{itemize}
Hyperparameters in Algorithms \ref{alg:PACIdentification} and \ref{alg:Initialization_EarlyStopping} are as follows.
\begin{itemize}
    \item $k$: $5 \times 10^5$.
    \item $\delta_\mathcal{M}$: $0.1$.
    \item $\tau$: $0$.
    \item $N$: $500$.
\end{itemize}

\subsection{Settings for Reinforcement learning in Warehouse Delivery}
We describe the experimental parameter settings for Q-learning in Section \ref{Example:warehouse}. The task is specified by the linear temporal logic formula $\varphi = \eventually(\mathsf{Pickup\_Item} \land \eventually \mathsf{Reach\_Goal}) \land \Box \neg \mathsf{Collision}$ \citep{baier2008principles}. The learning process is executed on the product space of the MDP and the deterministic B\"{u}chi automaton (DBA) translated from the task, that is, the set of states $S^\times = S \times Q$, where $S$ is the set of states of the MDP and $Q$ is the set of states of the DBA. Let $F^\times = S \times F$ be the set of accepting states in the product space, where $F$ is the set of accepting states of the DBA. The value function $Q: S^\times \times A \to \mathbb{R}$ to be learned is defined as follows:
\begin{align}
    Q^*(s^\times, a) = \sup_{\pi \in \Pi^\mathsf{det}} \mathbb{E}_\pi \left[ \sum_{t=0}^\infty \gamma^t R(S^\times_t) \mid S^\times_0 = s^\times, A_0 = a \right],
\end{align}
where the discount factor is given by $\gamma = 0.99$, $\Pi^\mathsf{det}$ denotes the set of deterministic policies $\pi: S^\times \to A$, and $\mathbb{E}_\pi$ denote the expectation under the policy $\pi$.
The reward function $R : S^\times \to \{0, 1\}$ is defined as follows:
\begin{align}
    R(s^\times) = 
    \begin{cases}
        1 & \mbox{ if } s^\times \in F^\times, \\
        0 & \mbox{ otherwise},
    \end{cases}
\end{align}
The update rule of the Q-learning is given by
\begin{align}
    Q(s^\times_t, a_t) \gets Q(s^\times_t, a_t) + \alpha \left(R(s^\times_t) + \lambda \max_{a_{t+1} \in \mathcal{A}(s^\times_{t+1})} Q(s^\times_{t+1}, a_{t+1}) - Q(s^\times_t, a_t) \right),
\end{align}
where the learning rate $\alpha = 0.1$.
During learning, we used the epsilon-greedy method to choose actions under the following scheduling for $\varepsilon$: For each learning step $t$, we compute $\varepsilon \in [0,1]$ as
\begin{align}
    \varepsilon = \max \{\eta^t \varepsilon_0, \varepsilon_\mathsf{min}\},
\end{align}
where $\eta = 0.995$, $\varepsilon=1$, and $\varepsilon_\mathsf{min} = 0.1$.
We conducted $10000$ episodes to learn the deterministic controller $\pi: S^\times \to A$.

\section{Application Scenarios}
\label{appendix:usecases}
We describe two concrete scenarios in which the identified causal states are useful.

\paragraph{Permissive robotic planning with unsafe terminal states.}
A cause defined by Def.~\ref{def:SPR_cause} is a set of states whose visitation increases the probability of reaching bad outcomes regardless of how the permissive controller resolves its actions, and the non-vacuity condition \textbf{(C2)} removes states reachable only through other causal states. Here, the nondeterminism arises from the controller's permissiveness, while the probabilistic transitions arise from the noisy environment or probabilistic events. The result is a compact set of ``robustly responsible'' states, which can guide local repair of the permissive controller toward avoiding the cause. In contrast, a cause defined by Def.~\ref{def:eSPR_cause} is a set of states whose visitation increases the probability of reaching bad outcomes under at least one way the controller resolves its actions, offering a set of ``potentially responsible'' states. This set is also useful for understanding the source of bad outcomes and for repairing the controller accordingly.

\paragraph{Linear temporal logic planning with a fixed controller and an adversary.}
Treating the accepting end components of a linear temporal logic task as the terminal set $E$, a cause defined by Def.~\ref{def:SPR_cause} identifies states whose visitation increases the task-satisfaction probability no matter what actions the other agent takes. Here, the nondeterminism arises from the other agent's action selection, while the probabilistic transitions arise from the controller's action selection or the noisy environment. Under this definition, the other agent is treated adversarially, so the cause is a set of ``robustly responsible'' states that promote task satisfaction even in the worst case, which helps identify favorable conditions for achieving the task and repair the policy to be more robust against adversarial behavior. In contrast, a cause defined by Def.~\ref{def:eSPR_cause} identifies states whose visitation increases the task-satisfaction probability under at least one action selection of the other agent, which is meaningful when the other agent is cooperative: such ``potentially responsible'' states promote task satisfaction provided that the cooperative agent acts suitably, and they suggest favorable conditions to be encouraged through coordination.

\end{document}